\documentclass[twoside,11pt]{article}

\usepackage[preprint]{jmlr2e}
\usepackage{amssymb}
\usepackage{thm-restate}
\usepackage{mathtools,thmtools}
\usepackage{bm,esint}
\usepackage{color}
\usepackage{float,graphicx,subcaption}
\usepackage{mathrsfs}  
\usepackage{bbm}
\usepackage{hyperref}
\usepackage{upgreek}
\usepackage{tikz-cd}
\usetikzlibrary{decorations.pathmorphing}
\usepackage[linesnumbered]{algorithm2e}

\newcommand{\size}{\mathrm{size}}

\newcommand{\depth}{\mathrm{depth}}

\newcommand{\pca}{\mathrm{PCA}}

\makeatletter
\newcommand{\oset}[3][0ex]{%
  \mathrel{\mathop{#3}\limits^{
    \vbox to#1{\kern-2\ex@
    \hbox{$\scriptstyle#2$}\vss}}}}
\makeatother
\newcommand{\simiid}{\oset[.6ex]{iid}{\sim}}

\newcommand{\dx}{{d_{\cX}}}
\newcommand{\dy}{{d_{\cY}}}

\renewcommand{\Im}{\mathrm{Im}}

\renewcommand{\tilde}{\widetilde}
\renewcommand{\hat}{\widehat}

\newcommand{\Id}{\mathrm{Id}}
\newcommand{\id}{\mathrm{id}}

\newcommand{\opt}{\mathrm{opt}}

\DeclareMathOperator*{\essinf}{ess\,inf}

\newcommand{\tr}{{\mathrm{tr}}}

\newcommand{\Err}{\widehat{\mathscr{E}}}

\newcommand{\Span}{\mathrm{span}}

\newcommand{\im}{\mathrm{Im}}
\newcommand{\embeds}{{\hookrightarrow}}

\renewcommand{\bar}{\overline}

\newcommand{\explain}[2]{\overset{\mathclap{\underset{\downarrow}{#2}}}{#1}}

\newcommand{\slot}{{\,\cdot\,}}

\newcommand{\supp}{\mathrm{supp}}

\newcommand{\T}{\mathbb{T}}
\newcommand{\R}{\mathbb{R}}
\newcommand{\C}{\mathbb{C}}
\newcommand{\E}{\mathbb{E}}
\newcommand{\N}{\mathbb{N}}

\newcommand{\Z}{\mathbb{Z}}
\newcommand{\Prob}{\mathrm{Prob}}

\renewcommand{\div}{{\mathrm{div}}}

\newcommand{\Lip}{\mathrm{Lip}}

\renewcommand{\P}{\mathbb{P}}
\newcommand{\cC}{\mathcal{C}}
\newcommand{\cD}{\mathcal{D}}
\newcommand{\cE}{\mathcal{E}}
\newcommand{\cF}{\mathcal{F}}
\newcommand{\cG}{\mathcal{G}}
\newcommand{\cH}{\mathcal{H}}

\newcommand{\cJ}{\mathcal{J}}
\newcommand{\cK}{\mathcal{K}}
\newcommand{\cL}{\mathcal{L}}

\newcommand{\cN}{\mathcal{N}}

\newcommand{\cP}{\mathcal{P}}

\newcommand{\cR}{\mathcal{R}}

\newcommand{\cX}{\mathcal{X}}
\newcommand{\cY}{\mathcal{Y}}

\newcommand{\define}{\textbf}

\renewcommand{\hat}{\widehat}

\newcommand{\set}[2]{{\left\{ #1 \,\middle|\, #2 \right\}}}

\newcommand{\shrink}{\mathrm{shrink}}
\newcommand{\dist}{\mathrm{dist}}

\newtheorem{claim}{Claim}
\tikzcdset{every label/.append style = {font = \normalsize},
            cells={nodes={font=\large}}}

\newcommand{\NL}{\mathscr{N\!\!L}}

\newcommand{\rev}[1]{#1}

\usepackage{lastpage}
\jmlrheading{24}{2023}{1-\pageref{LastPage}}{4/23; Revised
9/23}{10/23}{23-0478}{Samuel Lanthaler}

\ShortHeadings{PCA-Net: upper and lower bounds}{Lanthaler}
\firstpageno{1}

\begin{document}

\title{Operator learning with PCA-Net: \\
upper and lower complexity bounds}

\author{\name Samuel Lanthaler \email slanth@caltech.edu \\
       \addr Computing and Mathematical Sciences\\
       California Institute of Technology\\
       Pasadena, CA 91125, USA}
\editor{Aarti Singh}

\maketitle

\begin{abstract}%   <- trailing '%' for backward compatibility of .sty file
PCA-Net is a recently proposed neural operator architecture which combines principal component analysis (PCA) with neural networks to approximate operators between infinite-dimensional function spaces. The present work develops approximation theory for this approach, improving and significantly extending previous work in this direction: First, a novel universal approximation result is derived, under minimal assumptions on the underlying operator and the data-generating distribution. Then, two potential obstacles to efficient operator learning with PCA-Net are identified, and made precise through lower complexity bounds; the first relates to the complexity of the output distribution, measured by a slow decay of the PCA eigenvalues. The other obstacle relates to the inherent complexity of the space of operators between infinite-dimensional input and output spaces, resulting in a rigorous and quantifiable statement of \rev{a ``curse of parametric complexity'', an infinite-dimensional analogue of the well-known curse of dimensionality encountered in high-dimensional approximation problems.} In addition to these lower bounds, upper complexity bounds are finally derived. A suitable smoothness criterion is shown to ensure an algebraic decay of the PCA eigenvalues. Furthermore, it is shown that PCA-Net can overcome the general curse for specific operators of interest, arising from the Darcy flow and the Navier-Stokes equations.
\end{abstract}


\section{Introduction}
The application of neural networks \citep{DLbook} to computational science and engineering is receiving growing interest. At their core, many problems of scientific interest involve the approximation of an underlying operator, which defines a mapping between two infinite-dimensional spaces of functions. \emph{Neural operators} \citep{anandkumar_neural_2020,pca,lu2021learning,kovachki2021neural} are a generalization of neural networks to such an infinite-dimensional setting. They aim to approximate, or ``learn'', operators from data given in the form of input and output pairs. Neural operators hold promise as surrogate models to accelerate and complement traditional numerical methods in many-query problems, and they can be used for data-driven discovery of the underlying input-output map, even when no mathematical model is available.

Recent years have seen the emergence of several neural operator architectures. This includes deep operator networks (DeepONet) \citep{lu2021learning}, building on early work on operator learning in \citep{chen1995universal}. We also mention subsequent extensions of DeepONets, for example \citep{mionet,prasthofer2022variable,seidman2022nomad,LMHM2022,patel2022variationally}. 
DeepONets have been deployed with success in a variety of applications \citep{di_leoni_deeponet_2021,mao_deepmandmnet_2021,cai_deepmmnet_2021}. Another popular approach is based on a class of neural operators introduced in  \citep{anandkumar_neural_2020,kovachki2021neural}. Here, neural operators are defined in close analogy with conventional neural networks, where the weight matrices in the hidden layers are generalized to integral operators. Special cases of this framework include the graph neural operator \citep{anandkumar_neural_2020,li_multipole_2020} and the Fourier neural operator (FNO) \citep{li_fourier_2021}. In this context, we also mention related frameworks in \citep{raonic2023,tripura2023,gupta2021multiwavelet}. Another notable, and somewhat distinct, approach to operator learning is the operator-valued random feature model proposed in \citep{nelsen2021random}.

Universal approximation results for many of these frameworks are known in a variety of settings; the universality of DeepONets is established in \citep{chen1995universal,lu2021learning,thdeeponet}, FNOs are shown to be universal in \citep{thfno}, another universal approximation result for neural operators is derived in \citep{kovachki2021neural}. We also mention recent work \citep{LLS2023}, which proves a general universal approximation result for the so-called ``averaging neural operator'' (ANO), a minimal architecture that is at the core of many other frameworks, thereby allowing to unify much of the analysis of this emerging zoo of neural operator architectures. 

Going beyond universality, it is crucial to improve our understanding of the required computational complexity of neural operators in order to assess when the methods will be effective. Pertinent numerical experiments may be found in \citep{de2022cost}. Relevant analysis of linear problems from this point of view has been given in \citep{boulle2022learning,de2021convergence}. The required complexity of neural network-based methods for specific PDE operators of interest are studied from a approximation theoretic point of view, in e.g. 
\citep{thfno,SchwabZech2019,Kutyniok2022,thdeeponet,ryck2022generic}. 
A focus of these papers is on beating the \emph{``curse of dimensionality''}. 

Since the input and output spaces are infinite-dimensional in these problems, clarification may be needed as to the meaning of beating the curse of dimensionality: it is interpreted as identifying conditions under which the required size (number of tunable parameters) of the operator approximation grows only algebraically 
with the inverse of the desired error.
\rev{This notion of the "curse of dimensionality" is specific to operator learning and somewhat distinct from the conventional meaning of this term in the context of high-dimensional problems. To disambiguate the well-known curse of dimensionality in high-dimensional approximation problems from the analogue of this curse in infinite-dimensional approximation (i.e. operator learning), we will refer to the latter as the ``curse of parametric complexity''. The present work will provide further clarification and motivation for the use of this term in the context of operator learning, and will connect the curse of dimensionality with this curse of parametric complexity.}

The focus of the present work is the so-called \emph{PCA-Net}, a methodology which combines ideas from principal component analysis with neural networks \citep{hesthaven_nonintrusive_2018,pca}. 
Principal component analysis (PCA) is a standard tool for dimension reduction in high-dimensional statistics and unsupervised learning \citep{jolliffe2002principal}. In \citep{pca}, a combination of PCA with neural networks has been proposed as a data-driven operator learning framework. As indicated above, the goal of operator learning is to approximate an unknown operator $\Psi^\dagger: \cX \to \cY$, mapping between two infinite-dimensional spaces $\cX$ and $\cY$. Given data in the form of pairs of inputs and outputs, we seek to determine an accurate, data-driven approximation of $\Psi^\dagger$. The PCA-Net operator learning architecture achieves this goal by (i) using PCA to reduce the dimensions of the input and output spaces and (ii) approximating a map between the resulting finite-dimensional latent spaces \citep{pca}. First analysis, including a universal approximation result, have been derived in \citep{pca}. Furthermore, in the same work, the efficacy of the proposed architecture has been demonstrated empirically for prototypical problems, including the solution operator of the viscous Burgers equation and the Darcy flow equation. However, so far, a detailed mathematical analysis providing a theoretical underpinning for this empirically observed efficiency of PCA-Net, has been outstanding.

The present work fills this gap by developing relevant approximation theory for PCA-Net. The main contributions of this paper are the following:
\begin{itemize}
\item \textbf{Universal approximation:} We prove a novel universal approximation theorem for PCA-Net, Theorem \ref{thm:universal}, under significantly relaxed conditions on the distribution of the data-generating measure and the underlying operator $\Psi^\dagger$ compared to previous work; the universality of PCA-Net is here shown under natural minimal conditions, which are in fact necessary for PCA to be well-defined on the input and output spaces. 
\item \textbf{\rev{Curse of parametric complexity:}} A rigorous result is proven which demonstrates that the \rev{curse of parametric complexity, an infinite-dimensional scaling limit of the conventional curse of dimensionality,} cannot be overcome by PCA-Net in general (cp. Theorem \ref{thm:cod}); more precisely, this result shows that it is impossible to derive algebraic complexity bounds when considering general classes of operators, such as the class of all Lipschitz- or even $\cC^k$-continuous operators. Hence, we conclude that at this level of generality, the curse is unavoidable. %In this context, we also mention forthcoming work \citep{LS2023}, where similar ideas are refined and considerably generalized, and it is demonstrated that this curse is in fact inescapable for a wide variety of neural operator architectures.
\item \textbf{Overcoming the \rev{curse of parametric complexity}:} Given the negative result on the general curse, we argue that a central challenge in operator learning is to identify the relevant class of operators which \emph{do allow} for efficient approximation by a given operator learning framework. To gain further insight into the relevant mathematical structure that can be leveraged by PCA-Net, we restrict attention to two prototypical PDE operators of interest arising from the Darcy flow and Navier-Stokes equations. In both cases, we show that PCA-Net can overcome the general curse of parametric complexity; algebraic error and complexity estimates are established in Theorems \ref{thm:darcy} and \ref{thm:ns}, demonstrating that these operators belong to a restricted class which is efficiently approximated by PCA-Net.
\end{itemize}

\subsection{Overview}
In section \ref{sec:pcanet}, relevant background on PCA and the PCA-Net methodology is provided; First,  PCA and empirical PCA are reviewed in section \ref{sec:pcadef}, and two error estimates for the PCA projection error are stated in section \ref{sec:pcaproj}. Next, it is explained how PCA-Net combines PCA with a neural network, resulting in an operator learning architecture. 

In Section \ref{sec:approx}, we develop approximation theory for PCA-Net. A new universal approximation result for PCA-Net is derived in section \ref{sec:universal}.
%, extending and improving previous analysis in \citep{pca} by relaxing the assumptions on the input measure $\mu$ and the underlying operator $\Psi^\dagger$. 
In section \ref{sec:quant}, two potential obstacles to the efficacy of PCA-Net are identified through rigorous lower complexity bounds.
%; one obstacle is associated with a high complexity of the output space (manifested in a potentially slow decay of the PCA eigenvalues), the other one is associated with the inherent complexity of the space of operators between infinite-dimensional spaces (``curse of dimensionality''). 
Upper complexity bounds are the subject of the remaining sections: In section \ref{sec:quant-enc}, a smoothness criterion is shown to rule out the first potential obstacle to efficient operator learning.
%is ruled out if the output functions are sufficiently smooth. As many PDEs, in particular elliptic and parabolic PDEs, produce smooth solutions, this result rules out the first potential inefficiency of PCA-Net for a large class of operators. 
We then finally show, in section \ref{sec:overcoming}, how PCA-Net can overcome the curse of parametric complexity for two prototypical PDE operators arising in the context of the Darcy flow and Navier-Stokes equations, respectively. Conclusions and perspectives for future work are summarized in Section \ref{sec:conclusion}.
%$\Psi^\dagger: \cX \to \cY$. To this end, we first consider the solution operator $\Psi^\dagger: a \mapsto u$ associated with the Darcy flow PDE $-\nabla \cdot (a\nabla u) = f$, mapping the coefficient field $a\in \cX$ to the solution $u\in \cY$. Numerical experiments have demonstrated that PCA-Net is effective at approximating the solution operator for the Darcy flow equation \citep{pca}; building on our results on the PCA projection error, we will derive quantitative error and complexity estimates for PCA-Net. In addition to the Darcy flow problem, we derive similar upper bounds for the solution operator of the Navier-Stokes equations, mapping initial data to the solution at a later time. These bounds provide a first theoretical explanation and a rigorous foundation for the effectiveness of PCA-Net for problems of practical relevance.

\subsection{Notation}
Throughout the following discussion, $\cH,\cX, \cY$ denote separable Hilbert spaces. We will use $\Vert \slot \Vert_{\cH}$ and $\langle \slot, \slot \rangle_\cH$ to denote the norm and inner product on $\cH$; if it is clear from the context, we may occasionally omit the subscript to aid readability. On finite-dimensional Euclidean spaces, $|\slot|$ is used for the Euclidean norm, and $|\slot|_\infty$ denotes the maximum-norm. The space of probability measures on $\cH$ is denoted $\cP(\cH)$. We denote by $u\sim \mu$ a random variable distributed according to probability measure $\mu$. $\E_{u\sim \mu}[F(u)]$ denotes the expectation of $F$ with respect to $\mu$. We consistently use $\Psi^\dagger$ to denote the underlying (truth) operator, and $\Psi$ will denote a (PCA-Net) approximation of $\Psi^\dagger$. For two numbers $a,b$, we will write $a\sim b$ for equivalence up to constants, i.e. there exists $C>0$ such that $C^{-1} a \le b \le Ca$. Similarly $a\lesssim b$, $a\gtrsim b$ denotes inequality up to a constant, i.e. $a \le C b$ and $Ca \ge b$, respectively. On occasion, we write a subscript $a\lesssim_k b$ to emphasize the dependence of the implied constant $C = C(k)$ on a given parameter $k$. We follow the convention that constants in estimates can change their value from line to line; their dependence on the relevant parameters will always be indicated. Other notation is introduced as needed.

\section{PCA-Net methodology}
\label{sec:pcanet}

PCA-Net combines principal component analysis (PCA) for dimension reduction of the input and output spaces, with a neural network mapping between the resulting finite-dimensional latent spaces. Before summarizing PCA-Net, we first review PCA in subsection \ref{sec:pcadef}, and derive two high-probability estimates for the PCA projection error in subsection \ref{sec:pcaproj}. Next, we summarize how PCA-Net combines PCA with a neural network, resulting in an operator learning architecture in subsection \ref{sec:pca-net}.

\subsection{Principal Component Analysis}
\label{sec:pcadef}

In the present section, we provide necessary background material on PCA, and we prove a high-probability estimate for the PCA projection error, building on previous results \citep{reisswahl2020,milbradt2020high}.

\paragraph{PCA.}
Given a Hilbert space $\cH$, a probability measure $\mu$ on $\cX$, and a projection dimension $d$, PCA aims to minimize the average reconstruction error $\E_{u\sim \mu}[\Vert u - P u \Vert^2]$ over the set $\Pi_d$ of orthogonal projections $P: \cH \to \cH$ of rank $d$. It is well-known (e.g. \citep{jolliffe2002principal,pca,reisswahl2020}) that this can be achieved by considering the covariance operator $\Sigma = \E_{u\sim \mu}[u \otimes u]$, which is diagonalizable; i.e., there exists a sequence ${\lambda}_1\ge {\lambda}_2 \ge \dots \ge 0$ of eigenvalues and corresponding orthonormal basis of eigenvectors ${\phi}_1,{\phi}_2,\dots \in \cH$, such that $\Sigma {\phi}_j = {\lambda}_j {\phi}_j$ for all $j$. The optimal PCA projection of dimension $d$, ${P}_{\le d}: \cH \to \cH$, can then be written as a composition ${P}_{\le d} = \cD_\cH^\opt \circ \cE_\cH^\opt$, where the optimal PCA encoder $\cE_\cH^\opt$ is given by,
\begin{align}
\label{eq:EH}
\cE_\cH^\opt: \cH \to \R^d, \quad \cE_\cH^\opt(u) := (\langle u, \phi_1\rangle, \dots, \langle u, \phi_d\rangle),
\end{align}
and the corresponding PCA decoder $\cD_\cH^\opt$ is the mapping,
\begin{align}
\label{eq:DH}
\cD_\cH^\opt: \R^d \to \cH, \quad \cD_\cH^\opt(\eta) = \sum_{j=1}^d \eta_j \phi_j.
\end{align}
It can be shown that ${P}_{\le d}$ defines an orthogonal projection on $\cH$, such that
\[
\E_{u\sim \mu}[\Vert u - {P}_{\le d} u \Vert_{\cH}^2]
=
\min_{P\in \Pi_d} \E_{u\sim \mu}[\Vert u - P u \Vert_{\cH}^2].
\]
In the following we denote by 
\begin{align}
\label{eq:Rd}
\cR^\opt_d(\mu) := \min_{P\in \Pi_d} \E_{u\sim \mu}[\Vert u - P u \Vert_{\cH}^2],
\end{align}
this optimal PCA projection error. One can show, e.g. \citep[Thm. 3.8]{thdeeponet}, that $\cR^\opt_d(\mu)$ is related to the PCA eigenvalues by
\begin{align}
\label{eq:Rd1}
\cR^\opt_d(\mu) = \sum_{j>d} \lambda_j.
\end{align}

 \paragraph{Empirical PCA.}
Empirical PCA applies the above procedure to the empirical distribution $\mu_N = \frac1N \sum_{k=1}^N \delta_{u_k}$, obtained by sampling from $\mu$: Given a finite number of independent and identically distributed (i.i.d.) samples $u_1,\dots, u_N \simiid \mu$, define the covariance operator ${\Sigma}_N$, by 
\begin{align}
\label{eq:empirical-cov}
{\Sigma}_N = \frac{1}{N} \sum_{k=1}^N u_k \otimes u_k.
\end{align}
Letting $\hat{\lambda}_{1} \ge \hat{\lambda}_{2} \ge \dots \ge 0$ and $\hat{\phi}_{1}, \hat{\phi}_{2}, \dots \in \cH$ denote the eigenvalues and corresponding orthonormal eigenbasis. The empirical PCA projection $\hat{P}_{\le d}$ of dimension $d$ is given by $\hat{P}_{\le d} = \cD_\cH \circ \cE_\cH$, where $\cE_\cH$, $\cD_\cH$ are defined as in \eqref{eq:EH}, \eqref{eq:DH}, but replacing the eigenvectors $\phi_j$ by their empirical counterparts $\hat{\phi}_j$.

\subsection{Projection error of empirical PCA}
\label{sec:pcaproj}

We first note that the empirical PCA projection error approximates the optimal projection error, provided that a sufficient amount of data is available. We state the following result in high probability.
\begin{proposition}
\label{prop:pcaproj-qual}
Let $\cH$ be a separable Hilbert space. Let $\mu\in \cP(\cH)$ be a probability measure with finite second moments, $\E_{u\sim \mu}[\Vert u \Vert_{\cH}^2] < \infty$. Then for any $\delta, \epsilon>0$, there exists a requisite amount of data $N_0 = N_0(\mu,d,\delta,\epsilon)$, such that the encoding error for empirical PCA with dimension $d$, and based on $N\ge N_0$ samples $u_1,\dots, u_N \simiid \mu$, satisfies
\begin{align}
\label{eq:pcaproj-qual}
\E_{u\sim \mu}\left[ 
\Vert u - \cD_\cH\circ \cE_\cH(u) \Vert_{\cH}^2
\right] 
\le \cR^\opt_d(\mu) + \epsilon,
\end{align}
with probability at least $1-\delta$.
\end{proposition}
The proof of Proposition \ref{prop:pcaproj-qual} relies on a well-known bound on the excess risk for empirical PCA in terms of the Hilbert-Schmidt distance $\Vert \Sigma - \hat{\Sigma} \Vert_{2}$ of the true and empirical covariance operators \citep[e.g. Sect. 2.2]{reisswahl2020}. This is combined with a general Monte-Carlo estimate to prove Proposition \ref{prop:pcaproj-qual}. We present the details in Appendix \ref{app:pcaproj}.

The result of Proposition \ref{prop:pcaproj-qual} is purely qualitative, as it doesn't give us any estimate on the required amount of data $N$. Our next goal is to establish a \emph{quantitative} bound, under additional assumptions on $\mu$. We will call a probability measure $\mu\in \cP(\cH)$ \define{sub-Gaussian}, if there exists $K_\mu \ge 0$, such that
\begin{align}
\label{eq:momentbd}
\E_{u\sim \mu}\left[
\Vert u \Vert^p_{\cH}
\right]^{1/p}
\le
K_\mu \sqrt{p}, \quad \forall\, p\ge 1.
\end{align}
According to this definition \eqref{eq:momentbd}, $\mu$ is sub-Gaussian if the real-valued random variable $\Vert u \Vert_{\cH}$, with $u \sim \mu$, is sub-Gaussian in the conventional sense. The moment bound \eqref{eq:momentbd} is one of many equivalent characterizations of real-valued sub-Gaussian random variables (see e.g. \citep[Sect. 2.5.1]{vershynin2018high}). We then have:

\begin{proposition}
\label{prop:pcaproj}
Let $\cH$ be a separable Hilbert space, and let $\mu$ be a sub-Gaussian probability measure on $\cH$. Fix $\delta \in (0,1/2)$. The encoding error for empirical PCA with dimension $d$, and based on $N \ge \log(2/\delta)$ samples $u_1,\dots, u_N\simiid \mu$, satisfies the following upper bound,
\begin{align}
\label{eq:pcaproj}
\E_{u\sim \mu}\left[
\Vert u - \cD_\cH\circ \cE_\cH(u) \Vert_{\cH}^2
\right]
\le 
\cR^\opt_d(\mu) + \sqrt{\frac{Qd \log(2/\delta)}{N}}.
\end{align}
with probability at least $1-\delta$. Here, $Q = Q(K_\mu)$ depends only on the constant $K_\mu$ in \eqref{eq:momentbd}.
\end{proposition}

Proposition \ref{prop:pcaproj} above is a natural high-probability analogue of a previous result on empirical PCA of \citep{pca}, derived in expectation in that work. Proposition \ref{prop:pcaproj} uses the same bound on the PCA excess risk as Proposition \ref{prop:pcaproj-qual}, but combines it with a general Bernstein concentration bound for $\cH$-valued random variables to derive quantitative rates. We include the details in Appendix \ref{app:pcaproj}.

\begin{remark}
\label{rmk:pcaproj}
We note that under more fine-grained information on the underlying measure $\mu$, considerable improvements to the upper bound of Proposition \ref{prop:pcaproj} are possible; this has e.g. been achieved in \citep{milbradt2020high} for a different notion of a ``sub-Gaussian'' distribution, requiring that 
\begin{align}
\label{eq:momentbd2}
\sup_{p\in \N} \frac{\E_{u\sim\mu}[|\langle u, v \rangle_{\cH}|^p]^{1/p} }{\sqrt{p}}
\le K'_\mu \E_{u\sim\mu}[|\langle u, v \rangle_{\cH}|^2]^{1/2}, \quad \forall \, v \in \cH,
\end{align}
for a constant $K_\mu'$ depending only on $\mu$.
The condition \eqref{eq:momentbd2} is stronger than \eqref{eq:momentbd}; Assuming \eqref{eq:momentbd2}, it can in fact be shown that the empirical PCA projection error is of order $\mathcal{O}\left(\cR^\opt_d(\mu)\right)$ whenever $N \gtrsim d \log(1/\delta)$, thereby achieving essentially optimal PCA convergence rates. However, in the present context, PCA will be applied for dimension reduction on both the input and output spaces under a non-linear mapping $\Psi^\dagger$. Unfortunately, even if $\mu$ satisfies \eqref{eq:momentbd2}, it is unclear whether this property is preserved under the push-forward by $\Psi^\dagger$, i.e. whether a bound of the form \eqref{eq:momentbd2} continues to hold for $\Psi^\dagger_\#\mu$. In contrast, the bound \eqref{eq:momentbd} is robust under such a push-forward. Therefore, we contend ourselves with the more pessimistic bound of Proposition \ref{prop:pcaproj}, and leave potential improvements, possibly building on \citep{milbradt2020high}, as a challenge for future work.
\end{remark}

\subsection{PCA-Net architecture}
\label{sec:pca-net}

We next recall the PCA-Net architecture proposed in \citep{pca}, which combines empirical PCA with a neural network mapping, to approximate an underlying operator $\Psi^\dagger: \cX \to \cY$. In the following, let $\cX$ and $\cY$ be separable Hilbert spaces and let $\Psi^\dagger: \cX \to \cY$ be a non-linear operator. The goal of the PCA-Net methodology is to approximate $\Psi^\dagger$ from a finite number of input-/output-samples $\{u_k,\Psi^\dagger(u_k)\}_{k=1}^N$. To this end, PCA-Net combines an encoding $\cE_\cX: \cX \to \R^\dx$, a neural network $\psi: \R^\dx \to \R^\dy$ and a decoding $\cD_\cY: \R^{\dy} \to \cY$ (cp. Figure \ref{fig:pcanet}). 

\begin{figure}[H]
\centering
\begin{tikzcd}
\cX \arrow{r}{\Psi^\dagger} \arrow[swap,squiggly]{d}{\cE_\cX} 
& \cY  \\%
\R^{\dx} \arrow[squiggly]{r}{\psi}& \R^{\dy} \arrow[squiggly,swap]{u}{\cD_\cY}
\end{tikzcd}
\caption{Diagrammatic illustration of PCA-Net based on a PCA encoder $\cE_\cX$, a neural network $\psi$, and a PCA decoder $\cD_\cY$.}
\label{fig:pcanet}
\end{figure}
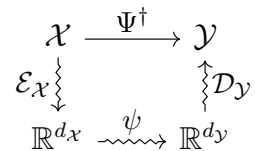  

Here, $\cE_\cX$ and $\cD_\cY$ are chosen as an empirical PCA encoder and decoder, respectively. A precise definition is given below, see equations \eqref{eq:encX}, \eqref{eq:decY}. Given these ingredients, the resulting PCA-Net is defined as the mapping
\begin{align}
\label{eq:pcanet}
\Psi: \cX \to \cY, \quad \Psi(u) := \cD_\cY \circ \psi \circ \cE_\cX(u).
\end{align}
The encoder $\cE_\cX$ and decoder $\cD_\cY$ perform a dimension reduction on the input and output spaces. The neural network $\psi$ approximates a mapping on the resulting finite-dimensional latent spaces. 

\rev{Associated with the encoder $\cE_\cX$, we will define a natural PCA decoder $\cD_\cX$ on $\cX$. Similarly $\cD_\cY$ is accompanied by an encoder $\cE_\cY$ on $\cY$.}
The intuition is that these underlying encoder/decoder pairs, and the neural network $\psi$, should satisfy the following approximate identities \citep{pca}:
\[
\cD_\cX \circ \cE_\cX \approx \id_{\cX}, 
\quad 
\cD_\cY \circ \cE_\cY \approx \id_{\cY},
\quad
\cD_\cY \circ \psi \circ \cE_\cX \approx \Psi^\dagger.
\]
Here, $\id_\cX$ and $\id_{\cY}$ denote the identity mappings on $\cX$ and $\cY$, respectively. Our analysis of the PCA-methodology aims to quantify the accuracy of these approximations when the encoders/decoders are defined by PCA and $\psi$ is chosen to be a neural network.

\paragraph{PCA encoding and decoding.} We will now specify the particular choice of $\cE_\cX$ and $\cD_\cY$. In the following, we assume that the input samples $\{u_k\}_{k=1}^N$ are i.i.d. samples from a probability measure $\mu$ on $\cX$. Note that the output-samples $\{\Psi^\dagger(u_k)\}_{k=1}^N$ are then i.i.d. with respect to the corresponding push-forward measure $\Psi^\dagger_\#\mu$ on $\cY$. For a given choice of latent dimensions $\dx$ and $\dy$, we apply empirical PCA to the samples $\{u_k\}_{k=1}^N \subset \cX$ and $\{\Psi^\dagger(u_k)\}_{k=1}^N\subset \cY$. In the following, we denote by 
\begin{align}
\phi^\cX_{1}, \dots, \phi^\cX_{\dx} \in \cX,
\quad
\phi^\cY_{1}, \dots, \phi^\cY_{\dy} \in \cY,
\end{align}
the empirical PCA bases on $\cX$ and $\cY$, respectively. We emphasize that the empirical PCA bases are themselves random variables, as they depend on the random input-/output-samples $\{u_k,\Psi^\dagger(u_k)\}_{k=1}^N$. The first basis, $\{\phi^\cX_j\}_{j=1}^\dx$, defines an encoder on $\cX$,
\begin{align}
\label{eq:encX}
\cE_\cX: \cX \to \R^\dx, \quad \cE_\cX(u) := (\langle u, \phi^\cX_{1} \rangle, \dots, \langle u, \phi^\cX_{\dx} \rangle),
\end{align}
with corresponding decoder,
$\cD_\cX: \R^\dx \to \cX$, $\cD_\cX(\xi) := \sum_{j=1}^\dx \xi_j \phi^\cX_{j}$.
Similarly, $\phi^\cY_1,\dots, \phi^\cY_\dy$ defines an encoder on $\cY$,
$\cE_\cY: \cY \to \R^\dy$, $\cE_\cY(v) := (\langle v, \phi^\cY_{j} \rangle)_{j=1}^\dy$,
with corresponding decoder,
\begin{align}
\label{eq:decY}
\cD_\cY: \R^\dy \to \cY, \quad \cD_\cY(\eta) := \sum_{j=1}^\dy \eta_j \phi^\cY_{j}.
\end{align}

\paragraph{Neural networks.}
\label{sec:neuralnetwork}
\rev{Given a depth $L\in \N$, layer widths $d_k$ ($k=0,\dots, L+1$), and weights and biases $A_k\in \R^{d_{k+1}\times d_{k}}, b_k\in \R^{d_{k+1}}$, $k=0,\dots, L$, a \emph{(deep) neural network} (DNN) $\psi$ is a mapping $\xi\mapsto \psi(\xi)$, defined as a composition of non-linear layers,
\begin{gather}
\label{eq:neuralnetwork}
\left\{
\begin{aligned}
\xi_0 &:= \xi, \\
\xi_{k+1} &:= \sigma(A_k \xi_{k} + b_k), \quad \text{for } k=0,\dots, L-1, \\
\psi(\xi) &:= A_L \xi_L + b_L.
\end{aligned}
\right.
\end{gather}
}
The non-linear layers are expressed in terms of an activation function $\sigma: \R \to \R$ which is applied componentwise. In the following, we will restrict attention to the ReLU activation function $\sigma(\xi) := \max(\xi,0)$. Extension of our results to more general choices of $\sigma$ is possible. 

Given a DNN $\psi$, we define $\size(\psi) := \sum_{k=0}^L (\Vert A_k \Vert_0 + \Vert b_k \Vert_0)$ as the total number of non-zero weights and biases in the architecture, and we define $\depth(\psi) := L$ as the number of layers. With these definitions, $\size(\psi)$ and $\depth(\psi)$ provide a measure of the complexity of the DNN.

\rev{
\begin{remark}
In the present work, we employ the above notion of size as a measure of the complexity of a neural network. This is sufficient for our present purposes, but it should be pointed out that this notion of complexity does not take into account the number of bits which would be required to encode each parameter in a practical implementation. Recent work such as \cite{yarotsky2020phase,siegel2022optimal}, and references therein, has shown that it is possible for very carefully constructed ReLU neural networks to achieve unexpectedly high convergence rates for the approximation of Sobolev functions. This comes at the expense of requiring a large number of bits to encode the individual weights, e.g. \cite[discussion on page 5]{siegel2022optimal}. 
\end{remark}
}

\paragraph{Neural network training.} 
Ideally, the neural network would be chosen as a minimizer of the expected loss
\begin{align*}
\cL(\psi) := \E_{(\xi,\eta)}\left[ \vert \psi(\xi) - \eta\vert^2 \right],
\end{align*}
where the expectation is over pairs $(\xi,\eta) \in \R^{\dx}\times \R^{\dy}$ of encoded input-/output-pairs of the form $\xi = \cE_\cX(u)$, $\eta = \cE_\cY(\Psi^\dagger(u))$, with $u\sim \mu$.
In practice, the neural network $\psi$ in \eqref{eq:pcanet} is usually trained to minimize the following empirical loss,
\begin{align*}
\hat{\cL}(\psi) := \frac{1}{N} \sum_{k=1}^N \vert \psi(\xi_k) - \eta_k \vert^2,
\end{align*}
where $\xi_k := \cE_\cX(u_k) \in \R^\dx$ and $\eta_k := \cE_{\cY}(\Psi^\dagger(u_k))\in \R^\dy$ are the encoded input and output samples.

\begin{remark}
\label{rem:training}
In the present work, we will not address the practical training of the neural network $\psi$, analysis of which is a notoriously difficult problem. Instead, we will contend ourselves with an ``approximation theoretic'' approach; it will be shown that certain $\psi$ exists, but no a priori guarantee is given that this $\psi$ can be found by the numerical optimization of the empirical loss $\hat{\cL}$. The samples $u_1,\dots, u_N$ will, however, enter our analysis of the empirical PCA encoder $\cE_\cX$ and decoder $\cD_\cY$.
\end{remark}

\section{Approximation theory}
\label{sec:approx}

This section develops approximation theory for PCA-Net, and is divided into four subsections. First, in Section \ref{sec:universal}, we discuss a new universality theorem for PCA-Net. Next, we point out two potential obstacles to efficient operator learning with PCA-Net in Section \ref{sec:quant}. As explained there, one such obstacle relates to the complexity of the output distribution on $\cY$, which could in principle lead to an arbitrarily slow decay of the PCA-Net error with the PCA-dimension $\dy$ (cp. Proposition \ref{prop:lower-bd}). The second potential obstacle to efficient operator learning relates to the inherent complexity of the space of operators between infinite-dimensional spaces (cp. Theorem \ref{thm:cod}); this makes rigorous a notion of ``curse of dimensionality'' for operator learning intuited in earlier work \citep{thfno,thdeeponet}. The lower bounds of Section \ref{sec:quant} are complemented by upper bounds in the subsequent sections; In Section \ref{sec:quant-enc}, we show that a suitable smoothness condition rules out the first obstacle, ensuring an algebraic decay of the PCA encoding error on $\cY$. Finally, in Section \ref{sec:overcoming}, we demonstrate that PCA-Net can leverage additional structure for two operators of interest, allowing it to overcome the general curse of dimensionality when approximating the solution operator of the Darcy flow and Navier-Stokes equations. 

\subsection{Universal approximation}
\label{sec:universal}

We first state the following universal approximation theorem in high probability:
\begin{theorem}[Universal approximation]
\label{thm:universal}
Let $\cX,\cY$ be separable Hilbert spaces and let $\mu \in \cP(\cX)$ be a probability measure on $\cX$. Let $\Psi^\dagger: \cX \to \cY$ be a $\mu$-measurable mapping. Assume the following moment conditions,
\[
\E_{u\sim \mu}[\Vert u \Vert_{\cX}^2], \;\; \E_{u\sim \mu}[\Vert \Psi^\dagger(u) \Vert_{\cY}^2] < \infty.
\]
Then for any $\delta, \epsilon > 0$, there are dimensions $\dx= \dx(\epsilon, \delta)$, $\dy = \dy(\epsilon, \delta)$, a requisite amount of data $N = N(\dx,\dy,\mu,\Psi^\dagger)$, a neural network $\psi = \psi(\slot;\{u_k\})$ \rev{depending on this data,} such that the PCA-Net, $\Psi = \cD_{\cY} \circ \psi \circ \cE_{\cX}$, satisfies
\[
\E_{u\sim \mu} \left[ \Vert \Psi^\dagger(u) - \Psi(u;\{u_k\}) \Vert^2_{\cY}  \right]
\le
\epsilon,
\]
with probability at least $1-\delta$ in the input data $u_1,\dots, u_N \sim \mu$.
\end{theorem}
We have written $\Psi(u) = \Psi(u;\{u_k\})$ to emphasize the dependency of the PCA-Net encoder $\cE_\cX$, decoder $\cD_\cY$ \rev{and the neural network $\psi = \psi(\slot;\{u_k\})$} on the given data $u_1,\dots, u_N$.
The detailed proof, included in Appendix \ref{app:universal}, combines well-known universal approximation results for neural networks with the high-probability bounds on the PCA projection error in Proposition \ref{prop:pcaproj-qual} in Section \ref{sec:pcaproj}, above; more precisely, it is shown in Lemma \ref{lem:err-decomp} that the PCA error can be decomposed into an error due to the encoding on $\cX$, the decoding error on $\cY$ and a neural network approximation error. The encoding and decoding errors are expressed in terms of the PCA projection error, which can be bounded by invoking Proposition \ref{prop:pcaproj}. The neural network approximation error can then be made arbitrarily small by a suitable choice of the neural network $\psi$. 

\rev{
\begin{remark}
As already pointed out in Remark \ref{rem:training}, the practical training of the neural network $\psi$ is not discussed in the present work. Hence the above theorem should be viewed as a pure existence result, in the sense that the proof does not provide an explicit algorithmic recipe for determining the neural network $\psi$, given the data $u_1,\dots, u_N$.
\end{remark}
}

Theorem \ref{thm:universal} shows that PCA-Net is able to approximate almost arbitrary operators to any desired accuracy, provided a sufficient number of data points are available for empirical PCA, and provided that the underlying neural network is sufficiently large. A previous universal approximation result in \citep[Theorem 3.1]{pca} was stated only for Lipschitz continuous $\Psi^\dagger$ and under an assumption of finite fourth moments. In contrast, Theorem \ref{thm:universal} requires no regularity condition on the underlying operator $\Psi^\dagger$, and shows that a bound on the second moments of the input measure $\mu$ and the push-forward measure $\Psi^\dagger_\#$ suffices for universal approximation. We have formulated Theorem \ref{thm:universal} in high-probability, whereas \citep[Theorem 3.1]{pca} is derived in expectation. This last difference is mostly for consistency with the high probability results in later sections; indeed, Theorem \ref{thm:universal} is in fact derived from a corresponding result in expectation (cp. Proposition \ref{prop:universal}).

The main drawback of universal approximation results is that they are purely qualitative, and do not provide any information about the required size of the data or the neural network; hence, universal approximation cannot provide information on the efficiency of operator learning with PCA-Net. Deriving more quantitative bounds is particularly relevant in view of the two potential obstacles to efficient operator learning, elaborated upon in the next section.

\subsection{Obstacles to effective operator learning}
\label{sec:quant}

Theorem \ref{thm:universal} does not provide any quantitative information on the required complexity to achieve a given accuracy. For the practical success of PCA-Net in operator learning, it is crucial that the PCA-Net approximation is not only possible in principle, but also efficient in practice; we interpret ``efficiency'' as the statement that the PCA dimensions $\dx$, $\dy$, the requisite amount of data $N$ and the size of the neural network $\psi$, that is required to achieve a desired accuracy $\epsilon>0$, should grow at most at an algebraic rate $\epsilon^{-\gamma}$, with quantifiable exponent $\gamma > 0$. As explained in this section, there are at least two potential obstacles to the efficiency of PCA-Net. 

\paragraph{Complexity of the output distribution.}
The first potential reason for the inefficiency of PCA-Net is a consequence of the following lower bound on the approximation error.
\begin{proposition}
\label{prop:lower-bd}
Let $\mu \in \cP(\cX)$ be a probability measure, and let $\Psi^\dagger\in L^2_\mu(\cX;\cY)$ be an operator. Let $\lambda_1 \ge \lambda_2 \ge \dots \ge 0$ be the PCA eigenvalues of the push-forward measure $\Psi^\dagger_\#\mu$ on $\cY$. Then we have the following lower bound,
\begin{align}
\label{eq:lower-bd}
\E_{u\sim \mu}\left[\Vert \Psi^\dagger(u) - \Psi(u) \Vert_{\cY}^2\right]
\ge
 \sum_{j>\dy} \lambda_j, % \cR^\opt_{\dy}(\Psi^\dagger_\#\mu) 
\end{align}
 for any PCA-Net $\Psi$ with PCA dimension $\dy$ on $\cY$.
\end{proposition}
The proof of Proposition \ref{prop:lower-bd} is an almost verbatim repetition of the argument in \citep[Theorem 3.6]{thdeeponet}; we won't repeat the proof here. As a consequence of \eqref{eq:lower-bd}, the approximation error that can be achieved with a PCA dimension $\dy$ is lower bounded by the decay of the PCA eigenvalues of the push-forward measure $\Psi^\dagger_\#\mu$, i.e. by the optimal PCA projection error $\cR^\opt_d(\Psi^\dagger_\#\mu) = \sum_{j>\dy} \lambda_j$. In particular, if this decay is very slow, e.g. $\cR^\opt_d(\Psi^\dagger_\#\mu) \gtrsim \log(\dy)^{-1}$, then an \emph{exponentially} large PCA dimension $\dy(\epsilon) \sim \exp(\epsilon^{-1})$ is required, which entails that also an exponential number of samples $N \gtrsim \exp(\epsilon^{-1})$ and an exponential neural network size $\size(\psi) \gtrsim \exp(\epsilon^{-1})$ are required: \rev{The bound on $N$ follows from the fact that the empirical PCA covariance operator \eqref{eq:empirical-cov} is of rank at most $N$, and hence the determination of $d_\cY$ non-trivial PCA eigendirections requires at least $N\ge d_\cY$ samples. Furthermore, since the neural network $\psi$ must have $d_\cY$ non-trivial output components, at least $d_\cY$ non-zero weights are required.} The first potential obstacle thus relates to the \emph{complexity of the output space}, encoded in the PCA eigenvalue decay of the measure $\Psi^\dagger_\#\mu$. We note in passing that the problem of a slow eigenvalue decay can sometimes be ameliorated by replacing the linear decoder $\cD_\cY$ by a non-linear mapping, leading to improved results both theoretically and empirically \citep{seidman2022nomad,LMHM2022}.

\paragraph{Curse of dimensionality \rev{and parametric complexity.}}

The last section shows that the complexity, or ``size'', of the output space $\cY$ can be one obstacle to operator learning with PCA-Net. A second potential obstacle to efficient operator learning is that the space of operators $\Psi^\dagger: \cX \to \cY$ itself is very large. 

It is well-known that the task of approximating a high-dimensional function $f: \R^d \to \R$ by ordinary methods, such as polynomial approximation, suffers from a curse of dimensionality, where the number of degrees of freedom needed to achieve a desired accuracy scales exponentially in $d$. Indeed, optimal error bounds for interpolation are typically of the form, 
\begin{align}
\label{eq:interp}
\sup_{\xi\in D} |f(\xi) - p(\xi)| \lesssim \Vert f \Vert_{C^k} N^{-k/d},
\end{align}
where $p$ denotes the (e.g. polynomial) interpolant and $N$ represents the number of degrees of freedom of the interpolation space. Such upper bounds deteriorate for large $d$, with convergence rates becoming arbitrarily slow. Similar results have also been established for neural networks, see e.g. \citep{yarotsky_error_2017,achour2022general,siegel2022optimal}, including upper and lower bounds on the required number of weights. In the context of operator learning, the underlying input and output spaces are infinite-dimensional. Given the unfavorable scaling of \eqref{eq:interp} with the dimension $d$, it is thus far from obvious that operator learning should be practicable at all. 

In analogy with the above, we next introduce the notion of an \define{algebraic convergence rate}: Given a set $\mathcal{C} \subset L^2_\mu(\cX; \cY)$ of operators, we say that $\gamma_\cC > 0$ is an algebraic convergence rate for $\mathcal{C}$, if for any $\Psi^\dagger \in \mathcal{C}$, there exist a constant \rev{$C = C(\Psi^\dagger)>0$, depending only on properties of the operator $\Psi^\dagger$,} and a PCA-Net mapping $\Psi := \cD_\cY\circ \psi \circ \cE_\cX: \cX \to \cY$, such that 
\begin{align}
\label{eq:convrate}
\E_{u\sim \mu}\left[\Vert \Psi^\dagger(u) - \Psi(u) \Vert_{\cY}^2 \right] \le C \, \size(\psi)^{-\gamma_\cC},
\end{align}
The central point is that the convergence rate $\gamma_\cC$ in \eqref{eq:convrate} should be uniform over $\cC$, and that only the multiplicative constant $C$ depends on $\Psi^\dagger$, \rev{thus allowing the approximation of any operator $\Psi^\dagger \in \cC$ to a desired accuracy, \emph{at a rate specified by $\gamma_\cC$.}} This is a natural analogue of \eqref{eq:interp}. In \eqref{eq:convrate}, the linear encoder $\cE_\cX: \cX \to \R^{\dx}$ and decoder $\cD_\cY: \R^\dy \to \cY$ are allowed to be arbitrary; In particular, we do not restrict the dimensions $\dx$, $\dy$. Even though the additional dependence on $\dx$ and $\dy$ is also of practical relevance, we postulate that any operator $\Psi^\dagger \in L^2_\mu(\cX;\cY)$, or indeed class of operators $\mathcal{C} \subset L^2_\mu(\cX;\cY)$, which is ``efficiently'' approximated by PCA-Net, must (at the very least) possess a finite algebraic convergence rate $\gamma_\cC$ in the above sense \eqref{eq:convrate}.

The next result shows that for often considered classes of operators $\Psi^\dagger: \cX \to \cY$, such as the set of all Lipschitz continuous or $r$-times Fr\'echet differentiable operators, efficient operator learning by PCA-Net is in fact impossible. \rev{In analogy with the finite-dimensional curse of dimensionality, we refer to this as the ``curse of parametric complexity'':}
\begin{theorem}[Curse of parametric complexity]
\label{thm:cod}
Let $\cX, \cY$ be separable Hilbert spaces, with $\dim(\cX) = \infty$, $\dim(\cY) \ge 1$. Let $\mu \in \cP(\cX)$ be a non-degenerate Gaussian measure. Fix $k\in \N$ and let $\mathcal{C}^k$ denote the set of all $k$-times Fr\'echet differentiable operators $\Psi^\dagger: \cX \to \cY$ whose $k$-th total derivative is uniformly bounded on $\cX$. For \emph{any} $\gamma > 0$, there exists $\Psi^\dagger_\gamma\in \cC^k$ and a constant $c_\gamma > 0$, such that 
\[
\E_{u\sim \mu}\left[
\Vert \Psi^\dagger_\gamma(u) - \Psi(u) \Vert_{\cY}^2
\right]
\ge c_\gamma \rev{\size(\psi)^{-\gamma}},
\]
for any PCA-Net $\Psi = \cD_\cY \circ \psi \circ \cE_\cX$. In particular, there \emph{cannot} exist a finite algebraic convergence rate $\gamma_{\cC}$ for $\cC^k$, as in \eqref{eq:convrate}.
\end{theorem}

\rev{
\begin{remark}
Theorem \ref{thm:cod} makes precise a notion of ``curse of dimensionality'' in operator learning with PCA-Net, which was anticipated in earlier work by \cite{thfno,thdeeponet}. As will be explained, this result can indeed be interpreted as a scaling limit of the conventional curse of dimensionality. In contrast to the approximation of a high-dimensional function in $d\gg 1$ dimensions, the goal of operator learning is to approximate a mapping defined on an infinite-dimensional space, where $d = \infty$. The ``curse of parametric complexity'' can be viewed as a scaling limit of the finite-dimensional curse of dimensionality, when $d \to \infty$.
\end{remark}

\begin{remark}
The underlying input/ouptut spaces $\cX$, $\cY$ are fixed in our setting. In particular, if $\cX$ is a function space e.g. consisting of functions $u: D \to \R$ for some domain $D\subset \R^n$, then the dimension $n$ of this domain $D$ is fixed. The curse of parametric complexity identified in Theorem \ref{thm:cod} is independent of the presence (or absence) of an additional curse of dimensionality that may arise due to the (potentially large) dimension $n$ of the domain. 
\end{remark}
}

\rev{
\begin{remark}
To simplify the statement of Theorem \ref{thm:cod}, it is assumed that $\mu$ is a non-degenerate Gaussian. Inspection of the proof shows that the non-degeneracy requirement could be relaxed to only assuming that the covariance operator $\Sigma = \E_{u\sim \mu}[u\otimes u]$ have infinitely many non-trivial eigendirections (with eigenvalue $\lambda_j \ne 0$), thus allowing for some degeneracy. The crucial assumption is that the input distribution $\mu$ be ``truly infinite-dimensional''. 
\end{remark}
}

\begin{proof}[Sketch of proof]
The proof of Theorem \ref{thm:cod} is based on the following lower bound for ReLU neural network approximation of functions in the unit cube of the Sobolev space $W^{k,\infty}([0,1]^d)$,
\[
F_{k,d} := \set{f\in W^{k,\infty}([0,1]^d)}{\Vert f \Vert_{W^{k,\infty}} \le 1},
\]
which is derived in the present work and may be of independent interest:
\begin{restatable}{proposition}{PropLower}\label{prop:lower2}
Fix $k,d\in \N$. There exists $c_{k,d}>0$ depending only on $k$ and $d$, $f\in F_{k,d}$ and an absolute constant $\lambda > 0$ independent of both $d$ and $k$, such that for any neural network $\psi$, we have the lower bound
\begin{align}
\label{eq:lower2}
\Vert f - \psi \Vert_{L^2([0,1]^d)} \ge c_{k,d} \, \size(\psi)^{-\lambda k/d}.
\end{align}
\end{restatable}
This finite-dimensional lower bound \eqref{eq:lower2} builds on the recent work \citep{achour2022general}. Given the lower bound of Proposition \ref{prop:lower2}, the proof of Theorem \ref{thm:cod} is then based on the intuition that $d$ can be chosen arbitrarily large if the underlying input space $\cX$ is infinite-dimensional, and hence the exponent in \eqref{eq:lower2} can be arbitrarily small; additional work is needed to make this intuition rigorous. Detailed proofs of Theorem \ref{thm:cod} and Proposition \ref{prop:lower2} are given in Appendix \ref{app:cod}.
\end{proof} 

In particular, Theorem \ref{thm:cod} shows that it is \emph{impossible} to derive algebraic error and complexity estimates for PCA-Net, when e.g. assuming only Lipschitz regularity of $\Psi^\dagger$. It should be emphasized that this result holds even when the relevant space of input and output functions can be efficiently approximated; indeed, no restriction on the decay of the PCA eigenvalues is assumed in Theorem \ref{thm:cod}, allowing them to decay at an arbitrarily fast rate on both $\cX$ and $\cY$. In fact, Theorem \ref{thm:cod} even allows for $\cY = \R$, in which case reconstruction on $\cY$ is \emph{trivial}  (in contrast, the assumptions do imply that infinitely many PCA eigenvalues on $\cX$ are non-zero). As alluded to above, the reason for the obstacle to efficient operator learning expressed by Theorem \ref{thm:cod} is  the intrinsic complexity of the space of all $\mathcal{C}^k$-regular operators (or functionals), defined on an infinite-dimensional input space $\cX$. 
\begin{remark}
Under slightly stronger assumptions than Theorem \ref{thm:cod} (in particular, assuming an algebraic decay of the PCA eigenvalues $\lambda_j$), one can likely show that there in fact exists $\Psi^\dagger \in \mathcal{C}^k$ and constants $c,\gamma >0$, depending only on $\mu$ and $k$, such that 
\[
\E_{u\sim \mu}\left[
\Vert \Psi^\dagger(u) - \Psi(u) \Vert_{\cY}^2
\right]
\ge c \log\left(\size(\psi)\right)^{-\gamma}.
\]
for any PCA-Net $\Psi = \cD_\cY \circ \psi \circ \cE_\cX$. Thus, achieving accuracy $\epsilon$ requires an \emph{exponential complexity} of the underlying neural network, $\size(\psi) \gtrsim \exp(c\epsilon^{-1/\gamma})$. Similar exponential lower bounds will be the subject of forthcoming work \citep{LS2023}.
\end{remark}
Given the negative result of Theorem \ref{thm:cod}, we posit that a central challenge in operator learning is to \emph{identify} the relevant class of operators $\mathcal{C} \subset L^2_\mu(\cX;\cY)$, which allow for efficient approximation by a given operator learning framework, and which possess a prescribed (finite) algebraic convergence rate $\gamma_\cC$ in \eqref{eq:convrate}.

\subsection{Quantitative encoding error bounds}
\label{sec:quant-enc}
As pointed out above, one potential obstacle to the efficacy of PCA-Nets is a slow decay of the PCA eigenvalues on the output function space. In this section, we provide a general estimate on the PCA projection error, $\cR^\opt_d(\nu) = \sum_{\ell > d} \lambda_\ell$, based on smoothness properties of the underlying functions. Given a domain $D\subset \R^n$, we recall that the Sobolev space $H^s(D;\R^{n})$ is defined as the space of all functions $u: D \to \R^{n}$ possessing $L^2$-integral weak derivatives up to order $s$. We note that $H^s(D;\R^{n})$ is a Hilbert space. We then have:
\begin{proposition}
\label{prop:PCAbound}
Let $n,n'$ be integers. Let $\cY = H^s(D; \R^{n'})$, $s\ge 0$, be a Sobolev space defined on either a Lipschitz domain $D\subset \R^n$ or the periodic torus $D = \T^n$. Assume that $\nu \in \cP(\cY)$ is a probability measure and that there exists $\zeta>0$, such that $\E_{u\sim\nu}\left[\Vert u \Vert_{H^{s+\zeta}}^2\right] < \infty$. Then, there exists a constant $C = C(n,n')>0$, such that
\[
\cR^\opt_d(\nu) \le C d^{-2\zeta/n} \E_{u\sim\nu}\left[ \Vert u \Vert_{H^{s+\zeta}}^2 \right], \quad \forall \, d\in \N.
\]
We recall that $\cR^\opt_d(\nu)$ is the minimal projection error over all projections $P$ of rank $d$ (cp. \ref{eq:Rd}).
\end{proposition}

Thus, smoothness of the output functions ensures an algebraic decay of the PCA eigenvalues and thus rules out the first potential obstacle to efficient operator learning by PCA-Net, pointed out after Proposition \ref{prop:lower-bd}.
A proof of Proposition \ref{prop:PCAbound} is provided in Appendix \ref{app:PCAbound}. Proposition \ref{prop:PCAbound} extends a result in \citep[Prop. 3.14]{thdeeponet}, which was restricted to the periodic case. The main novel ingredient in our proof here, and which allows the generalization to arbitrary Lipschitz domains $D$, is based on a general Sobolev extension result by Stein \citep[e.g. Appendix B]{thfno}.

\subsection{Overcoming the curse of dimensionality}
\label{sec:overcoming}

The last section provides a general criterion to tame the lower bound of Proposition \ref{prop:lower-bd}, which could limit the efficiency of PCA-Net when the relevant distribution of outputs in $\cY$ is very complex (cf. Section \ref{sec:quant}). As proved in Theorem \ref{thm:cod}, a second obstacle to efficient operator learning is the \emph{curse of dimensionality}. It would be very desirable to find a useful mathematical characterization of the entire class of operators $\cC \subset L^2_\mu(\cX;\cY)$ for which an algebraic convergence rate, as in \eqref{eq:convrate}, can be established. At present, this appears a very distant goal. Instead, in this section, we aim to develop additional intuition of the basic mechanisms that PCA-Net can exploit to achieve algebraic convergence rates for specific examples. To this end, we consider two prototypical operators arising in the context of PDEs; the Darcy flow and Navier-Stokes equations. 

\subsubsection{Darcy flow}
\label{sec:darcy}
Let $D \subset \R^n$ be a bounded domain. We consider the operator $\Psi^\dagger: a \mapsto w$, mapping the coefficient field $a$ to the solution $w$ of the following elliptic problem:
\begin{align} \label{eq:darcy}
\left\{
\begin{aligned}
-\nabla \cdot \left( a(x) \nabla w(x) \right) &= f(x), &&(x \in D),\\
w(x) &= 0, &&(x\in \partial D).
\end{aligned}
\right.
\end{align}
Here,  $D \subset \R^n$ is a given smooth domain, the right-hand side $f$ is fixed, and we assume Dirichlet boundary conditions. Let $H^1_0(D)$ be the Sobolev space consisting of weakly differentiable functions $w: D \to \R$ which vanish on the boundary $\partial D$, and whose gradient is square-integrable. It is well-known, e.g. \citep[Chapt. 6]{evans2022partial}, that if the coefficient field $a\in L^\infty(D)$, satisfies two-sided (coercivity and upper) bounds 
\begin{align}
\label{eq:twosided}
0 < \lambda \le a(x) \le \Lambda < \infty, \quad \forall \, x\in D,
\end{align}
for constant $\lambda, \Lambda$,
and if $f\in H^{-1}(D)$ belongs to the dual space of $H^1_0(D)$, then there exists a unique solution $w\in H^{1}_0(D)$ to \eqref{eq:darcy}. Furthermore, there exists a constant $C = C(\lambda,\Lambda,D)>0$, such that
\begin{align}
\label{eq:apriori}
\Vert w \Vert_{H^{1}_0} \le C \Vert f \Vert_{H^{-1}}.
\end{align}
It is thus natural to consider the output space $\cY = H^1_0(D)$. We will follow the celebrated work by Cohen, Devore and Schwab \citep{cohen2010convergence}, and subsequent extensions in \citep{CDS2011,SchwabZech2019,OSZ2019,OSZ2020}, and consider the following setting: We assume that the underlying measure $\mu$ can be written as the law of random coefficient fields $a(x) = a(x;\bm{z})$, of the parametrized form
\begin{align} \label{eq:expansion}
a(x;\bm{z}) = \bar{a}(x) + \sum_{\ell=1}^\infty \gamma_\ell z_\ell \rho_\ell(x),
\end{align}
where $\bm{z} = (z_1,z_2,\dots)$ is a sequence of random variables (not necessarily independent), such that $|z_\ell|\le 1$. In our analysis, we will assume that the functions $\rho_1,\rho_2,\dots\in \cX$ are orthonormal in a Hilbert space $\cX$, where $\cX \embeds L^\infty$ possesses a continuous embedding. Note that the series \eqref{eq:expansion} converges, if 
\begin{align}
\label{eq:alpha}
0 \le \gamma_\ell \le M \ell^{-1-\alpha},
\end{align}
decays at an algebraic rate with constants $M, \alpha>0$. We will assume this bound \eqref{eq:alpha}. To ensure coercivity (cp. Remark \ref{rem:coerc} below) we assume that there exists $\kappa > 0$, such that 
\begin{align}
\label{eq:kappa}
\sum_{\ell=1}^\infty \gamma_\ell \Vert \rho_\ell \Vert_{L^\infty} \le \frac{\kappa}{1+\kappa} \bar{a}_{\mathrm{min}},
\end{align}
with $\bar{a}_{\mathrm{min}} := \essinf_{x\in D} \bar{a}(x) > 0$. 
\begin{remark}
\label{rem:coerc}
The upper bound \eqref{eq:kappa} ensures uniform coercivity, since
\[
a(x;\bm{z}) \ge \bar{a}_{\mathrm{min}} - \sum_{\ell=1}^\infty \gamma_\ell \Vert \rho_\ell \Vert_{L^\infty} \ge \frac{1}{1+\kappa} \bar{a}_{\mathrm{min}} =: \lambda >0,
\]
for any $\bm{z} \in U := [-1,1]^\N$. On the other hand, the assumed embedding $\cX \embeds L^\infty(D)$ implies that there exists a constant $C>0$, such that $\Vert \slot \Vert_{L^\infty} \le C \Vert \slot \Vert_{\cX}$, and hence \eqref{eq:alpha} ensures a uniform upper bound,
\begin{align*}
a(x;\bm{z}) 
&\le 
\Vert \bar{a} \Vert_{L^\infty} + CM\sum_{\ell=1}^\infty \ell^{-1-\alpha} 
=: \Lambda < \infty.
\end{align*}
In particular, all $a$ in the support of the probability measure $\mu$ satisfy the two-sided bounds \eqref{eq:twosided}, and hence the elliptic PDE \eqref{eq:darcy} is well-posed.
\end{remark}

\begin{remark}
We do not assume any \emph{(explicit) knowledge} of the functions $\bar{a}$, $\rho_\ell$, the parameters $\gamma_\ell$, $\alpha$, $\lambda$, $\Lambda$ or indeed any information on the law of the joint random variable $\bm{z} = (z_1,z_2,\dots) \in [-1,1]^\N$. In particular, the  random variables $z_1, z_2,\dots$ need not be independent. For the following arguments, it is sufficient that an expansion of the form \eqref{eq:expansion} exists.
\end{remark}

Given this setting, we prove the following theorem:

\begin{theorem} \label{thm:darcy}
Under the setting and the prevailing assumptions of this section and with $\mu\in \cP(\cX)$ the law of $a(\slot;\bm{z})$ given by \eqref{eq:expansion}. For any $\delta,\eta > 0$ and $\epsilon > 0$, there exists a PCA-Net $\Psi = \cD_\cY \circ \psi \circ \cE_\cX$ satisfying the error bound,
\[
\E_{a\sim \mu}\left[\Vert \Psi^\dagger(a) - \Psi(a) \Vert^2_{H^1_0} \right]
\le 
C\epsilon,
\]
with probability at least $1-\delta$, and with constant $C = C(\mu,\eta)>0$ depending only on $\mu$ and $\eta$. With the same implied constant, the required PCA dimensions are at most $\dx = \dy = d \sim \epsilon^{-\frac1{2\alpha} - \eta}$ and the required number of samples for PCA is at most $N \sim d^{1+4\alpha} \log(1/\delta)$. Furthermore, the following complexity bounds hold for the ReLU network $\psi$,
\[
\size(\psi) \le C \epsilon^{-\frac{1}{\alpha} - \eta },
\quad
\depth(\psi) \le C \log(\epsilon^{-1})^2,
\]
independently of the data $a_1,\dots, a_N\sim \mu$.
\end{theorem}

We provide a sketch of the proof of Theorem \ref{thm:darcy} at the end of this subsection. 
\rev{
\begin{remark}
Regarding the fudge factor $\eta>0$, we note that in our estimates leading to Theorem \ref{thm:darcy}, the constant $C=C(\mu,\eta)$ blows up as $\eta \to 0$, i.e. $C(\mu,\eta)\to \infty$. Therefore, while $\eta>0$ can be chosen arbitrarily small, our estimates break down when $\eta = 0$.
\end{remark}
}

\begin{remark}
Theorem \ref{thm:darcy} shows that approximation of the Darcy flow operator $\Psi^\dagger$ is possible with algebraic bounds on the PCA dimensions $\dx$, $\dy$, the number of required PCA samples $N$, and the number of neural network parameters, $\size(\psi)$. In fact, even under a mild decay rate $\alpha > 1$, the required size of $\psi$ is at most \emph{linear} in the desired accuracy.
\end{remark}

 We conjecture that the scaling of $N \sim d^{1+4\alpha}$ in Theorem \ref{thm:darcy} is highly pessimistic; indeed, under a potential improvement of the empirical PCA estimate of Proposition \ref{prop:pcaproj}, as discussed in Remark \ref{rmk:pcaproj}, the much more favorable scaling $N \sim d$ appears natural.

\paragraph{Sketch of proof of Theorem \ref{thm:darcy}.}
To prove Theorem \ref{thm:darcy}, we first define a parametric mapping $\cF: U \to \cY$, where $U=[-1,1]^\N$, $\cY = H^1_0(D)$, by 
\begin{align}
\label{eq:F}
\cF(\bm{z}) := \Psi^\dagger\left(\bar{a} + \textstyle\sum_{\ell=1}^\infty \gamma_\ell z_\ell \rho_\ell\right).
\end{align}
This mapping has been studied in a series of papers \citep{cohen2010convergence,CDS2011,SchwabZech2019,OSZ2019,OSZ2020}, and is known to allow for a convergent Taylor series expansion (in the variables $\bm{z}\in U$). To state the next Lemma, which implies a suitable convergence of Taylor series, we recall that $\bm{\nu}$ is called a multi-index in this infinite-dimensional context, if $\bm{\nu}=(\nu_1,\nu_2,\dots)$ is a sequence of non-negative integers, such that $\nu_\ell =0$ for almost all $\ell$. For $\bm{z} \in U$ and a multi-index $\bm{\nu}$, we define the monomial $\bm{z}^{\bm{\nu}} = \prod_{\nu_\ell\ne 0} z_\ell^{\nu_\ell}$. The following lemma follows immediately from \citep[Thm. 1.3]{CDS2011}.
\begin{lemma}
\label{lem:Taylor}
Assume the setting and prevailing assumptions of this section. Then there exists a set of coefficients $t_{\bm{\nu}} \in \cY$ (the Taylor coefficients), indexed by multi-indices $\bm{\nu}$, such that for any $m \in \N$, there is a set $\Lambda_m$ of multi-indices ${\bm{\nu}}$, with cardinality $|\Lambda_m|=m$, such that
\begin{align}
\label{eq:Taylor}
\sup_{\bm{z} \in U} \Vert \cF(\bm{z}) - \sum_{\bm{\nu} \in \Lambda_m} t_{\bm{\nu}} \bm{z}^{\bm{\nu}} \Vert_{\cY}
\le C m^{-\alpha + \eta},
\end{align}
for any fixded constant $\eta>0$. Here $C = C(\cF,\mu,\eta)>0$ is a constant depending on $\eta$, but is independent of $m$.
\end{lemma}
For completeness, we provide the details in Appendix \ref{app:Taylor}.
As a consequence of this lemma, we can estimate the optimal PCA encoding errors not only on $\cX$ but also on $\cY$. Indeed, we will derive the following result.
\begin{restatable}{proposition}{darcyencerror}
\label{prop:Darcy-enc-err}
Under the setting and prevailing assumptions of this section. For any $\eta > 0$, there exists a constant $C = C(\cF,\mu,\eta)>0$, such that
\begin{align*}
\cR^\opt_\dx(\mu) 
\le 
C\dx^{-2\alpha-1},
\quad 
\cR^\opt_\dy(\Psi^\dagger_\#\mu)
\le 
C \dy^{-2\alpha+\eta}.
\end{align*}
\end{restatable}
The $\cX$ estimate of Proposition \ref{prop:Darcy-enc-err} is a straightforward consequence of the assumed expansion \eqref{eq:expansion}. The $\cY$ estimate follows from the observation that \eqref{eq:Taylor} provides a bound of the PCA projection error by comparing it to the projection onto $\Span\set{t_{\bm{\nu}}}{{\bm{\nu}} \in \Lambda_\dy} \subset \cY$. Details of the required argument are given in Appendix \ref{app:Darcy-enc-error}. Since the input measure $\mu$ and its push-forward $\Psi^\dagger_\#\mu$ are concentrated on a bounded set of functions in the respective norms on $\cX$ and $\cY$, Proposition \ref{prop:pcaproj} immediately implies that empirical PCA with $\dx = \dy = d$ and with a sufficient number of $N \gtrsim d^{1+4\alpha}\log(\delta/2)$ samples achieves, up to a constant and with high probability, the same asymptotic error as the optimal PCA projection on $\cX$ and $\cY$.

The main remaining challenge is then to construct a neural network $\psi$, such that the composition $\Psi = \cD_\cY \circ \psi \circ \cE_\cX$ approximates $\Psi^\dagger$ to within a prescribed tolerance. The construction of such $\psi$ relies on a neural network approximation result for the parametric mapping $\cF(\bm{z})$. The following result follows from \citep[Theorem 4.11]{OSZ2020} (the present statement is closer in formulation to \citep[Theorem 3.9]{SchwabZech2019}, and has appeared with slightly sharper bounds in \citep{thdeeponet}):
\begin{lemma} \label{lem:be-expansion}
Under the prevailing assumption and with $\cF$ defined by \eqref{eq:F}. Let $\eta>0$ be a small (fudge) constant. There exists a constant $C=C(\cF, \mu, \eta)>0$, depending only on $\cF$, on the decay rate $\alpha$ and on $\eta$, such that for any PCA encoder $\cE_\cY: \cY \to \R^{d_{\cY}}$, and for every $m\in \N$, there exists a ReLU network $\psi^\star: \R^m \to \R^{d_{\cY}}$, $z\mapsto \psi^\star(z_1,\dots, z_m)$ with
\begin{align} \label{eq:be-expansion}
\sup_{\bm{z}\in [-1,1]^\N} 
\left\Vert
\cE_\cY \circ \cF(\bm{z}) - \psi^\star(z_{1},\dots, z_{m})
\right\Vert_{\ell^2} 
\le C m^{-\alpha+\eta},
\end{align}
and such that
$\size(\psi^\star) \le Cm \left(\log(m)^2 + d_{\cY}\right)$ and $\depth(\psi^\star) \le C\log(m)^2$.
\end{lemma}
The main additional problem in the present context of PCA-Net is that the neural network $\psi$ in the definition of a PCA-Net, $\Psi = \cD_\cY \circ \psi \circ \cE_\cX$, acts on the \emph{encoded input} $\cE_\cX(a)$, with encoder $\cE_\cX$ obtained from empirical PCA. Therefore, we cannot directly access the coefficients $z_1,z_2,\dots$ in the parametric expansion \eqref{eq:expansion} of $a$, and indeed there is no way to exactly recover these coefficients from the PCA encoding of $a$. A careful discussion of this ``compatibility issue'' between the PCA encoding and the a priori expansion \eqref{eq:expansion} is therefore necessary. This is the main issue addressed in Appendix \ref{app:darcy}, where we show that such $\psi$ can indeed be constructed and the additional error due to the incompatibility can be controlled. This then leads to the statement of Theorem \ref{thm:darcy}. 
%\qed

\subsubsection{Navier-Stokes equations}
\label{sec:ns}

In the previous section, we showed that the solution operator $\Psi^\dagger$ of the Darcy flow equations can be efficiently approximated by PCA-Net. With some additional effort, this result could likely be extended to a more general class of so-called $(\bm{b},\epsilon)$-holomorphic operators, e.g. \citep[Section 2.1]{SchwabZech2019}. In particular, the relevant underlying class of operators is here characterized by \emph{analytic regularity}; this assumption goes well beyond $\cC^k$-regularity, and thereby, PCA-Net can overcome the general curse of dimensionality of Theorem \ref{thm:cod}, for this restricted class of operators.

However, many operators of interest, in particular in the context of advection dominated problems such as hyperbolic PDEs, are not holomorphic in this sense. In the present section, we therefore discuss another mechanism by which polynomial complexity estimates can be obtained, even in the absence of holomorphy: namely, PCA-Net can efficiently emulate numerical methods. Following ideas developed in \citep{thdeeponet,thfno}, and starting from a known (and convergent) numerical method, such an emulation result provides an upper bound on the required complexity of $\Psi$, by showing that a specific choice of the neural network weights can emulate the given numerical method. The derivation of explicit estimates requires us to fix a particular numerical method for the analysis,  but the intuition behind these emulation results is that $\Psi$ can, in principle, efficiently emulate a very rich class of numerical methods. This includes methods with high convergence rates, which the neural network can explore during optimization; this expressive power of neural networks can thus provide a theoretical rationale for their efficiency within the PCA-Net methodology. 

In the following, we will focus on an emulation result for the Navier-Stokes equations, based on spectral methods. The underlying idea is similar to a recent emulation result for Fourier neural operators \citep{thfno}; however, while Fourier neural operators very naturally (and by design) provide the necessary ingredients to build a spectral method, the PCA-Net methodology requires an extension of the results of \citep{thfno}, including a detailed discussion of compatibility of the PCA-projection with such emulation results. In addition, at a more technical level, it is here shown that the smoothness assumption on the activation function, which was \emph{essential} in all proofs of \citep{thfno}, can be substantially relaxed. Indeed, the results of the present work are based on the popular ReLU activation function $\sigma(x) = \max(x,0)$, leading to comparable complexity estimates as in \citep{thfno} differing only by $\log$-factors.

To illustrate the general approach, we consider the periodic Navier-Stokes equations in spatial dimension $n=2$, over a fixed time interval $[0,T]$:
\begin{align} \label{eq:ns}
\left\{
\begin{gathered}
\partial_t u + u\cdot \nabla u + \nabla p = \nu \Delta u, \\
\div(u) = 0, \, u(t=0) = \bar{u}.
\end{gathered}
\right.
\end{align}
Here $u: \T^n\times [0,T] \to \R^n$ is the flow vector field, $p: \T^n \times [0,T] \to \R$ is the scalar pressure and $\nu\ge 0$ is the viscosity (we allow $\nu=0$, corresponding to the incompressible Euler equations). We have denoted by $\bar{u}: \T^n\to \R^n$ the (divergence-free) initial data. Since the solution operator $\Psi^\dagger$ associated with \eqref{eq:ns} is only known to be well-defined in two spatial dimensions, we focus on this case. We however point out that all results readily extend to the three-dimensional case, under additional (unproven) smoothness assumptions. 

\begin{remark}
Classical well-posedness results for the Navier-Stokes ($\nu > 0$) and Euler ($\nu = 0$) equations, e.g. \citep{majda2001vorticity} and references therein, imply in the two-dimensional case $n=2$, that if the random initial data $\bar{u} \sim \mu$, is uniformly bounded $\Vert \bar{u}\Vert_{H^r}\le \bar{M}$ for $r>n/2+1$, then the corresponding solution $u(t)$ at a fixed later time $t\in [0,T]$ satisfies a similar bound
$\Vert u(t) \Vert_{H^r} \le M$, 
for some $M = M(T,\bar{M})$.
\end{remark}

The main result of the present section is Theorem \ref{thm:ns}, below:

\begin{theorem} \label{thm:ns}
Consider the two-dimensional, periodic Navier-Stokes equations. Fix parameters $M,T>0$, and integer $r>n/2+1$. Assume that $\mu\in \cP(L^2(\T^2;\R^2))$ is a probability measure on initial data $\bar{u}$ of \eqref{eq:ns}, such that $\Vert \bar{u} \Vert_{H^r} \le M$ $\mu$-almost surely. Let $\Psi^\dagger$ be the forward solution operator of the Navier-Stokes equations, mapping initial data to the solution at the final time $T$, $\Psi^\dagger: u(0) \mapsto u(T)$. 
For any $\epsilon, \delta > 0$, there exists a PCA-Net $\Psi = \cD_\cY \circ \psi \circ \cE_\cX$, such that 
\[
\E_{u\sim \mu}
\left[
\Vert \Psi^\dagger(u) - \Psi(u) \Vert^2_{L^2_x}
\right]
\le 
C
\epsilon,
% \epsilon^2 + C \left\{
% d_\cX^{-2r/d} + d_\cY^{-2r/d} + \left[\frac{\dx}{J}\right]^{1/2} + \left[\frac{\dy}{J}\right]^{1/2}
% \right\}.
\]
with probability at least $1-\delta$, and with a constant $C = C(M,r,T)>0$. The PCA dimensions $\dx$ and $\dy$ are bounded by $\dx, \dy \le C \epsilon^{-1/r}$, the requisite amount of data $N \le C\dx^{1+2r} \log(1/\delta)$ and the neural network $\psi$
satisfies the complexity bounds
\[
\left\{
\begin{aligned}
\size({\psi}) &\le C \epsilon^{-1/r}\left( \epsilon^{-1/2} \log(\epsilon^{-1})^2 + \epsilon^{-1/r} \right),
\\
\depth({\psi}) &\le C \epsilon^{-1/2} \log(\epsilon^{-1})^2.
\end{aligned}
\right.
\]
\end{theorem}

A sketch of the proof of Theorem \ref{thm:ns} is provided below.

\begin{remark}
In fact, $\psi$ in Theorem \ref{thm:ns} can be written as an $n_T$-fold composition 
\[
{\psi} = Q\circ \underbrace{ {\psi}_\ast \circ \dots \circ {\psi}_\ast }_{n_T \text{-fold}} \circ R,
\]
where the mappings $R: \R^{d_\cX} \to \R^{d_H}$, $Q: \R^{d_H} \to \R^{d_\cY}$ are linear (input and output layers), $n_T \le C \epsilon^{-1/2}$, and ${\psi}_\ast: \R^{d_H} \to \R^{d_H}$ is a ReLU neural network with ``hidden layer'' dimension $d_H \le C\epsilon^{-1/r}$, such that
\[
\left\{
\begin{aligned}
\size({\psi}_\ast) &\le C \epsilon^{-1/r} \log(\epsilon^{-1})^2,
\\
\depth({\psi}_\ast) &\le C \log(\epsilon^{-1})^2.
\end{aligned}
\right.
\]
\end{remark}

\paragraph{Sketch of proof of Theorem \ref{thm:ns}.}

The derivation of the quantitative error and complexity bounds of Theorem \ref{thm:ns} is based on an \emph{emulation result}; the idea is to show that for any choice of PCA bases, there exists a ReLU neural network $\psi$ which can efficiently emulate a numerical method which is known to converge at a precisely quantifiable rate. The complete details of the required argument, including all proofs, will be provided in Appendix \ref{app:ns}. Here, we instead give a general overview of the main ideas, to aid intuition. 

As a first step towards this emulation result, we review a convergent spectral scheme in section \ref{sec:scheme}. Then, we construct a ReLU neural network emulation of this spectral scheme in section \ref{sec:scheme-emulation}, leading to Algorithm \ref{alg:emulation} and  the neural network size estimates of Lemma \ref{lem:nn-representation}. The constructed neural network emulator defines a mapping from the truncated Fourier coefficients of the initial data $u(0)$, to (an approximation of) the truncated Fourier coefficients of the solution $u(T)$. The relevant set of truncated Fourier modes with cut-off parameter $K\in \N$ is given by 
\[
\cK = \cK_K := \set{k = (k_1,k_2)\in \Z^2}{|k|_\infty := \max(|k_1|,|k_2|)\le K}.
\]
The mapping on these truncated Fourier coefficients defines a mapping between two-finite dimensional Euclidean spaces, upon identifying $\C^{\cK} \simeq \R^{2\cK}$. This mapping can be represented by an ordinary neural network. Given this construction, we then proceed to analyze the approximation error of the neural network emulation in section \ref{sec:scheme-error}, leading to the following proposition. This is our core emulation result, essentially stating the fact that ReLU neural networks can indeed efficiently emulate the underlying spectral scheme. 

\begin{proposition}\label{prop:ns-nn}
Let $M, \, r, \, T > 0$ be given. For any $\epsilon > 0$, there exists $K\in \N$, $K\sim \epsilon^{-1/r}$, and a ReLU neural network $\hat{\psi}: \C^{\cK} \to \C^{\cK}$, $\cK = \cK_K$, such that 
\[
\Vert \hat{\psi}(\hat{u}(0)) - \hat{u}(T) \Vert_{\ell^2}
\le
\epsilon, 
\qquad
\text{whenever } \, \Vert u(0) \Vert_{H^r}\le M,
\]
where $\hat{u}(t) = \{\hat{u}_k(t)\}_{|k|_\infty\le K}$ denote the Fourier coefficients of the solution $u(t)$ of \eqref{eq:ns}, with initial data $u(0)$. Furthermore, we have the following complexity estimates:
\[
\left\{
\begin{aligned}
\size(\hat{\psi}) &\le C \epsilon^{-2/r - 1} \log(\epsilon^{-1})^2,
\\
\depth(\hat{\psi}) &\le C \epsilon^{-1} \log(\epsilon^{-1})^2.
\end{aligned}
\right.
\]
The constant $C=C(M,T,r)>0$ depends only $M, T, r$, but is independent of $\epsilon$. Furthermore, $\hat{\psi}$ can be written as a $n_T$-fold composition $\hat{\psi} = \hat{\psi}_\ast \circ \dots \circ \hat{\psi}_\ast$, where $\hat{\psi}_\ast: \C^\cK \to \C^{\cK}$ is a ReLU neural network with 
\[
\left\{
\begin{aligned}
\size(\hat{\psi}_\ast) &\le C \epsilon^{-2/r} \log(\epsilon^{-1})^2,
\\
\depth(\hat{\psi}_\ast) &\le C \log(\epsilon^{-1})^2,
\end{aligned}
\right.
\]
and $n_T \le C \epsilon^{-1}$, corresponding to the number of time-steps of the underlying scheme.
\end{proposition}
In fact, Proposition \ref{prop:ns-nn} is the two-dimensional case of a general $d$-dimensional result derived in the appendix (cf.  Proposition \ref{prop:nsd-nn}). Given this neural network emulation result, the remaining issue is that the empirical PCA encoder $\cE_\cX$ and decoder $\cD_\cY$ do \emph{not} act on Fourier coefficients. And hence, additional work is necessary to suitably adapt the construction of the neural network in Proposition \ref{prop:ns-nn}, ultimately resulting in an efficient PCA-Net approximation $\Psi = \cD_\cY \circ \psi \circ \cE_\cX$ in Section \ref{sec:PCA-error}. We summarize the result in the following lemma:
\begin{lemma}
\label{lem:NS-approx-err}
For any $\epsilon > 0$, there exists a PCA-Net $\Psi = \cD_\cY \circ \psi \circ \cE_\cX$, such that 
\begin{align*}
\E_{u\sim \mu}\left[
\Vert \Psi(u) - \Psi^\dagger(u) \Vert_{L^2(\mu)}^2
\right]^{1/2}
&\le
C \epsilon + C\E_{u\sim \mu}\left[\Vert u - \cD_\cX \circ \cE_\cX(u) \Vert_{L^2}^2\right]^{1/2} 
\\
&\qquad + \E_{v \sim \Psi^\dagger_\#\mu}\left[\Vert v - \cD_\cY \circ \cE_\cY(v) \Vert_{L^2}^2\right]^{1/2},
\end{align*}
where $\psi$ is a neural network of size 
\[
\size(\psi) \le C \epsilon^{-2/r} \left( \epsilon^{-1} \log(\epsilon^{-1})^2 + (d_\cX + d_\cY) \right), 
\quad
\depth(\psi) \le C \epsilon^{-1} \log(\epsilon^{-1}),
\]
and $C = C(M,r,T) > 0$ is a constant independent of $\epsilon, d_\cX, d_\cY$.
\end{lemma}
Finally, the smoothness bound on the PCA eigenvalues of Proposition \ref{prop:PCAbound} can be used to estimate the PCA projection errors with high probability, see Lemma \ref{lem:NSd-enc-err} for details. Combining the above lemma with the estimates on the PCA-projection errors, results in Theorem \ref{thm:ns}. 
%\qed

\section{Conclusion}
\label{sec:conclusion}

PCA-Net is a data-driven operator learning methodology introduced in \citep{pca,hesthaven_nonintrusive_2018}, which (i) uses PCA to reduce the dimensions of the input and output spaces and (ii) uses neural networks to approximate a map between the resulting finite-dimensional latent spaces. 
The main aim of the present work is to develop relevant approximation theory for PCA-Net. 

Our first main result is a novel universal approximation theorem for PCA-Net, Theorem \ref{thm:universal}. Compared to previous work \citep{pca}, this theorem establishes universality under significantly relaxed conditions on the distribution of the data-generating measure and the underlying operator $\Psi^\dagger$. The present assumptions are in fact minimal conditions to ensure that PCA is well-defined on the input and output spaces. 

The next main contribution of the present work is a detailed discussion of two potential obstacles to efficient operator learning with PCA-Net in Section \ref{sec:quant}; the first obstacle relates to the complexity of the output distribution. The second obstacle relates to the inherent complexity of the space of operators between infinite-dimensional input and output spaces, and  gives rigorous meaning to the notion of a curse of dimensionality; Theorem \ref{thm:cod} shows that is impossible to derive algebraic complexity bounds, when considering general classes of operators, such as the class of all Lipschitz- or even $\cC^k$-continuous operators. Hence, we conclude that at this level of generality, the curse of dimensionality is inevitable. 

Given this negative result demonstrating the curse of dimensionality over general classes of operators, we posit that a central challenge in the approximation theory of PCA-Net (and other operator learning methodologies) is to identify and characterize operators, and classes of such operators, which allow for efficient approximation by the methodology. To obtain first insight into this problem for PCA-Net, we focus our attention on two prototypical PDE operators of interest, arising from the Darcy flow and Navier-Stokes equations. In both cases, we show that PCA-Net can overcome the general curse of dimensionality, establishing algebraic error and complexity estimates in Theorems \ref{thm:darcy} and \ref{thm:ns}. This demonstrates that these operators belong to a restricted class which is efficiently approximated by PCA-Net. In the case of Darcy flow, our proof relies on the analytic regularity (holomorphy) of the underlying operator. For the Navier-Stokes equations, we rely on an emulation result, showing that PCA-Net can emulate a known spectral scheme to efficiently approximate the underlying operator. 

It is an open challenge for future work to improve our understanding of the relevant class of operators for which operator learning is feasible, and to derive a useful mathematical characterization of relevant features that enable efficient approximation by PCA-Net and other operator learning architectures. Future work could also aim to derive more precise lower bounds which characterize the curse of dimensionality. In this context, we mention the forthcoming article \citep{LS2023}, where similar ideas are refined and considerably generalized, and \emph{exponential} lower bounds are derived for a more general class of neural operators, albeit with respect to the supremum norm; those results do not translate to any lower bounds in the $L^2_\mu$-norm considered here. In a different research direction, we point out that the present work has focused only on an approximation theoretic point of view, leaving out important questions related to optimization and generalization errors, given a finite amount of data. A significant challenge for future work is to address the practical training of the underlying neural network, and in particular, to determine bounds on the amount of training data that is necessary to achieve a desired accuracy. 

\rev{Finally, it could be of interest to link the results of the present work with recent work on overcoming the curse of dimensionality in the solution of high-dimensional PDEs, e.g. \citep{chow2017algorithm,
chow2019algorithm,
darbon2016algorithms,
darbon2020overcoming,
marwah2021parametric,
chen2021representation,marwah2023neural}. In the context of these high-dimensional PDEs, the focus has usually been on approximating an individual solution rather than learning the underlying solution operator. It would be of interest to further explore the possibility of operator learning for such high-dimensional PDEs, with the aim of deriving complexity estimates that exhibit good scaling with respect to the infinite-dimensional input and output function spaces, and in addition, scale favorably with respect to the spatial dimension $n \gg 1$ of the underlying spatial domain on which the PDE is defined.} We leave these general research directions as interesting avenues for future work.

% Acknowledgements and Disclosure of Funding should go at the end, before appendices and references

\acks{
The author would like to thank Siddhartha Mishra for helpful discussions and guidance when developing many of the ideas that have gone into this work. 
This work has been supported by Postdoc.Mobility grant P500PT-206737 from the Swiss National Science Foundation.
}

%\bibliographystyle{plain}
%\bibliography{references}

\vskip 0.2in
\bibliography{references}

\appendix

\section{PCA projection error bound}
\label{app:pcaproj}
Let $\hat{P}_{\le d} = \cD_\cH \circ \cE_\cH$ denote the empirical PCA projection, given samples $u_1,\dots, u_N \simiid \mu$. We recall that the excess risk is defined by
\begin{align}
\label{eq:excess}
\cE^\pca_{\le d} = 
\E_{u\sim \mu}\left[
\Vert u - \hat{P}_{\le d} u \Vert_{\cH}^2
\right]
-
\min_{P\in \Pi_d} \E_{u\sim \mu}\left[
\Vert u - P u \Vert_{\cH}^2
\right].
\end{align}
The following general bound on the excess risk is well-known
\citep[Sect. 2.2]{reisswahl2020}:
\begin{align}
\label{eq:epca}
\cE^\pca_{\le d}
\le
\sqrt{2d} \Vert \Sigma - \hat{\Sigma} \Vert_{HS},
\end{align}
where $\Vert \slot \Vert_{HS}$ denotes the Hilbert-Schmidt norm, and where $\hat{\Sigma} = \frac1N \sum_{k=1}^N u_k \otimes u_k$ the empirical covariance operator, and $\Sigma = \E[\hat{\Sigma}]$ is its expectation. We recall that the set of Hilbert-Schmidt operators $HS(\cH)$ is itself a Hilbert space with norm $\Vert \slot \Vert_{HS}$, and that $\Vert A \Vert_{HS}^2 = \tr(A^TA)$ can be expressed in terms of the trace.

\subsection{Proof of Proposition \ref{prop:pcaproj-qual}}
Let $\cE^\pca_{\le d}$ be the excess risk \eqref{eq:excess}.
Note that the claimed bound \eqref{eq:pcaproj-qual} in Proposition \ref{prop:pcaproj-qual} can be written in terms of the excess risk, as 
$\cE^\pca_{\le d} \le \epsilon$.
Rather than proving the high-probability result of Proposition \ref{prop:pcaproj-qual} directly, we will deduce it from the following expectation result:
\begin{lemma}
\label{lem:pcaproj-qual}
Let $\cH$ be a separable Hilbert space. Let $\mu\in \cP(\cH)$ be a probability measure with finite second moments, $\E_{u\sim \mu}[\Vert u \Vert_{\cH}^2] < \infty$. Then for any $\epsilon>0$ and integer $d\in \N$, there exists a requisite amount of data $N = N(\mu,d,\epsilon)$, such that the encoding error for empirical PCA with dimension $d$, and based on $N$ samples $u_1,\dots, u_N \simiid \mu$, satisfies
\[
\E_{\{u_j\}\sim \mu^{\otimes N}} \left[ \cE^\pca_{\le d} \right] \le 
\sqrt{2d} \,
\E\left[
\Vert \Sigma - \hat{\Sigma}\Vert_{HS} 
\right]
\le \epsilon.
\]
\end{lemma}
The corresponding high-probability result, claimed in Proposition \ref{prop:pcaproj-qual}, then follows readily from Lemma \ref{lem:pcaproj-qual}, as shown next.
\begin{proof}[Proof of Proposition \ref{prop:pcaproj-qual} from Lemma \ref{lem:pcaproj-qual}]
Let $\epsilon, \delta > 0$ be given. To deduce the claimed high-probability bound of Proposition \ref{prop:pcaproj-qual}, we apply Lemma \ref{lem:pcaproj-qual} with $\epsilon$ replaced by $\epsilon' := \epsilon \delta$, to deduce that there exists a required amount of data $N = N(\epsilon, \delta)$, such that 
\[
\E\left[
\cE^\pca_{\le d} 
\right]
\le \epsilon' = \epsilon \delta.
\]
\rev{Using the non-negativity of the random variable $\cE^\pca_{\le d}$ together with Markov's inequality,} this now immediately implies,
\[
\Prob\left[
\cE^\pca_{\le d} > \epsilon
\right]
\le
\frac{\E\left[ \cE^\pca_{\le d} \right]}{\epsilon}
\le \delta,
\]
i.e. with probability at least $1-\delta$, we have $\cE^\pca_{\le d} \le \epsilon$.
\end{proof}

We finally come to the proof of the result in expectation, Lemma \ref{lem:pcaproj-qual}. 
\begin{proof}[Proof of Lemma \ref{lem:pcaproj-qual}]
We define $Z := u \otimes u$, with $u\sim \mu$, as an $(HS(\cH), \Vert \slot\Vert_{HS})$-valued random variable; note that $Z$ takes values in the space of Hilbert-Schmidt operators $HS(\cH)$. Note that $\hat{\Sigma} = \frac1N \sum_{k=1}^N Z_k$ is the sum of i.i.d. copies of $Z$. Given a cut-off $M>0$, we define random variables $Z_{k,\le M} := (u_k\otimes u_k) 1_{[\Vert u_k \Vert \le M]}$, $Z_{k,> M} := (u_k\otimes u_k) 1_{[\Vert u_k \Vert > M]}$. Then, by the triangle inequality:
\begin{align*}
\Vert \Sigma - \hat{\Sigma} \Vert_{HS}
&\le
\left\Vert 
\frac1N \sum_{k=1}^N Z_{k,\le M} - \E[Z_{k,\le M}]
\right\Vert_{HS}
+
\frac{1}{N} \sum_{k=1}^N \Vert Z_{k,>M} - \E[Z_{k,>M}] \Vert_{HS}
\end{align*}
We next observe that, upon taking expectations, the second contribution can be estimated by
\[
\E\left[ \Vert Z_{k,>M} - \E[Z_{k,>M}] \Vert_{HS}\right]
\le
2 \E\left[ \Vert Z_{k,>M} \Vert_{HS}\right]
=
2 \E_{u\sim\mu}\left[ \Vert u \Vert^2 1_{[\Vert u \Vert > M]} \right].
\]
On the other hand, since $\Vert Z_{k,\le M}\Vert_{HS} = \Vert u_k \Vert^2 1_{\Vert u \Vert\le M} \le M^2$ is uniformly bounded, we can apply Jensen's inequality and the standard Monte-Carlo estimate on the expectation of the first term to obtain
\begin{align*}
\E
\left[
\left\Vert 
\frac1N \sum_{k=1}^N Z_{k,\le M} - \E[Z_{k,\le M}]
\right\Vert_{HS}
\right]
&\le
\E\left[
\left\Vert 
\frac1N \sum_{k=1}^N Z_{k,\le M} - \E[Z_{k,\le M}]
\right\Vert_{HS}^2\right]^{1/2}
\\
&\le \frac{M^2}{\sqrt{N}}.
\end{align*}
Combining these estimates, it follows that
\begin{align}
\label{eq:ed}
\E\left[ \Vert \Sigma - \hat{\Sigma} \Vert_{HS} \right]
\le
\frac{M^2}{\sqrt{N}}
+
2 \E_{u\sim\mu}\left[ \Vert u \Vert^2 1_{[\Vert u \Vert > M]} \right].
\end{align}
By assumption, the second moment $\E_{u\sim\mu}[\Vert u \Vert^2]$ is finite. Therefore, for any $\epsilon > 0$, there exists $M = M(\epsilon,d,\mu) >0$, such that
\[
2 \E_{u\sim\mu}\left[ \Vert u \Vert^2 1_{[\Vert u \Vert > M]} \right]
\le \frac{\epsilon}{2\sqrt{2d}}.
\]
Choosing next $N = N(\epsilon,d,\mu)$ sufficiently large, we ensure that $M^2/\sqrt{N} \le \epsilon /2\sqrt{2d}$. With this choice of $N$, it follows from \eqref{eq:ed}, that
\[
\E\left[ \Vert \Sigma - \hat{\Sigma} \Vert_{HS} \right]
\le \epsilon/\sqrt{2d},
\]
and hence
\[
\E_{\{u_j\}\sim \mu^{\otimes N}} \left[ \cE^\pca_{\le d} \right] \le 
\sqrt{2d} \,
\E\left[
\Vert \Sigma - \hat{\Sigma}\Vert_{HS} 
\right]
\le \epsilon.
\]
\end{proof}

\subsection{Proof of Proposition \ref{prop:pcaproj}}

We first recall a vector-valued Bernstein inequality that is applicable on infinite-dimensional Hilbert spaces~\citep{caponnetto2007optimal,maurer2021concentration,rudi2017generalization}. The following result is an immediate consequence of \citep[Cor. 1, p. 144]{pinelis1986remarks}:
\begin{theorem}[Vector-valued Bernstein inequality]\label{thm:bern_vv}
Let $Z$ be an $H$-valued random variable, where $H$ is a separable Hilbert space. Suppose there exist positive numbers $b>0$ and $\sigma>0$ such that 
\begin{equation}\label{eq:bern_cond}
    \E\Vert{Z-\E Z}\Vert_{H}^p\leq \frac{1}{2}p!\sigma^2 b^{p-2} \quad\text{for all}\quad p\geq 2.
\end{equation}
Then for any $\delta\in(0,1)$ and $N\in\N$, denoting by $\{Z_k\}_{k=1}^N$ a sequence of $N$ iid copies of $Z$, it holds that
\begin{equation}\label{eq:bern_vv}
    \Prob\left\{\left\Vert{\frac{1}{N}\sum_{k=1}^NZ_k - \E Z}\right\Vert_{H} \leq \frac{2b\log(2/\delta)}{N} + \sqrt{\frac{2\sigma^2\log(2/\delta)}{N}}\right\}\geq 1-\delta.
\end{equation}
\end{theorem}
A random variable $Z$ satisfying the \emph{Bernstein moment condition}~\eqref{eq:bern_cond} is subexponential in the sense that $\Vert{Z-\E Z}\Vert_{H}$ is subexponential on $\R$. The parameters $\sigma^2$ and $b$ represent the variance and variance proxy/subexponential norm, respectively.

\begin{proof}[Proof of Proposition \ref{prop:pcaproj}]
Let $\hat{P}_{\le d} = \cD_\cH \circ \cE_\cH$ denote the empirical PCA projection, given samples $u_1,\dots, u_N \simiid \mu$. We recall that the excess risk is defined by
\[
\cE^\pca_{\le d} = 
\E_{u\sim \mu}\left[
\Vert u - \hat{P}_{\le d} u \Vert_{\cH}^2
\right]
-
\min_{P\in \Pi_d} \E_{u\sim \mu}\left[
\Vert u - P u \Vert_{\cH}^2
\right].
\]
Note that the claimed bound \eqref{eq:pcaproj} in Proposition \ref{prop:pcaproj} can be written in the form 
\begin{align}
\label{eq:qdn}
\cE^\pca_{\le d} \le \sqrt{\frac{Qd \log(2/\delta)}{N}}
\end{align}
We again rely on the following general bound on the excess risk (cp. eq. \ref{eq:epca}):
\begin{align}
\label{eq:epca1}
\cE^\pca_{\le d}
\le
\sqrt{2d} \Vert \Sigma - \hat{\Sigma} \Vert_{HS}.
\end{align}
with $\Vert \slot \Vert_{HS}$ the Hilbert-Schmidt norm, and where $\hat{\Sigma} = \frac1N \sum_{k=1}^N u_k \otimes u_k$ the empirical covariance operator, and we have $\Sigma = \E[\hat{\Sigma}]$. We recall that the set of Hilbert-Schmidt operators $HS(\cH)$ is itself a Hilbert space with norm $\Vert \slot \Vert_{HS}$. Let $Z := u \otimes u$, $u\sim \mu$, be an $HS(\cH)$-valued random variable. Then $\hat{\Sigma} = \frac1N \sum_{k=1}^N Z_k$ is the sum of i.i.d. copies of $Z$. Since $\Vert u\Vert_{\cH}$ is sub-Gaussian by assumption, it follows that $\Vert Z \Vert_{2} = \Vert u \Vert_{\cH}^2$ is sub-exponential, and as a consequence, $\Vert Z - \E Z\Vert_{2}$ is also sub-exponential. In particular, Theorem \ref{thm:bern_vv} applied on the Hilbert space $H := HS(\cH)$ now implies that there exists a constant $q>0$, depending only on the law $\mu$ of $u$ (which also determines $Z$), such that
\begin{align*}
\Vert \hat{\Sigma} - \Sigma \Vert_{HS}
&=
\left\Vert \frac1N \sum_{k=1}^N Z_k - \E[Z] \right\Vert_{2}
\le
q \max\left\{\frac{\log(2/\delta)}{N},\sqrt{\frac{\log(2/\delta)}{N}}\right\},
\end{align*}
with probability at least $1-\delta$.
By assumption, we have $N \ge \log(2/\delta)$, and hence upon defining $Q := 2q^2$, we obtain
\begin{align*}
\Vert \hat{\Sigma} - \Sigma \Vert_{HS}
\le
q \max\left\{\frac{\log(2/\delta)}{N},\sqrt{\frac{\log(2/\delta)}{N}}\right\}
\le
\sqrt{\frac{Q \log(2/\delta)}{2N}},
\end{align*}
with probability at least $1-\delta$. Substitution in \eqref{eq:epca1} now yields \eqref{eq:qdn}. 
\end{proof}

\section{Universal approximation}
\label{app:universal}

\subsection{Proof of Theorem \ref{thm:universal}}

The high-probability result of Theorem \ref{thm:universal} can easily be derived from the following result in expectation:

\begin{proposition}[Universal approximation in expectation]
\label{prop:universal}
Let $\cX,\cY$ be separable Hilbert spaces and let $\mu \in \cP(\cX)$ be a probability measure on $\cX$. Let $\Psi^\dagger: \cX \to \cY$ be a $\mu$-measurable mapping. Assume the following moment conditions,
\[
\E_{u\sim \mu}[\Vert u \Vert_{\cX}^2], \;\; \E_{u\sim \mu}[\Vert \Psi^\dagger(u) \Vert_{\cY}^2] < \infty.
\]
Then for any $\epsilon > 0$, there are dimensions $\dx= \dx(\epsilon)$, $\dy = \dy(\epsilon)$, a requisite amount of data $N = N(\dx,\dy,\mu,\Psi^\dagger)$, and a neural network $\psi$ such that the PCA-Net, $\Psi = \cD_{\cY} \circ \psi \circ \cE_{\cX}$, satisfies
\[
\E_{\{u_k\}\sim \mu^{\otimes N}}\left[ \E_{u\sim \mu} \left[ \Vert \Psi^\dagger(u) - \Psi(u;\{u_k\}) \Vert^2_{\cY}  \right]\right]
\le
\epsilon.
\]
Here, the expectation is over the given data $u_1,\dots, u_N\sim \mu$, which determine the input-/output-pairs $\{u_k,\Psi^\dagger(u_k)\}_{k=1}^N$.
\end{proposition}
We first show that Proposition \ref{prop:universal} implies Theorem \ref{thm:universal}.
\begin{proof}[Proof of Theorem \ref{thm:universal} from Proposition \ref{prop:universal}]
Define $Z\in \R$ by
\[
Z := \E_{u\sim\mu} \Vert \Psi^\dagger(u) - \Psi(u;\{u_k\}) \Vert_{\cY}^2.
\]
Note that $Z$ is itself a random variable due to its dependence on the random data $\{u_k\} \sim \mu^{\otimes N}$. Applying Proposition \ref{prop:universal} with $\epsilon' := \epsilon \delta$ instead of $\epsilon$, we conclude that there exist dimensions $\dx$, $\dy$, a requisite amount of data $N\in \N$, and a neural network $\psi$, such that 
\[
\E_{\{u_k\}\sim \mu^{\otimes N}}[Z] \le \epsilon' = \epsilon \delta.
\]
We can now estimate the probability of the event $\{ Z > \epsilon \}$,
\[
\Prob[Z > \epsilon]
\le
\frac{\E_{\{u_k\} \sim \mu^{\otimes N}} \left[ Z \right]}{\epsilon}
\le
\delta.
\]
In particular with this choice of $\dx$, $\dy$, $N$ and $\psi$, we have
\[
\E_{u\sim\mu} \Vert \Psi^\dagger(u) - \Psi(u;\{u_k\}) \Vert_{\cY}^2 \le \epsilon,
\]
with probability at least $1-\delta$.
\end{proof}

It thus remains to prove Proposition \ref{prop:universal}, which is the main aim of this appendix. The detailed proof is the subject of the following section.

\subsection{Proof of Proposition \ref{prop:universal} (in expectation)}

We are interested in the following PCA-Net approximation error,
\begin{align}
\label{eq:approx-err}
\Err = \E_{u\sim\mu}\left[\Vert \Psi^\dagger(u) - \Psi(u) \Vert_{\cY}^2\right]^{1/2},
\end{align}
where we have suppressed the dependence of $\Psi(u) = \Psi(u;\{u_k\})$ on the data, arising from the data-dependence of empirical PCA.
We also recall that a PCA-Net is written as a composition of three mappings $\Psi(a) = \cD_\cY\circ \psi \circ \cE_\cX$. Intuitively, we expect each of these mappings to give a contribution to the overall error. Our analysis of $\Err$ will thus be based on a decomposition of the approximation error $\Err$ in terms of three contributions: PCA encoding/decoding errors $\Err_{\cX}$ and $\Err_{\cY}$ on $\cX$ and $\cY$, and a neural network error $\Err_{\psi}$. At times, it is more convenient not to consider the encoding and neural network error separately, we will denote the combined error (associated with the composition $\psi \circ \cE_\cX$) by $\Err_{\psi}^\ast$. 

\begin{lemma}[Error decomposition]
\label{lem:err-decomp}
Let $\mu\in \cP(\cX)$ be a probability measure on a Hilbert space $\cX$. Let $\Psi^\dagger:\cX \to \cY$ be a $\mu$-measurable, non-linear operator, with $\E_{u\in \mu}[\Vert \Psi^\dagger(u) \Vert_{\cY}^2] < \infty$. Let $\Psi(u) := \cD_\cY \circ \psi \circ \cE_\cX(u)$ be a PCA-Net with PCA dimensions $\dx$ and $\dy$. Then the approximation error satisfies the following estimates,
\begin{align}
\Err 
&\le
\Err_{\cY} 
 + \Err_{\psi}^\ast
\le
 \Err_{\cY} + \Lip(\Psi^\dagger) \Err_{\cX} 
 + \Err_{\psi}.
\end{align}
Here we have introduced the following encoding errors on $\cX$ and $\cY$,
\begin{subequations}
\begin{align}
\Err_{\cX} &= \E_{u\sim\mu}\left[\Vert u -  \hat{P}^{\cX}_{\le \dx} u \Vert_{\cX}^2\right]^{1/2},
\\
\Err_{\cY} &= \E_{v\sim \Psi^\dagger_\#\mu}\left[\Vert v -  \hat{P}^{\cY}_{\le \dy} v \Vert_{\cY}^2\right]^{1/2}, 
\end{align}
\end{subequations}
with $\hat{P}^{\cX}_{\le \dx} u := \cD_{\cX} \circ \cE_{\cX}(u)$ and $\hat{P}^{\cY}_{\le \dy} u := \cD_{\cY} \circ \cE_{\cY}(u)$ the empirical PCA projections,
and we introduce the following neural network errors,
\begin{subequations}
\begin{align}
\Err_{\cX}^\ast &= \E_{u\sim\mu} \left[\Vert \cE_\cY \circ \Psi^\dagger(u) - \psi \circ \cE_\cX(u) \Vert_{\ell^2(\R^\dy)}^2\right]^{1/2}, \\
\Err_{\psi} &= \E_{\alpha \sim (\cE_{\cX})_\#\mu} 
 \left[
 	\Vert 
 		\cE_\cY \circ \Psi^\dagger \circ \cD_\cX(\alpha) - \psi(\alpha)
 	\Vert_{\ell^2(\R^{d_\cY})}^2
 \right]^{1/2}.
 \end{align}
\end{subequations}
\end{lemma}
The derivation of Lemma \ref{lem:err-decomp} is straight-forward, and has been provided in detail for the conceptually similar DeepONet architecture in \citep[Theorem 3.3]{thdeeponet}. We will not repeat the derivation here.

We next state a general approximation result for operators, which will be a key ingredient in our proof of the universal approximation result, Proposition \ref{thm:universal}. This next Lemma \ref{lem:lip} is formulated with respect to the supremum norm, which prepares the stage for a corresponding approximation result in the $L^2_\mu$-norm in Lemma \ref{lem:lip2}.

\begin{lemma}
\label{lem:lip}
Let $\Psi^\dagger: K \to \cY$ be a continuous mapping, where $K \subset \cX$ is compact, and $\cX, \cY$ are Banach spaces. Assume that $\sup_{u\in K} \Vert \Psi^\dagger(u) \Vert_{\cY} \le M$ for some $M>0$. Then for any $\epsilon > 0$, there exists a Lipschitz continuous $\Psi^\star: \cX \to \cY$, such that 
\[
\sup_{u\in K} \Vert \Psi^\dagger(u) - \Psi^\star(u) \Vert_{\cY} \le \epsilon,
\]
and $\Vert \Psi^\star(u) \Vert_{\cY} \le M$, for all $u\in \cX$.
\end{lemma} 

Our proof of Lemma \ref{lem:lip} uses only classical tools of topology; we suspect that similar results may have been derived before. For completeness, we provide a detailed proof, below.

\begin{proof}
To prove the above lemma, we first construct a suitable Lipschitz continuous partition of unity $\{\psi_j: \cX \to [0,1]\}$ on $K$, which has a Lipschitz continuous extension to all of $\cX$. Then, we show that an interpolation of $\Psi^\dagger$ of the form $\Psi^\star(u) := \sum_{j} \Psi^\dagger(u_j)\psi_j(u)$ satisfies the claimed properties.

Since $K$ is compact, $\Psi^\dagger$ is in fact uniformly continuous. In particular, there exists $\delta > 0$, such that for all $u,u'\in K$, we have 
\begin{align} 
\label{eq:unifcont}
\Vert u - u' \Vert_{\cX} \le \delta \quad \Rightarrow \quad \Vert \Psi^\dagger(u) - \Psi^\dagger(u') \Vert_{\cY} \le \epsilon. 
\end{align}
Using the compactness of $K$, let $\{u_j\}_{j=1}^N$ be a finite collection of elements of $K$, such that the open balls of radius $\delta>0$ form a cover of $K$, i.e.
\[
K \subset \bigcup_{j=1}^N B_{\delta}(u_j).
\]
Define a ``hat'' function $\rho_\delta: [0,\infty) \to [0,\infty)$, by
\[
\rho_\delta(r) := 
\begin{cases}
1-r/\delta, &\quad (0\le r \le \delta), \\
0, &\quad (r>\delta).
\end{cases}
\]
We note that $\Psi: \cX \to \R$, given by
\begin{align}
\label{eq:psiK}
\Psi(u) := \dist(u,K) + \sum_{j=1}^N \rho_{\delta}(\Vert u - u_j \Vert_{\cX}),
\end{align}
is a Lipschitz continuous function, and that there exists $\delta_0>0$, such that $\Psi(u) \ge \delta_0$ for all $u\in \cX$. To see the latter, we note that since $K$ is compact, there exists $0< \delta_0 \le \delta$, such that 
\[
K \subset
B_{\delta_0}(K) := \set{u\in \cX}{\dist(u,K)<\delta_0} \subset \bigcup_{j=1}^N B_{\delta}(u_j).
\]
After possibly decreasing $\delta_0$, we may wlog assume that $\delta_0 < 1/2$ and $\delta_0 < \delta/2$. It then follows from \eqref{eq:psiK} that, for any $u\notin B_{\delta_0}(K)$, we have
$\Psi(u) \ge \dist(u,K) \ge \delta_0$. On the other hand, if $u\in B_{\delta_0}(K)$, then there exists $u_j$, such that $\Vert u - u_j \Vert < \delta_0 \le \delta/2$. And hence, 
\[
\Psi(u) \ge \rho_\delta(\Vert u - u_j \Vert_{\cX}) \ge 1 - 1/2 = 1/2 \ge \delta_0.
\]
Thus, $\Psi(u) \ge \delta_0$ for all $u\in \cX$. Next, define
\[
\psi_j(u) := \frac{\rho_\delta(\Vert u - u_j \Vert_{\cX}) }{ \Psi(u) }, 
\quad
(j=1,\dots, N).
\]
Then $\psi_j$ is globally Lipschitz continuous for all $j=1,\dots ,N$, $0 \le \psi_j \le 1$, and $\sum_{j=1}^N \psi_j(u) = 1$ for all $u\in K$, i.e. the $\psi_j$ form a partition of unity on $K$. Furthermore, it is easy to see that $\sum_{j=1}^N \psi_j(u) \le 1$ for all $u\in \cX \setminus K$. Given this construction, we now define
\[
\Psi^\star(u) := \sum_{j=1}^N \psi_j(u) \Psi^\dagger(u_j).
\]
Since $\Vert \Psi^\dagger(u_j) \Vert_{\cY} \le \sup_{u \in K}\Vert \Psi^\dagger(u) \Vert_{\cY}\le M$ for all $j=1,\dots, N$, and since $\sum_{j=1}^N \psi_j(u) \le 1$, it follows that for any $u\in \cX$:
\[
\Vert \Psi^\star(u) \Vert_{\cY}
\le
\sum_{j=1}^N \psi_j(u) \Vert \Psi^\dagger(u_j)\Vert_{\cY}
\le 
\sum_{j=1}^N \psi_j(u) M \le M.
\]
This shows the upper bound $\sup_{u\in \cX} \Vert \Psi^\star(u) \Vert_{\cY}\le M$. We finally claim that $\sup_{u\in K} \Vert \Psi^\dagger(u) - \Psi^\star(u) \Vert_{\cY} \le \epsilon$. To prove this, we note that for $u\in K$:
\begin{align*}
\Vert \Psi^\dagger(u) - \Psi^\star(u) \Vert_{\cY}
&\le 
\sum_{j=1}^N \psi_j(u) \Vert \Psi^\dagger(u) - \Psi^\dagger(u_j) \Vert_{\cY}
\le 
\sum_{j=1}^N \psi_j(u) \epsilon = \epsilon,
\end{align*}
where we have used the uniform continuity \eqref{eq:unifcont} of $\Psi^\dagger$ in the second inequality and the fact that $\psi_j(u) \ne 0$, only if $\Vert u - u_j \Vert_{\cX} \le \delta$. This completes the proof.
\end{proof}

Based on Lemma \ref{lem:lip}, we can prove the following $L^2_\mu$-approximation result:

\begin{lemma}
\label{lem:lip2}
Let $\cX$, $\cY$ be separable Banach spaces, let $\mu\in \cP(\cX)$ be a probability measure on $\cX$. Assume that $\Psi^\dagger: \cX \to \cY$ is a Borel measurable mapping, such that $\E_{u\sim\mu}\left[ \Vert \Psi^\dagger(u) \Vert_{\cY}^2 \right]< \infty$. Then for any $\epsilon > 0$, there exists a bounded Lipschitz continuous mapping $\Psi^\star: \cX \to \cY$, such that 
\[
\E_{u\sim\mu} \left[ 
\Vert \Psi^\dagger(u) - \Psi^\star(u) \Vert_{\cY}^2
\right]^{1/2}
< \epsilon,
\]
and $\sup_{u\in \cX} \Vert \Psi^\star(u) \Vert_{\cY} < \infty$.
\end{lemma}

\begin{proof}
Let $\epsilon > 0$ be given. Define $\Psi^\dagger_M$ by 
\[
\Psi^\dagger_M(u) :=
\begin{cases}
\Psi^\dagger(u), & (\Vert \Psi^\dagger(u) \Vert_{\cY} \le M), \\
M\frac{\Psi^\dagger(u)}{\Vert \Psi^\dagger(u) \Vert_{\cY}}, & (\Vert \Psi^\dagger(u)\Vert_{\cY} > M).
\end{cases}
\]
For sufficiently large $M>0$, we have $\E_{u\sim\mu} \left[ \Vert \Psi^\dagger(u) - \Psi^\dagger_M(u) \Vert_{\cY}^2\right]^{1/2} \le \epsilon / 2$. It will thus suffice to show that there exists $\Psi^\star: \cX \to \cY$ which is globally Lipschitz, satisfies $\Vert \Psi^\star(u) \Vert_{\cY} \le M$ for all $u\in \cX$, and
\[
\E_{u\sim\mu} \left[ \Vert \Psi^\star(u) - \Psi^\dagger_M(u) \Vert_{\cY}^2\right]^{1/2} \le \epsilon / 2.
\]
To simplify notation, we will drop the subscript $M$, and assume directly that $\Vert \Psi^\dagger(u) \Vert_{\cY} \le M$ for all $u\in \cX$. We now prove the existence $\Psi^\star$. First, by Lusin's theorem, we note that there exists a compact $K\subset \cX$, such that 
\begin{align}
\label{eq:lusin}
\mu(\cX \setminus K) \le \left(\frac{\epsilon}{8M}\right)^2, \qquad \Psi^\dagger|_{K}: K \to \cY \text{ is continuous}.
\end{align}
By Lemma \ref{lem:lip}, there exists Lipschitz continuous $\Psi^\star: \cX \to \cY$, such that $\sup_{u\in \cX} \Vert \Psi^\star(u) \Vert_{\cY}\le M$, and 
\[
\sup_{u\in K} \Vert \Psi^\dagger(u) - \Psi^\star(u) \Vert_{\cY} \le \epsilon/4.
\]
 Thus, it follows that
\begin{align*}
\E_{u\sim\mu} \left[ \Vert \Psi^\dagger(u) - \Psi^\star(u) \Vert_{\cY}^2 \right]^{1/2} 
&\le
\E_{u\sim\mu} \left[ 1_K(u) \Vert \Psi^\dagger(u) - \Psi^\star(u) \Vert_{\cY}^2 \right]^{1/2} 
\\
&\qquad + \E_{u\sim\mu} \left[ 1_{\cX\setminus K}(u) \Vert \Psi^\dagger(u) - \Psi^\star(u) \Vert_{\cY}^2 \right]^{1/2} 
\\
&\le
\epsilon/4
+ 2M\left[\mu\left(\cX\setminus K\right)\right]^{1/2}
\le \epsilon/2.
\end{align*}
The last inequality follows from the choice of $K$ in \eqref{eq:lusin}. This proves the claim.
\end{proof}

After these preparatory results, we can now finally provide the detailed proof of Proposition \ref{prop:universal} (i.e. universal approximation in expectation):

\begin{proof}[Proof of Proposition \ref{prop:universal}] 
\label{pf:universal}
By assumption, $\Psi^\dagger: \cX \to \cY$ is Borel measurable and has finite second moment, $\E_{u\sim\mu}\left[\Vert \Psi^\dagger(u) \Vert_{\cY}^2\right] < \infty$. Our goal is to show that for any $\epsilon > 0$, there exist PCA dimensions $\dx$, $\dy$, a requisite amount of data $N$, depending only on $\epsilon$, and a neural network $\psi$, such that the PCA-Net approximation $\Psi = \cD_\cY \circ \psi \circ \cE_\cX$ of $\Psi^\dagger$, with empirical PCA encoder/decoder and based on $u_1,\dots, u_N \simiid \mu$ i.i.d. samples, satisfies
\[
\E_{\{u_k\}\sim \mu^{\otimes N}}
\E_{u\sim\mu} \left[ \Vert \Psi^\dagger - \Psi \Vert_{\cY}^2\right]^{1/2} \le \epsilon.
\]
 To this end, let $\epsilon > 0$ be given. The proof will follow in six steps 1--6, each constructing and estimating a different part of the PCA-Net $\Psi$. We start from Lemma \ref{lem:err-decomp}, which gives the following general error decomposition:
 \[
	\E_{u\sim\mu}\left[
	\Vert \Psi^\dagger(u) - \Psi(u) \Vert_{\cY}^2
	\right]^{1/2}
	\le
	\Err_{\cY} + \Err_{\cX}^\ast,	
 \]
 where 
\begin{align*}
\Err_{\cY} &= \E_{v\sim \Psi^\dagger_\#\mu}\left[\Vert v -  \cD_{\cY} \circ \cE_{\cY}(v) \Vert_{\cY}^2\right]^{1/2}, \\
\Err_{\cX}^\ast &= \E_{u\sim\mu} \left[\Vert \cE_\cY \circ \Psi^\dagger(u) - \psi \circ \cE_\cX(u) \Vert_{\ell^2(\R^\dx)} \right]^{1/2},
\end{align*}

\textbf{Step 1: (Bounding the decoding error on $\cY$)}
The first error, $\Err_{\cY}$, is the PCA encoding error on $\cY$. To control this error, we first choose $\dy = \dy(\epsilon)$ sufficiently large, such that the (optimal) PCA projection error of the push-forward measure $\Psi^\dagger_\#\mu$ on $\cY$, is bounded by
\[
\cR^\opt_{\dy}(\Psi^\dagger_\#\mu)
\equiv
\E_{v\sim \Psi^\dagger_\#\mu}\left[
		\Vert v - \cD_\cY^\opt \circ \cE^\opt_\cY(v) \Vert_{\cY}^2
	\right]
\le \epsilon/8.
\]
Next, invoking Lemma \ref{lem:pcaproj-qual} for the probability measure $\Psi^\dagger_\#\mu$ and with $\epsilon$ replaced by $\epsilon/8$, it follows that there exists a requisite amount of data $N = N(\Psi^\dagger, \mu,\epsilon)$, such that
\begin{align}
\label{eq:eps1}
	\E_{\{u_j\}\sim\mu^{\otimes N}} \E_{v\sim \Psi^\dagger_\#\mu}\left[
		\Vert v - \cD_\cY \circ \cE_\cY(v) \Vert_{\cY}^2
	\right] \le \frac{\epsilon}{8} + R_\dy^\opt(\Psi^\dagger_\#\mu) \le \frac{\epsilon}{4}.
\end{align}

\textbf{Step 2: (Lipschitz approximation $\Psi^\star \approx \Psi$)}
Next, we aim to estimate the approximation error $\Err_{\cX}^\ast$. To this end, we note that for $\Psi^\star\approx \Psi^\dagger$ a Lipschitz continuous approximation of $\Psi^\dagger$, we have
\begin{align*}
\Err_{\cX}^\ast
&\le 
\E_{u\sim\mu} \left[\Vert \cE_\cY \circ \Psi^\star(u) - \psi \circ \cE_\cX(u) \Vert_{\ell^2(\R^\dx)}^2 \right]^{1/2} \\
&\qquad 
+
\E_{u\sim\mu} \left[\Vert \cE_\cY \circ \Psi^\dagger(u) - \cE_\cY \circ \Psi^\star(u) \Vert_{\ell^2(\R^\dx)}^2 \right]^{1/2}
\\
&\le
\E_{u\sim\mu} \left[\Vert \cE_\cY \circ \Psi^\star(u) - \psi \circ \cE_\cX(u) \Vert_{\ell^2(\R^\dx)}^2 \right]^{1/2} \\
&\qquad 
+
\E_{u\sim\mu} \left[\Vert \Psi^\dagger(u) -\Psi^\star(u) \Vert_{\cY}^2 \right]^{1/2},
\end{align*}
where we used that the PCA encoding $\cE_\cY: \cY \to \R^\dy$ is a contraction in the last step. Note that the last term is independent of the PCA encoding, and only depends on the approximation $\Psi^\star \approx \Psi^\dagger$. Using Lemma \ref{lem:lip2}, we can find a Lipschitz continuous bounded mapping $\Psi^\star: \cX \to \cY$, such that 
\begin{align}
\label{eq:Gt}
\E_{u\sim\mu} \left[\Vert \Psi^\dagger(u) -\Psi^\star(u) \Vert_{\cY}^2 \right]^{1/2} 
\le
\frac{\epsilon}{4},
\quad
\sup_{u\in \cX} \Vert \Psi^\star(u) \Vert_{\cY} \le M < \infty,
\end{align}
for some constant $M>0$, depending only on $\epsilon$. 

\textbf{Step 3: (Decomposition in $\cX$-encoding and neural network approximation errors)} 
Next, we continue to bound the first term in the estimate on $\Err^\ast_\cX$ as follows:
\begin{align*}
\E_{u\sim\mu} &\left[\Vert \cE_\cY \circ \Psi^\star(u) - \psi \circ \cE_\cX(u) \Vert_{\ell^2(\R^\dx)} \right]^{1/2}
\\
&\le 
\E_{u\sim\mu} \left[\Vert \cE_\cY \circ \Psi^\star(u) - \cE_\cY \circ \Psi^\star \circ \cD_\cX \circ \cE_\cX(u) \Vert_{\ell^2(\R^\dx)} \right]^{1/2}
\\
&\quad + 
\E_{u\sim\mu} \left[\Vert \cE_\cY \circ \Psi^\star \circ \cD_\cX \circ \cE_\cX(u) - \psi \circ \cE_\cX(u) \Vert_{\ell^2(\R^\dx)} \right]^{1/2}
\\
&\le 
\Lip(\Psi^\star) \;
\E_{u\sim\mu} \left[\Vert u - \cD_\cX \circ \cE_\cX(u) \Vert_{\cX}^2 \right]^{1/2}
\\
&\quad + 
\E_{u\sim\mu} \left[\Vert \cE_\cY \circ \Psi^\star \circ \cD_\cX \circ \cE_\cX(u) - \psi \circ \cE_\cX(u) \Vert_{\ell^2(\R^\dx)} \right]^{1/2}
\\
&=: 
\Lip(\Psi^\star) \; \Err_\cX + \Err_{\Psi^\star},
\end{align*}
where we denote 
\begin{align*}
\Err_{\Psi^\star}
&:= 
\E_{u\sim\mu} \left[\Vert \cE_\cY \circ \Psi^\star \circ \cD_\cX \circ \cE_\cX(u) - \psi \circ \cE_\cX(u) \Vert_{\ell^2(\R^\dx)} \right]^{1/2}
\\
&= 
\E_{w\sim \left(\cE_{\cX}\right)_{\#}\mu} \left[\Vert \cE_\cY \circ \Psi^\star \circ \cD_\cX(w) - \psi(w) \Vert_{\ell^2(\R^\dx)} \right]^{1/2}.
\end{align*}

\textbf{Step 4: (Bounding the encoding error on $\cX$)}
Again, employing the estimate from Lemma \ref{lem:pcaproj-qual} similar to the argument leading to \eqref{eq:eps1}, we note that there exists a PCA dimension $\dx = \dx(\mu,\Lip(\Psi^\star),\epsilon)$ and a requisite amount of data $N = N(\mu,\Lip(\Psi^\star),\epsilon)$, such that we can bound the encoding error $\cE_\cX$ by
  \[
  \E_{\{u_j\}\sim\mu^{\otimes N}}  \E_{u\sim\mu}\left[\Vert u -  \cD_{\cX} \circ \cE_{\cX}(u) \Vert_{\cX}^2\right]^{1/2} \le \frac{\epsilon}{4\,  \Lip(\Psi^\star)}.
  \]
  This ensures that 
  \begin{align}
  \label{eq:eps2}
  \Lip(\Psi^\star) \, \E_{\{u_j\}\sim\mu^{\otimes N}} [\Err_{\cX}]  \le \frac{\epsilon}4.
  \end{align}

  \textbf{Step 5: (Bounding the neural network approximation error)}
  \rev{
  Finally, we define the Lipschitz continuous function $G: \R^\dx \to \R^\dy$ by $G(w) := \cE_\cY \circ \Psi^\ast \circ \cD_\cX(w)$, and denote $\nu := \left(\cE_{\cX}\right)_{\#}\mu$. We note that \eqref{eq:Gt} implies that 
  \[
  \sup_{w\in \R^\dx} |G(w)|
  =
  \sup_{w\in \R^\dx} |\cE_\cY (\Psi^\ast(\cD_\cX(w)))|
  \le 
  \sup_{u \in \cX} \Vert \cE_\cY \Vert_{\cX \to \ell^2(\R^{\dx})} \Vert \Psi^\ast(u) \Vert_{\cY} 
  \le M,
  \] 
  where we made use of the fact that the PCA-encoder $\cE_\cY$ is contractive, i.e.
  \[
  | \cE_\cY(v) | \equiv \Vert \cE_\cY(v) \Vert_{\ell^2(\R^\dx)} \le \Vert v \Vert_{\cY}, \quad \forall \, v \in \cY.
  \]
  } For any $R > 0$, we then obtain the following estimate on $\Err_{\Psi^\star}$:
  \begin{align*}
  \Err_{\Psi^\star}
  &\le
	\left\{  M + \sup_{w\in \R^{\dx}} |\psi(w)| \right\}
	\nu\left(\R^{\dx} \setminus [-R,R]^\dx\right)^{1/2}
	\\
 &\qquad +
 	\sup_{w \in [-R,R]^\dx} |G(w) - \psi(w)|.
  \end{align*}
  Since $\nu$ is a probability measure on $\R^\dx$, by choosing $R>0$ sufficiently large, we can ensure that 
  \[
  	3M \nu\left(\R^{\dx} \setminus [-R,R]^\dx\right)^{1/2} 
  	\le \epsilon / 8.
  \] 
  Next, by the universal approximation theorem of ReLU neural networks, there exists a (sufficiently large) DNN $\tilde{\psi}: \R^\dx \to \R^\dy$, such that 
  \[
  	\sup_{w\in [-R,R]^\dx} |G(w) - \tilde{\psi}(w)| \le \epsilon / 8.
  \]
  Finally, adding the two-layer ReLU network $\gamma_R: \R^\dx \to \R^\dy$, given by $\gamma_R(x):= \min(\max(x,-R),R)$ ensures that $\psi(w) := \tilde{\psi}(\gamma_R(w))$ satisfies 
  \begin{align*}
  	\psi(w) = \tilde{\psi}(w), \quad \forall \, w\in [-R,R]^\dx,
  \end{align*}
  and (wlog assuming $\epsilon \le M$):
  \[
  	\sup_{w\in \R^\dx} |\psi(w)|
  	=
  	\sup_{w\in [-R,R]^\dx} |\tilde{\psi}(w)|
  	\le
  	M + \epsilon 
  	\le 2M.
  \]
  Thus, with this construction of $\psi$, we have
  \begin{align*}
  	\Err_{\Psi^\star} 
  	&\le \left\{  M + \sup_{w\in \R^{\dx}} |\psi(w)| \right\} \nu\left( \R^\dx \setminus [-R,R]^\dx\right)^{1/2} \\
   &\qquad + \sup_{w\in [-R,R]^\dx} |G(w) - \psi(w)|
  	\\
  	&=
  	3M \nu\left( \R^\dx \setminus [-R,R]^\dx\right)^{1/2} + \sup_{w\in [-R,R]^\dx} |G(w) - \tilde{\psi}(w)|
  	\\
  	&\le 
  	\epsilon / 8 + \epsilon / 8 = \epsilon / 4.
  \end{align*}
  Thus, for this neural network $\psi$, we have
  \begin{align}
  \label{eq:eps4}
  \Err_{\Psi^\star}  \le \frac{\epsilon}{4}.
  \end{align}
  
  \textbf{Step 6: (Conclusion)}
  To conclude, we have shown that for any $\epsilon > 0$, we can choose sufficiently large $\dx = \dx(\mu,\Psi^\dagger,\epsilon)$, $\dy = \dy(\mu,\Psi^\dagger,\epsilon)$, a sufficient requisite amount of data $N = N(\mu,\Psi^\dagger,\epsilon)$, and we can construct a ReLU neural network $\psi: \R^\dx \to \R^\dy$, such that the resulting PCA-Net $\Psi = \cD_\cY \circ \psi \circ \cE_\cX$, satisfies
  \begin{align*}
  	\E_{u\sim\mu}\left[
	\Vert \Psi^\dagger - \Psi \Vert_{\cY}^2
	\right]^{1/2}
	&\le
	\Err_{\cY} + \Err_{\cX}^\ast
	\\
	&\le 
	\Err_{\cY} + \Lip(\Psi^\star) \, \Err_{\cX} + \Err_{\Psi^\star} + \E_{u\sim\mu}\left[ \Vert \Psi^\dagger(u) - \Psi^\star(u) \Vert_{\cY}^2 \right]^{1/2},
  \end{align*}
  where the expectation of the terms on the right is bounded by  \eqref{eq:eps1},\eqref{eq:Gt}, \eqref{eq:eps2} and \eqref{eq:eps4}.
	Taking expectation over the data $u_1,\dots, u_N \sim \mu^{\otimes N}$, we conclude that the constructed PCA-Net satisfies
	\[
		\E_{\{u_j\}\sim\mu^{\otimes N}} \E_{u\sim\mu}\left[
	\Vert \Psi^\dagger - \Psi \Vert_{\cY}^2
	\right]^{1/2} \le \epsilon.
	\]
 	This concludes our proof of the universal approximation property of PCA-Nets.
\end{proof}

\section{Curse of Dimensionality}
\label{app:cod}

Our goal in this appendix is to prove Theorem \ref{thm:cod}. To this end, it suffices to prove the following
\begin{claim}
\label{claim}
Let $\cX, \cY$ be Hilbert spaces, with $\dim(\cX) = \infty$, $\dim(\cY) \ge 1$. Let $\mu \in \cP(\cX)$ be a non-degenerate Gaussian measure. For any $\gamma > 0$, there exists a $k$-times Fr\'echet differentiable operator $\Psi^\dagger_\gamma: \cX \to \cY$ and a constant $c_\gamma > 0$, such that for any PCA-Net $\Psi = \cD_\cY\circ \psi \circ \cE_\cX$, we have
\begin{align}
\label{eq:claim}
\E_{u\sim \mu}\left[
\Vert 
\Psi^\dagger_\gamma(u) - \Psi(u)
\Vert_{\cY}^2
\right]
\ge c_\gamma \size(\psi)^{\gamma}.
\end{align}
\end{claim}

To prove this claim, it is sufficient to prove that such $\Psi^\dagger_\gamma: \cX \to \cY$ exists with the \emph{additional constraint} that $\im(\Psi^\dagger_\gamma)$ belongs to a one-dimensional subspace of $\cY$. In this case, we can clearly identify $\Psi^\dagger_\gamma$ with a mapping $\cX \to \R$, and thus, our goal is to prove the proposition in the (most constrained) case, where the output space $\cY = \R$ is one-dimensional. This will be the goal of the remainder of this section.

\subsection{A lower bound in finite dimensions}

We let $F_{k,d} := \set{f: \R^d \to \R}{\Vert f \Vert_{W^{k,\infty}} \le 1}$ denote the unit ball in the space of $k$-times weakly differentiable functions on $\R^d$, and where
\[
\Vert f \Vert_{W^{k,\infty}}
=
\max_{|\alpha|\le k}
\Vert D^\alpha f \Vert_{L^\infty(\R^d)} < \infty.
\]
In the following, we denote by $\psi_\theta$ an arbitrary, but fixed, feedforward neural network architecture (as defined in Section \ref{sec:neuralnetwork}), with a total of $W$ non-zero parameters collected in a vector $\theta \in \R^W$  (concatenating all weights and biases).
We first recall the following result, which is an immediate consequence of \citep[Corollary 1 and Lemma 2]{achour2022general}:
\begin{lemma}
\label{lem:lower1}
Fix $k\in \N$. There exists $c_{k,d}>0$ and $W_{\mathrm{min}} \in \N$ depending only on $k$ and $d$, such that for any $W\ge W_{\mathrm{min}}$, we have the following lower bound:
\begin{align}
\label{eq:lower1}
\sup_{f\in F_{k,d}} \inf_{\theta\in \R^W} \Vert f - \psi_\theta \Vert_{L^2([0,1]^d)} \ge c_{k,d} W^{-3k/d}.
\end{align}
\end{lemma}
We note the close connection of the above lower bound \eqref{eq:lower1} and the claimed bound \eqref{eq:claim}. Indeed, choosing $d$ sufficiently large, we can ensure that $3k/d < \gamma$ and hence \eqref{eq:lower1} implies a lower bound of the form $c_\gamma W^{-\gamma}$. The main difference here is that (a) the architecture of $\psi_\theta$, including its potentially sparse structure, is \emph{fixed} in \eqref{eq:lower1} and (b) the estimate \eqref{eq:lower1} involves a supremum over all $f\in F_{k,d}$, whereas \eqref{eq:claim} asserts a similar bound for \emph{one specific}, fixed operator and an arbitrary size of the network $\psi$. The main challenge will be to prove Proposition \ref{prop:lower2}, repeated below, which gives a relevant lower bound in finite dimensions:
\PropLower*
% \begin{proposition}
% \label{prop:lower2}
% Fix $k\in \N$. There exists $c_{k,d}>0$ and $f\in F_{k,d}$ depending on $k$ and $d$, and an absolute constant $\lambda > 0$ independent of $d$ and $k$, such that for any neural network $\psi$, we have the lower bound
% \begin{align}
% \label{eq:lower2}
% \Vert f - \psi \Vert_{L^2([0,1]^d)} \ge c_{k,d} \, \size(\psi)^{-\lambda k/d}.
% \end{align}
% \end{proposition}

Let us first show that Proposition \ref{prop:lower2} implies Claim \ref{claim}.
\begin{proof}[Proof of Claim \ref{claim}]
Let a rate $\gamma > 0$ and a smoothness parameter $k\in \N$ be given. Choose a dimension $d$ such that $\gamma > \lambda k/d$, where $\lambda$ is the universal constant of Proposition \ref{prop:lower2}. Let $f\in F_{k,d}$ be a function satisfying the lower bound in \eqref{eq:lower2}. 

We recall that the underlying measure $\mu\in \cP(\cX)$ is a non-degenerate Gaussian measure. \rev{Let $\lambda_1 \ge \lambda_2 \ge \dots \ge 0$ and $\phi_1,\phi_2,\dots \in \cX$ denote the ordered eigenvalues and eigenfunctions of the covariance operator $\Sigma = \E_{u\sim \mu}[u\otimes u]$.} In particular, since $\mu$ is non-degenerate, all PCA eigenvalues $\lambda_j > 0$ are positive. With the dimension $d$ chosen above, let $\cE^\opt_{\cX}: \cX \to \R^d$ denote the optimal PCA encoding, $\cE^\opt_{\cX}(u) = (\langle u, \phi_1\rangle, \dots, \langle u, \phi_d\rangle)$. We denote by $\cD_\cX^\opt: \R^d \to \cX$ the PCA decoder, and by $P_{\le d} = \cD_\cX^\opt\circ \cE_\cX^\opt$ the PCA projection onto the first $d$ PCA eigenfunctions. We now define 
\begin{align}
\label{eq:psid}
\Psi^\dagger_\gamma(u) := f(\cE_\cX^\opt(u)).
\end{align}
We clearly have $\Psi^\dagger_\gamma \in \mathcal{C}^k$. We claim that $\Psi^\dagger_\gamma$ also satisfies the asserted lower bound of Claim \ref{claim}. In particular, we aim to show that there exists a constant $c_\gamma > 0$, such that for any PCA-Net $\Psi = \cD_\cY\circ \psi \circ \cE_\cX$, we have
\begin{align}
\label{eq:claim1}
\E_{u\sim \mu}\left[
\Vert 
\Psi^\dagger_\gamma(u) - \Psi(u)
\Vert_{\cY}^2
\right]
\ge c_\gamma \size(\psi)^{\gamma}.
\end{align}

Let $u\sim \mu$ be a random variable. Since $\mu$ is Gaussian, we can expand $u$ in its Karhunen-Loeve expansion $u = \bar{u} + \sum_{j=1}^\infty \sqrt{\lambda_j} Z_j \phi_j$, where the $Z_j$ are iid standard Gaussian random variables and $\bar{u} = \E_{u\sim \mu}[u]$ is deterministic. In particular, this allows us to express $u$ in the form $u = u_1 + u_2$, where $u_1 \sim \mu_{\le d}$ and $u_2 \sim \mu_{>d}$ are independent Gaussian random variables with $\mu_{\le d} := P_{\le d,\#}\mu$ and $\mu_{>d} := P_{>d,\#}\mu$ denote the push-forward by projecting onto the relevant components.

Let now $\Psi = \psi \circ \cE_\cX$ be a PCA-Net with values in $\R$, and $\psi: \R^\dx \to \R$ a neural network. Note that the decoder $\cD_\cY$ is trivial when $\cY = \R$, and may wlog be assumed to be given by multiplication by $1$. By definition of $\Psi^\dagger_\gamma(u)$, we have $\Psi^\dagger_\gamma(u) = \Psi^\dagger_\gamma(u_1)$ under the decomposition $u = u_1+u_2$ from above. Hence, we obtain
\begin{gather}
\label{eq:u1u2}
\begin{aligned}
\E_{u\sim \mu}\left[
\vert
\Psi^\dagger_\gamma(u) - \Psi(u)
\vert^2
\right]
&=
\E_{u_2\sim \mu_{> d}}\E_{u_1\sim \mu_{\le d}}\left[
\vert 
\Psi^\dagger_\gamma(u_1) - \Psi(u_1+u_2)
\vert^2
\right]
\\
&\ge 
\inf_{u_2} 
\E_{u_1\sim \mu_{\le d}}\left[
\vert 
\Psi^\dagger_\gamma(u_1) - \Psi(u_1+u_2)
\vert^2
\right],
\end{aligned}
\end{gather}
where the infimum is over all $u_2 \in \Im(P_{>d}) = \Span\set{\phi_j}{j>d}$. Let us fix such $u_2$ for the moment. Identifying $u_1$ with its unique encoding $\xi = \cE_\cX^\opt(u_1) \in \R^d$, we note that, by definition of $\Psi^\dagger_\gamma$, we have $\Psi^\dagger_\gamma(u_1) = f(\xi)$.
Furthermore, we can write 
\[
\Psi(u_1+u_2) = \psi(A\xi + b) =: \psi_{A,b}(\xi),
\]
where $A$ is a matrix representing the linear map $\cE_\cX\circ \cD_\cX^\opt: \R^d \to \R^d$, and $b := \cE_\cX(u_2) \in \R^d$ is a fixed vector depending on the fixed choice of $u_2$. 
Substitution now implies that 
\begin{align*}
\E_{u_1\sim \mu_{\le d}}\left[
\vert 
\Psi^\dagger_\gamma(u_1) - \Psi(u_1+u_2)
\vert^2
\right]
&= 
\E_{\xi\sim \cN(m_{\le d},\Sigma_{\le d})}\left[
\vert 
f(\xi) - \psi_{A,b}(\xi)
\vert^2
\right],
\end{align*}
where $\cN(m_{\le d},\Sigma_{\le d}) = \cE^\opt_{\cX,\#}\mu_{\le d}$ is a Gaussian distribution with (encoded) mean $m_{\le d} := \cE_\cX(\bar{u})$ and diagonal covariance matrix $\Sigma_{\le d} = \mathrm{diag}(\lambda_1,\dots, \lambda_d)$. Since by assumption, $\mu$ is a non-degenerate Gaussian measure, it follows that all PCA eigenvalues $\lambda_j>0$ for all $j$. In particular, this implies that that $\cN(m_{\le d},\Sigma_{\le d})$ has a strictly positive density over the unit cube $[0,1]^d$. Thus, there exists a constant $c'_d$, such that
\begin{align*}
\E_{\xi\sim \cN(m_{\le d},\Sigma_{\le d})}\left[
\vert 
f(\xi) - \psi_{A,b}(\xi)
\vert^2
\right]
\ge 
c_d' \Vert f - \psi_{A,b} \Vert_{L^2([0,1]^d)}^2.
\end{align*}
By construction of $f$ (cp. Proposition \ref{prop:lower2}) and the fact that $\psi_{A,b}$ is a neural network, the latter quantity can be bounded from below by $c_{k,d} \size(\psi_{A,b})^{-\lambda k/d}$. Furthermore, it is easy to see that $\psi_{A,b}$ can be represented by a neural network with size bound
\[
\size(\psi_{A,b}) \le \size(\psi) + \underbrace{\size(\psi) d}_{\hat{=} \Vert A\Vert_0} + \underbrace{d}_{\hat{=} \Vert b\Vert_0} \le 2(d+1)\size(\psi).
\]
In particular, combining the above estimates and recalling that $\lambda k/d < \gamma$, it follows that there exists a constant $c_\gamma>0$, depending only on $k$, $d$ and $\mu$, such that
\[
\E_{u_1\sim \mu_{\le d}}\left[
\vert 
\Psi^\dagger_\gamma(u_1) - \Psi(u_1+u_2)
\vert^2
\right]
\ge c_\gamma \size(\psi)^{-\gamma},
\]
for any choice of $\psi$. The right-hand side is independent of $u_2$. Hence, taking the infimum over all $u_2$ on the left and recalling \eqref{eq:u1u2}, we obtain
\[
\E_{u\sim \mu}\left[
\vert 
\Psi^\dagger_\gamma(u) - \Psi(u)
\vert^2
\right]
\ge c_\gamma \size(\psi)^{-\gamma},
\]
as claimed.
\end{proof}

It finally remains to prove Proposition \ref{prop:lower2}. This is the subject of the next section.

\subsection{Proof of Proposition \ref{prop:lower2}}

The goal of this section is to prove the existence of $f\in F_{k,d}$ and a universal constant $\lambda > 0$, for which the lower bound \eqref{eq:lower2} holds for any neural network $\psi$. To construct $f$, we adapt an idea of Yarotsky \citep[Section 4]{yarotsky_error_2017}. In fact, the present analysis is very closely related to a refinement of Yarotsky's result \citep[Theorem 5]{yarotsky_error_2017}, which is derived in the forthcoming work \citep{LS2023}. The main difference here is that the present work develops a lower bound for the $L^2$-norm, while \citep{LS2023} focuses on the supremum norm. 

We now follow \citep{yarotsky_error_2017}, and define $\cN(f,\epsilon)$ as the minimal number of hidden computation units in a ReLU network $\psi$ which achieves approximation accuracy $\Vert f - \psi \Vert_{L^2([0,1]^d)} \le \epsilon$. We note that Lemma \ref{lem:lower1} implies in particular that for a fixed architecture $\psi_\theta$ with $W$ weights, the upper bound
\begin{align}
\label{eq:sin}
\sup_{f\in F_{k,d}} \inf_{\theta \in \R^W} \Vert f - \psi_\theta \Vert_{L^2([0,1]^d)} \le \epsilon,
\end{align}
requires $W \ge c\epsilon^{-d/3k}$ for some constant $c = c(d,k)$.
Repeating the proof of \citep[Lemma 3]{yarotsky_error_2017}, the following is then immediate:
\begin{lemma}
\label{lem:y2}
Fix $d$, $k$. For any $\epsilon > 0$, there exists $f_\epsilon \in F_{k,d}$ such that $\cN(f_\epsilon, \epsilon) \ge c_1 \epsilon^{-d/12k}$, with some constant $c_1 = c_1(d,k)>0$.
\end{lemma}
The idea is that any (sparse) neural network $\psi$ with at most $N$ computation units can be embedded in a maximally connected, enveloping architecture $\psi_\theta$, $\theta \in \R^W$, with $W = O(N^4)$ weights. Thus, if $\psi$ achieves error $\epsilon$, then \eqref{eq:sin} applied to the enveloping network $\psi_\theta$ implies that $N^4 \gtrsim W \gtrsim \epsilon^{-d/3k}$, and hence $N \gtrsim \epsilon^{-d/12k}$. Fixing $\epsilon > 0$ and choosing $f_\epsilon$ to be an optimizer of \eqref{eq:sin}, the claimed lower bound on $\cN(f_\epsilon,\epsilon)$ follows. Further details can be found in \citep[Pf. of Lemma 3]{yarotsky_error_2017}

Lemma \ref{lem:y2} will be our main technical tool for the proof of Proposition \ref{prop:lower2}. We also recall that $\cN(f,\epsilon)$ satisfies the following properties,  \citep[eq. (42)--(44)]{yarotsky_error_2017}:
\begin{subequations}
\label{eq:N}
\begin{gather}
\cN(af, |a|\epsilon) 
= \cN(f,\epsilon), \\
\cN(f\pm g, \epsilon + \Vert g \Vert_{L^\infty}) 
\le \cN(f,\epsilon), \\
\cN(f_1 \pm f_2, \epsilon_1 + \epsilon_2) 
\le \cN(f_1,\epsilon_1) + \cN(f_2,\epsilon_2).
\end{gather}
\end{subequations}
We are now ready to prove Proposition \ref{prop:lower2}.
\begin{proof}{(Proposition \ref{prop:lower2})}
We recall that the goal is to show that there exists $f\in F_{k,d}$ such that 
\[
\Vert f - \psi \Vert_{L^2([0,1]^d)} \ge c_{k,d} \, \size(\psi)^{-\lambda k/d}.
\]
for any neural network $\psi$.
We will achieve this by constructing $f\in F_{k,d}$ of the form 
\begin{align}
\label{eq:fform}
f = \sum_{s=1}^\infty a_s f_s, \quad \text{with } f_s \in F_{k,d} \; \forall \, s\in\N.
\end{align}
To this end, we define a rapidly converging sequence $\epsilon_s \to 0$, by $\epsilon_1 = 1/4$ and recursively set $\epsilon_{s+1} = \epsilon^2_s$. Note by iteration, that $\epsilon_s = 2^{-2^s}$. We also define $a_s := \frac12 \epsilon_{s}^{1/2} = \frac12 \epsilon_{s-1}$. For later reference, since $\epsilon_s \le 1/2$ for all $s$, we have
\begin{align}
\label{eq:atail}
\sum_{r=s+1}^\infty a_r = \frac12 \left( \epsilon_s + \epsilon_s^2 + \epsilon_s^{2^2} + \dots  \right) \le \frac12 \epsilon_s \sum_{r=0}^\infty 2^{-r} =  \epsilon_s.
\end{align}
We also note that $a_s \le 2^{-s}$. If the function $f$ is of the above form \eqref{eq:fform} then,
\begin{align*}
\Vert f \Vert_{W^{k,\infty}}
&\le 
\sum_{s=1}^\infty a_s \Vert f_s \Vert_{W^{k,\infty}}
\le 1,
\end{align*}
so that $f\in F_{k,d}$. This is irrespective of the specific choice of $f_s \in F_{k,d}$. We are going to  determine the $f_s$ recursively, formally starting from the empty sum, i.e. $f_0=0$. We fix $2\lambda := 1/25$ throughout the following construction. In the recursive step, given $f_1,\dots, f_{s-1}$, we distinguish two cases:

\textbf{Case 1:} Assume that 
\[
\cN\left(\sum_{r=1}^{s-1} a_r f_r, 2\epsilon_s\right) \ge \epsilon_s^{-2\lambda d/k}.
\]
In this case, we set $f_s := 0$, and trivially obtain
\begin{align}
\label{eq:fkc}
\cN\left(\sum_{r=1}^{s} a_r f_r, 2\epsilon_s\right) \ge  \epsilon_s^{-2\lambda d/k}.
\end{align}

\textbf{Case 2:} In the other case, we have
\[
\cN\left(\sum_{r=1}^{s-1} a_r f_r, 2\epsilon_s\right) <   \epsilon_s^{-2\lambda d/k}.
\]
By (\ref{eq:N}c), and the upper bound above:
\begin{align}
\cN\left(\sum_{r=1}^s a_r f_r, 2\epsilon_s\right)
&\ge 
\cN(a_s f_s, 4\epsilon_s) - \cN\left(\sum_{r=1}^{s-1} a_r f_r, 2\epsilon_s\right)
\notag\\
&\ge 
\cN(a_s f_s, 4\epsilon_s) - \epsilon_s^{-2\lambda d/k}.
\label{eq:Nafe}
\end{align}
Recalling Lemma \ref{lem:y2}, there exists $f_s\in F_{k,d}$, such that
\begin{align}
\label{eq:Nafe2}
\cN(a_s f_s, 4\epsilon_s) 
\explain{=}{(\ref{eq:N}a)} \cN(f_s, 4\epsilon_s/a_s) 
= \cN(f_s, 4\epsilon_s^{1/2})
\ge c_{d,k} \epsilon_s^{-d/24k}.
\end{align}
This defines our recursive choice of $f_s$ in Case 2. 

Our next goal is to show that with this recursive choice of $f_s$, \eqref{eq:fkc} holds for all sufficiently small $\epsilon_s \le \bar{\epsilon}$. This is trivial in Case 1. To prove it for Case 2, we define $\bar{\epsilon} = \bar{\epsilon}(d,k)>0$, as the unique solution of
\[
c_{d,k} \bar{\epsilon}^{-d/24k} = 2\bar{\epsilon}^{-2\lambda d/k}. 
\]
Then, by \eqref{eq:Nafe} and \eqref{eq:Nafe2}, and with our choice of $2\lambda := 1/25 < 1/24$, it follows that for all $\epsilon_s \le \bar{\epsilon}$, we have
\[
c_{d,k} {\epsilon}_s^{-d/24k} \ge 2{\epsilon}_s^{-2\lambda d/k},
\]
since the exponent on the left is strictly more negative than the exponent on the right.
Hence, \eqref{eq:Nafe} and \eqref{eq:Nafe2} imply that
\begin{align}
\label{eq:fkc1}
\cN\left(\sum_{r=1}^s a_r f_r, 2\epsilon_s\right)
\ge c_{d,k} \epsilon_s^{-d/24k} - \epsilon_s^{2\lambda d/k} 
\ge \epsilon_s^{2\lambda d/k}, 
\end{align}
for all $\epsilon_s \le \bar{\epsilon}$. This demonstrates inequality \eqref{eq:fkc} for $\epsilon_s\le \bar{\epsilon}$ also in Case 2. The above recursive construction thus yields a sequence $f_1,f_2, \dots \in F_{k,d}$, such that \eqref{eq:fkc1} holds for any $s\in\N$ such that $\epsilon_s \le \bar{\epsilon}$. Define $f := \sum_{s=1}^\infty a_s f_s$. We claim that for \emph{any} $\epsilon \le \bar{\epsilon}$ we have
\[
\cN(f,\epsilon) \ge \epsilon^{-\lambda d/k}.
\]
To this end, we fix $\epsilon \le \bar{\epsilon}$. Choose $s\in \N$, such that $\epsilon_s \le \bar{\epsilon}$ and $\epsilon_s^2 \le \epsilon \le \epsilon_s$. Then,
\begin{align*}
\cN(f,\epsilon) 
&\ge 
\cN(f,\epsilon_{s})
= 
\cN\left(\sum_{r=1}^\infty a_r f_r ,\epsilon_s\right)
\\
&\explain{\ge}{(\ref{eq:N}b)}
\cN\left(\sum_{r=1}^s a_r f_r ,\epsilon_s + \left \Vert \sum_{r=s+1}^\infty a_r f_r \right \Vert_{L^\infty}\right)
\\
&\ge
\cN\left(\sum_{r=1}^s a_r f_r, \epsilon_s + \sum_{r=s+1}^\infty a_r\right)
\\
&\explain{\ge}{(\ref{eq:atail})} 
\cN\left(\sum_{r=1}^s a_r f_r, 2\epsilon_s \right)
\\
&\explain{\ge}{\eqref{eq:fkc1}}  \epsilon_s^{-2\lambda d / k} 
\ge \epsilon^{-\lambda d / k},
\end{align*}
where the last inequality follows from $\epsilon_s \le \epsilon^{1/2}$. The claim of Proposition \ref{prop:lower2} thus follows for all $\epsilon \le \bar{\epsilon}$, and with universal constant $\lambda = 1/50$.
\end{proof}

\section{Proof of Proposition \ref{prop:PCAbound}}
\label{app:PCAbound}

\begin{proof}
We first assume that $D = \T^n$ is the periodic, $n$-dimensional torus. To prove the theorem in this case, let $d\in \N$ be given, and let $\phi_1,\dots, \phi_d, \dots \in L^2(\T^n;\R)$ an enumeration of the standard Fourier basis; to be more precise, we require that for all $k\in \N$, $\Delta \phi_k = \gamma_k \phi_k$ is an eigenfunction of the periodic Laplacian, and that the eigenvalues $\gamma_1\le \gamma_2 \le \dots$ are monotonically ordered. To give an example for $n=1$, a suitable choice is $\phi_1(x) \equiv 1/\sqrt{2\pi}$, $\phi_2(x) = \pi^{-1/2}\cos(x)$, $\phi_3(x) = \pi^{-1/2}\sin(x)$, $\phi_4(x) = \pi^{-1/2} \cos(2x)$, etc. We note that there exist constants $A = A(n)$, $B = B(n)>0$, such that 
\[
A k^{1/n} \le \sqrt{\gamma_k} \le  B k^{1/n}, \quad \forall\,  k>1.
\]
 Let $V_d := \Span\{\phi_1,\dots, \phi_d\}$. \rev{We now use the fact that 
 \[
 u - \Pi_{V_d} u = \sum_{k>d} \langle u, \phi_k \rangle \phi_k,
 \]
 and combine this with the following two-sided bounds,
 \[
 \left\Vert \sum_{k>d} c_k (-\Delta)^{s/2} \phi_k \right\Vert_{L^2}^2
 \lesssim
 \left\Vert \sum_{k>d} c_k \phi_k \right\Vert_{H^s}^2
 \lesssim
 \left\Vert \sum_{k>d} c_k (-\Delta)^{s/2} \phi_k \right\Vert_{L^2}^2,
 \]
 and the identity,
 \[
 \left\Vert \sum_{k>d} c_k (-\Delta)^{s/2} \phi_k \right\Vert_{L^2}^2
 =
 \sum_{k>d} \gamma_k^s |c_k|^2,
 \]
 which are valid for any sequence $(c_k)$ for which the last expression is finite.}
  Using these bounds with $c_k = \langle u, \phi_k \rangle$, we can estimate (with $C$ changing from line to line):
 \begin{align*}
 	\Vert u - \Pi_{V_d} u \Vert_{H^s}^2
 	&= \left\Vert \sum_{k>d} c_k \phi_k \right\Vert_{H^s}^2
 	\le C \left\Vert \sum_{k>d} c_k (-\Delta)^{s/2} \phi_k \right\Vert_{L^2}^2
 	\\
 	&=
 	C
 	\sum_{k > d}^\infty \gamma_k^{s} |\langle u, \phi_k \rangle|^2
 	=
 	C
 	\sum_{k > d}^\infty \frac{\gamma_k^{s+\zeta}}{\gamma_k^{\zeta}} |\langle u, \phi_k \rangle|^2
 	\\
 	&\le
 	C \frac{1}{\gamma_d^{\zeta}} \sum_{k>d} \gamma_k^{s+\zeta} |\langle u, \phi_k \rangle|^2
 	\le
 	\frac{C}{d^{2\zeta/n}} \Vert u \Vert_{H^{s+\zeta}}^2,
 \end{align*}
 and where the constant $C = C(n,s,\zeta)$ only depends on $n$, $s$ and $\zeta$, but is independent of $d$.
In particular, this implies that
\begin{align*}
R_d(\mu) 
&=
\min_{V_d} \E_{u\sim\mu}\left[ \Vert u - \Pi_{V_d} u \Vert_{H^s}^2 \right]
\le
Cd^{-2\zeta/n} \E_{u\sim\mu}\left[ \Vert u \Vert_{H^{s+\zeta}}^2 \right].
\end{align*}
To prove the general case, we note that for any Lipschitz domain $D \subset \R^n$, there exists a bounded, linear extension operator $L: H^s(D;\R^{n'}) \embeds H^s(\T^n;\R^{n'})$, which also defines a bounded mapping, $L: H^{s+\zeta}(D;\R^{n'}) \embeds H^{s+\zeta}(\T^n;\R^{n'})$  (see \citep[Lemma B.3]{thfno}), with bounded left-inverse  $L^+: H^s(\T^n; \R^{n'}) \to H^s(D;\R^{n'})$, and such that $L^+L = \id$. Here, $L$ can be thought of as an extension operator, which extends a function $u: D \to \R^{n'}$ to a periodic function with domain a bounding box $B \supset D$. And $L^+$ can be though of as a restriction operator, restricting $u: B \to \R^{n'}$ to $u|_{D}: D \to \R^{n'}$. 

The above argument in the periodic case shows that the optimal $d$-dimensional PCA subspace $\overline{V}_d \subset H^s(\T^n;\R^{n'})$ of the push-forward measure $L_\#\mu$ (which is defined on periodic functions) satisfies 
\[
\E_{u\sim L_\#\mu}\left[ \Vert u - \Pi_{\overline{V}_d} u\Vert_{H^s}^2 \right] \le C d^{-2\zeta/n} \E_{u\sim L_\#\mu}\left[ \Vert u \Vert_{H^{s+\zeta}}^2 \right].
\]
Since the support of $\supp(L_\#\mu) \subset \im(L)$ is contained in the image of $L$, it follows that $\overline{V}_d \subset \im(L)$. Since $L$ is an isomorphism onto its image, there exists a $d$-dimensional subspace $V_d \subset H^s(D;\R^{n'})$, such that $L(V_d) = \overline{V}_d$. Notice that for any $\bar{u}\in H^s(\T^n;\R^{n'})$, we have 
\[
L^+\Pi_{\bar{V}_d}\bar{u} \in L^+\bar{V}_d = L^+LV_d = V_d.
\]
Using the fact that $\Pi_{V_d}: H^s(D;\R^{n'})\to H^s(D;\R^{d'})$ is the orthogonal projection onto $V_d$, we can thus estimate, for any $\bar{u}\in H^s(\T^n;\R^{n'})$,
\[
\Vert L^+ \bar{u} - \Pi_{V_d} L^+ \bar{u} \Vert_{H^s}
=
\inf_{v \in V_d} 
\Vert L^+ \bar{u} - v \Vert_{H^s}
\le
\Vert L^+ \bar{u} - L^+ \Pi_{\bar{V}_d} \bar{u} \Vert_{H^s}.
\]
We thus conclude from $\mu = \id_\#\mu = L^+_\# L_\#\mu$, that
\begin{align*}
\E_{u\sim\mu} \left[
\Vert u - \Pi_{V_d} u \Vert_{H^s}^2
\right]
&=
\E_{u\sim L^+_\# L_\#\mu} \left[
\Vert u - \Pi_{V_d} u \Vert_{H^s}^2
\right]
\\
&=
\E_{\bar{u}\sim L_\#\mu} \left[
\Vert L^+ \bar{u} - \Pi_{V_d} L^+ \bar{u} \Vert_{H^s}^2
\right]
\\
&\le
\E_{u\sim L_\#\mu} \left[
\Vert L^+ \bar{u} - L^+\Pi_{\bar{V}_d}  \bar{u} \Vert_{H^s}^2
\right]
\\
&\le
\Vert L^+ \Vert^2
\E_{u\sim L_\#\mu} \left[
\Vert \bar{u} - \Pi_{\bar{V}_d}  \bar{u} \Vert_{H^s}^2
\right]
\\
&\le
C\Vert L^+ \Vert^2 d^{-2\zeta/n} \E_{\bar{u}\sim L_\#\mu} \left[
\Vert \bar{u} \Vert_{H^{s+\zeta}}^2
\right]
\\
&=
C\Vert L^+ \Vert^2 d^{-2\zeta/n} \E_{{u}\sim \mu} \left[
\Vert L{u} \Vert_{H^{s+\zeta}}^2
\right]
\\
&\le
C\Vert L^+ \Vert^2 \Vert L \Vert^2 d^{-2\zeta/n} \E_{{u}\sim \mu} \left[
\Vert u \Vert_{H^{s+\zeta}}^2
\right]
\end{align*}
Absorbing the additional factor $\Vert L^+ \Vert^2 \Vert L \Vert^2$ in the constant $C$, this implies the claim for general Lipschitz domains $D\subset \R^n$.
\end{proof}

\section{Proofs for Darcy flow}

The following subsections will discuss the derivation of the error and complexity bounds of Theorem \ref{thm:darcy}. 

\subsection{Proof of Lemma \ref{lem:Taylor}}
\label{app:Taylor}

Thm. 1.3 of \citep{CDS2011} shows that if $a(x;\bm{z})$ is of the parametrized form \eqref{eq:expansion}, and if the sequence $(\gamma_\ell)_{\ell\in\N}$ belongs to $\ell^p(\N)$ for some $0<p<1$, then the family of Taylor coefficients $(t_{\bm{\nu}})_{\bm{\nu}}$ is also $\ell^p$-summable. In addition, if $\Lambda_m$ denotes the set of the largest $m$ Taylor coefficients, then 
\[
\sup_{\bm{z}\in U} 
\Vert 
\cF(\bm{z}) - \sum_{\bm{\nu}\in \Lambda_m} t_{\bm{\nu}} \bm{z}^{\bm{\nu}}
\Vert_{\cY}
\le
\Vert (\Vert t_{\bm{\nu}} \Vert_{\cY})_{\bm{\nu}} \Vert_{\ell^p}
m^{-s}, \quad s:= \frac1p - 1.
\]
By assumption, we have $|\gamma_\ell|\lesssim \ell^{-1-\alpha}$, and hence, $\gamma_\ell$ belongs to $\ell^p$ whenever $p(1+\alpha) > 1$. Hence, for any $s = p^{-1} - 1 < \alpha$, there exists $C>0$, possibly depending on $s$, such that
\[
\sup_{\bm{z}\in U} 
\Vert 
\cF(\bm{z}) - \sum_{\bm{\nu}\in \Lambda_m} t_{\bm{\nu}} \bm{z}^{\bm{\nu}}
\Vert_{\cY}
\le
Cm^{-s}.
\]
Writing $s = \alpha - \eta$ for $\eta > 0$ yields the claim of Lemma \ref{lem:Taylor}. %\qed

\subsection{Proof of Proposition \ref{prop:Darcy-enc-err}}
\label{app:Darcy-enc-error}

The goal of the present section is to prove Proposition \ref{prop:Darcy-enc-err}, which we recall here:
\darcyencerror*

\begin{proof}
We recall that by assumption, the input measure is the law of random functions $a$ of the form 
\[
a = \bar{a} + \sum_{\ell=1}^\infty \gamma_\ell z_\ell \rho_\ell, 
\]
where $\rho_\ell\in \cX$ are orthonormal, $\gamma_\ell \le C \ell^{-1-\alpha}$, and $|z_\ell|\le 1$ for all $\ell$. We estimate the optimal PCA projection error $\cR^\opt_d(\mu)$ by comparing it against the (at most) $d$-dimensional projection $P'$ onto $\Span\{\bar{a},\rho_1,\dots, \rho_{d-1}\}$. This yields,
\begin{align*}
\cR^\opt_d(\mu)
&=
\min_{P\in \Pi_d} \E_{a\sim \mu}[\Vert a - P a\Vert_\cX^2]
\le
\E_{a\sim \mu}[\Vert a - P' a\Vert_\cX^2]
\\
&=
\E_{z} \sum_{\ell=d}^\infty \gamma_\ell^2 z_\ell^2
\le
\sum_{\ell=d}^\infty \gamma_\ell^2
\le 
C d^{-2\alpha-1}.
\end{align*}
To estimate the PCA projection error on $\cY$, we note that by Lemma \ref{lem:Taylor}, for any $\eta>0$, there exists a constant $C>0$, depending only on $\cF$, $\mu$ and $\eta$, and there exist coefficients $t_{\bm{\nu}} \in \cY$ (the Taylor coefficients of $\cF$) and a family of multi-index sets $\Lambda_m$, with $|\Lambda_m|=m$ for all $m\in \N$, such that 
\[
\sup_{\bm{z}\in U} \Vert \cF(\bm{z}) - \sum_{{\bm{\nu}} \in \Lambda_m} t_{\bm{\nu}} \bm{z}^{\bm{\nu}}\Vert_\cY
\le C m^{-\alpha+\eta}.
\]
Choosing $m=d$ to be the PCA dimension, and letting $P'': \cY \to \cY$ denote the orthogonal projection onto $\Span\set{t_{\bm{\nu}}}{{\bm{\nu}}\in \Lambda_d} \subset \cY$, it follows that
\begin{align*}
\cR^\opt_d(\Psi^\dagger_\#\mu)
&=
\min_{P\in \Pi_d} \E_{a \sim \mu} [\Vert \Psi^\dagger(a) - P \Psi^\dagger(a)\Vert_{\cY}^2]
\le
\E_{a \sim \mu} [\Vert \Psi^\dagger(a) - P'' \Psi^\dagger(a)\Vert_{\cY}^2]
\\
&=
\E_{\bm{z}} [\Vert \cF(\bm{z}) - P''\cF(\bm{z}) \Vert_{\cY}^2]
\le
\E_{\bm{z}} [\Vert \cF(\bm{z}) - \sum_{{\bm{\nu}} \in \Lambda_m} t_{\bm{\nu}} \bm{z}^{\bm{\nu}} \Vert_{\cY}^2]
\\
&\le \sup_{\bm{z}\in U} \Vert \cF(\bm{z}) - \sum_{{\bm{\nu}} \in \Lambda_m} t_{\bm{\nu}} \bm{z}^{\bm{\nu}}\Vert_\cY^2
\le C d^{-2\alpha+2\eta}.
\end{align*}
The claimed result now easily follows by noticing that $\eta>0$ in the last estimate is arbitrary; we can thus repeat the argument with $\eta/2$ in place of $\eta$ to arrive at the claimed result.
\end{proof}

\subsection{Proof of Theorem \ref{thm:darcy}}
\label{app:darcy}

\begin{proof}
\label{pf:darcy}
Our goal is to estimate the following PCA-approximation error,
\[
\Err
=
\E_{a\sim \mu}\left[
\Vert \Psi(a) - \Psi^\dagger(a) \Vert_{\cY}^2
\right]^{1/2},
\]
where $\cY = H^1_0(D)$ is the output function space.
Let $\phi^{\cX}_1,\dots, \phi^{\cX}_{d_{\cX}} \in \cX$ denote the empirical PCA-basis on $\cX$. Let $\cE_{\cX}: \cX \to \R^{d_{\cX}}$ denote the PCA-encoder $a \mapsto \langle a, \phi^{\cX}_j \rangle_{j=1}^{d_{\cX}}$ on $\cX$. 
%We will consider the affine decoder $\tD_{\cX}: \R^{d_{\cX}} \to \cX$, given by 
%\[
%\tD_{\cX}(\langle a, \phi^{\cX}_{j}\rangle_{j=1}^{\dx})
%:=
%\bar{a} + \sum_{j=1}^\dx \langle a - \bar{a}, \phi^\cX_j\rangle \phi^{\cX}_j.
%\]
According to Lemma \ref{lem:err-decomp}, the PCA-Net approximation error can be bounded by $\Err \le \Err_{\cX}^\ast + \Err_{\cY}$, in terms of a reconstruction error on $\cY$, 
\[
\Err_{\cY} = 
\E_{w\sim \Psi^\dagger_\#\mu}\left[
\Vert w - \cD_\cY \circ \cE_\cY(w) \Vert^2_{\cY} 
\right]^{1/2},
\]
and a neural network approximation error
\[
\Err_{\cX}^\ast
=
\E_{a\sim \mu}\left[
\Vert \psi \circ \cE_{\cX}(a) - \cE_{\cY} \circ \Psi^\dagger(a) \Vert^2_{\ell^2}
\right]^{1/2}.
\]
By the assumptions of Section \ref{sec:darcy}, any $a\in \supp(\mu)$ can be written in the form $a(\bm{z}) = a(\slot;\bm{z}) = \bar{a} + \sum_{\ell=1}^\infty \gamma_\ell z_\ell \rho_\ell$, for an orthonormal set of functions $\rho_\ell \in \cX$, and for coefficients $\bm{z} = (z_\ell)_{\ell\in\N}\in U:= [-1,1]^\N$. Thus, $\Err_{\cX}^\ast$ can be estimated trivially by
\[
\Err_{\cX}^\ast
\le
\sup_{\bm{z}\in U} 
\left\Vert 
\psi\left(\langle a(\bm{z}), \phi^{\cX}_j\rangle\right) - \cE_{\cY} \circ \Psi^\dagger\left(a(\bm{z})\right) 
\right\Vert_{\ell^2}.
\]
Let us define a mapping $\cF: U \to \cY$, where $U=[-1,1]^\N$, by 
\[
\cF(\bm{z}) := \Psi^\dagger\left(\bar{a} + \textstyle\sum_{\ell=1}^\infty \gamma_\ell z_\ell \rho_\ell\right).
\]
We note that by Lemma \ref{lem:be-expansion}, for any $\eta > 0$, and any $m\in \N$, there exists a neural network $\psi^\star: [-1,1]^m \to \R^{\dy}$, such that 
\begin{align} \label{eq:Fapprox}
\sup_{\bm{z}\in U}
\left\Vert
\cE_{\cY} \circ \cF(\bm{z}) - \psi^\star(z_1,\dots, z_m)
\right\Vert_{\ell^2(\R^{\dy})}
\le C m^{-\alpha +\eta/2},
\end{align}
and with a constant $C = C(\cF,\alpha,s)>0$ (we choose $\eta/2$ rather than $\eta$ for later convenience). Furthermore, we have the following complexity estimate:
\[
\size(\psi^\star) \le Cm \rev{( \log(m)^2 + \dy)},
\quad
\depth(\psi^\star) \le C\log(m)^2.
\]
 Our next goal will be to find a map $\psi^{\star\star}: \R^{\dx} \to [-1,1]^m$, such that for any $a = a(\bm{z})$, the composition 
\[
\psi: \R^{\dx}\to\R^\dy, \quad \psi(w) := \psi^\star \circ \psi^{\star\star}(w),
\]
provides a good approximation 
\begin{align*}
\psi(\cE_\cX(a(\bm{z}))) 
&= 
\psi(\langle a(\bm{z}), \phi^\cX_1 \rangle, \dots, \langle a(\bm{z}), \phi^\cX_m \rangle) 
\\
&\approx \psi^\star(z_1,\dots, z_m)  \approx \cE_\cY \circ \cF(\bm{z}).
\end{align*}
We note that by the orthonormality of the $\rho_\ell$ in $\cX$, we have 
\[
z_\ell = \gamma_\ell^{-1} \langle a - \bar{a}, \rho_\ell \rangle,
\]
so that in principle, the coefficients $z_\ell$ could be readily obtained from knowledge of $a = a(\bm{z})$. The main difficulty here is that the neural network $\psi$ only has access to the PCA projections $\langle a, \phi^\cX_j\rangle$, which allows us to reconstruct the orthogonal projection $\Pi_{\cX}a := \cD_\cX \circ \cE_\cX(a)$ onto the empirical PCA subspace in $\cX$, but it doesn't generally allow us to reconstruct $a$ itself. Instead, we define
\[
\tilde{z}_\ell :=  \gamma_\ell^{-1} \langle \Pi_\cX (a - \bar{a}), \rho_\ell \rangle
=
\gamma_\ell^{-1} 
\left\{ 
\sum_{j=1}^{\dx} \langle a, \phi^{\cX}_j\rangle \langle \phi^{\cX}_j, \rho_\ell \rangle
\right\}
- \gamma_\ell^{-1}\langle \Pi_\cX\bar{a},\rho_\ell\rangle.
\]
We note that the map $\R^\dx \to \R^m$, taking $\alpha = (\alpha_1,\dots, \alpha_\dx)$ with $\alpha_j := \langle a, \phi^\cX_j \rangle$ and mapping it to $(\tilde{z}_1,\dots, \tilde{z}_m)$, with $\tilde{z}_\ell = \tilde{z}_\ell(\alpha)$ as above, defines an affine mapping.
Based on these $\tilde{z}_\ell(\alpha)$, we now define $\psi^{\star\star}: \R^\dx \to [-1,1]^m$, $\alpha \mapsto (\psi^{\star\star}_1(\alpha),\dots, \psi^{\star\star}_m(\alpha))$ by 
\[
\psi^{\star\star}_\ell(\alpha) := \mathrm{shrink}\left(\tilde{z}_\ell(\alpha)\right),
\quad
\text{where }
\mathrm{shrink}(z) := 
\begin{cases}
1 & z > 1, \\
z & z \in [-1,1], \\
-1 & z < -1.
\end{cases}
\]
We point out that $\psi^{\star\star}$ can be represented by a ReLU neural network, with
\[
\size(\psi^{\star\star}) \le C m \dx, \quad \depth(\psi^{\star\star}) \le 2.
\]
We now consider the composition $\psi(\alpha) := \psi^\star \circ \psi^{\star\star}(\alpha)$: Let $\bm{z}\in U$ be given. For the following estimate, we denote $\hat{z}_\ell := \mathrm{shrink}(\tilde{z}_\ell(\alpha))$, with $\alpha \in \R^\dx$ having components $\alpha_j = \langle a(\bm{z}), \phi^{\cX}_j \rangle$. We note that, by definition, $\psi^{\star\star}(\alpha) = (\hat{z}_1,\dots, \hat{z}_m)$. Let us denote $\hat{\bm{z}} = (\hat{z}_1,\dots, \hat{z}_m, 0, 0, \dots) \in U$. By construction, we then have
\begin{align}
\left\Vert 
\psi(\cE_\cX(a(\bm{z})))- \cE_\cY \circ \cF(\bm{z})
\right\Vert_{\ell^2(\R^\dy)}
&=
\left\Vert 
\psi^\star(\hat{z}_1,\dots, \hat{z}_m)  - \cE_\cY \circ \cF(\bm{z})
\right\Vert_{\ell^2(\R^\dy)}
\notag\\
&\le
\left\Vert 
\psi^\star(\hat{z}_1,\dots, \hat{z}_m) - \cE_\cY \circ \cF(\hat{\bm{z}})
\right\Vert_{\ell^2(\R^\dy)}
\notag\\
&\qquad 
+
\left\Vert 
\cE_\cY \circ \cF(\hat{\bm{z}})  - \cE_\cY \circ \cF(\bm{z})
\right\Vert_{\ell^2(\R^\dy)}
\notag
\\
&\le
C m^{-\alpha+\eta/2}
+
\left\Vert 
\cF(\hat{\bm{z}})  - \cF(\bm{z})
\right\Vert_{\cY},
\label{eq:F0}
\end{align}
where we have made use of \eqref{eq:Fapprox} in the last step, and the fact that the PCA encoding $\cE_\cY: \cY \to \R^{\dy}$ is a contractive mapping. We note that 
\begin{align}
\label{eq:FF}
\begin{aligned}
\Vert \cF(\hat{\bm{z}}) - \cF(\bm{z}) \Vert_{\cY}
&=
\left\Vert 
\Psi^\dagger\Big(\bar{a} +  \textstyle\sum_{\ell=1}^\infty \gamma_\ell \hat{z}_\ell \rho_\ell \Big) 
- \Psi^\dagger\Big(\bar{a} +  \textstyle\sum_{\ell=1}^\infty \gamma_\ell z_\ell \rho_\ell \Big)  \right\Vert_{\cY}
\\
&\le 
\Lip(\Psi^\dagger) \left\Vert  \textstyle\sum_{\ell=1}^\infty \gamma_\ell [\hat{z}_\ell - z_\ell] \rho_\ell \right \Vert_{L^\infty}
\\
&\le
C^\ast \Lip(\Psi^\dagger) \left\Vert \textstyle\sum_{\ell=1}^\infty \gamma_\ell [\hat{z}_\ell - z_\ell] \rho_\ell \right \Vert_{\cX}.
\end{aligned}
\end{align}
Where the last estimate follows from the assumed embedding $\cX \embeds L^\infty$.
To estimate the last expression, we note that by the properties of the shrink operator, since $\hat{z}_\ell = \shrink(\tilde{z}_\ell)$, and the fact that $z_\ell \in [-1,1]$, we have 
\[
|\hat{z}_\ell - z_\ell| = |\shrink(\tilde{z}_\ell) - \shrink(z_\ell)| \le |\tilde{z}_\ell - z_\ell|.
\]
Furthermore, we note that
\[
\sum_{\ell=1}^m \gamma_\ell z_\ell \rho_\ell = P_m (a-\bar{a}),
\]
with $P_m$ the orthogonal projection onto $\Span\{\rho_1,\dots, \rho_m\}$, and similarly, by definition of $\tilde{z}_\ell$, we have
\[
\sum_{\ell=1}^m \gamma_\ell \tilde{z}_\ell \rho_\ell = P_m \Pi_\cX(a-\bar{a}).
\]
Together with the orthonormality of the $\rho_\ell$, and the fact that $\tilde{z}_\ell = 0$ for $\ell > m$, this then implies that 
\begin{align}
\label{eq:FFF}
\begin{aligned}
\left\Vert \sum_{\ell=1}^\infty \gamma_\ell[\hat{z}_\ell - z_\ell] \rho_\ell \right \Vert_{\cX}^2
&= \sum_{\ell=1}^\infty \gamma_\ell^2 |\hat{z}_\ell - z_\ell|^2 
\le \sum_{\ell=1}^\infty \gamma_\ell^2 |\tilde{z}_\ell - z_\ell|^2
\\
&=
\left\Vert \sum_{\ell=1}^m \gamma_\ell [\tilde{z}_\ell - z_\ell] \rho_\ell \right \Vert_{\cX}^2
+
\sum_{\ell>m} \gamma_\ell^2 z_\ell^2
\\
&\le
\Vert \Pi_{\cX}(a-\bar{a}) - (a-\bar{a}) \Vert_{\cX}^2
+
\sum_{\ell>m} \gamma_\ell^2,
\end{aligned}
\end{align}
\rev{where we made use of the fact that $P_m$ is an orthogonal projection, and hence
\[
\left\Vert \sum_{\ell=1}^m \gamma_\ell [\tilde{z}_\ell - z_\ell] \rho_\ell \right \Vert_{\cX}
= 
\Vert P_m \Pi_\cX(a-\bar{a}) - P_m(a-\bar{a}) \Vert_{\cX}
\le 
\Vert \Pi_\cX(a-\bar{a}) - (a-\bar{a}) \Vert_{\cX},
\]
to pass to the last line.}

By assumption on $\gamma_\ell \le C \ell^{-1-\alpha}$, we can estimate \rev{
\begin{align*}
\sum_{\ell > m} \gamma_\ell^2
&\le C \sum_{\ell>m} \ell^{-2-2\alpha}
\le C \int_{m}^\infty \ell^{-2-2\alpha} d\ell
= \frac{C}{1+2\alpha} m^{-1-2\alpha}.
\end{align*}
} We also note that the Lipschitz constant of $\Psi^\dagger: \supp(\mu)\subset \cX \to \cY$ depends only on $\lambda$, $\Lambda$ and the domain $D \subset \R^n$. Thus, combining \eqref{eq:F0}, \eqref{eq:FF} and \eqref{eq:FFF} with the last estimate, it follows that
\begin{align*}
\Vert \psi\circ \cE_{\cX}(a) - \cE_{\cY} \circ \Psi^\dagger(a) \Vert_{\ell^2(\R^\dy)}^2
&= \Vert \psi\circ \cE_{\cX}(a(\bm{z})) - \cE_{\cY} \circ \cF(\bm{z}) \Vert_{\ell^2(\R^\dy)}^2
\\
&\le
C\left( 
m^{-2\alpha + \eta} + \Vert \Pi_{\cX}(a-\bar{a}) - (a-\bar{a}) \Vert_{\cX}^2
\right),
\end{align*}
for a constant $C = C(\lambda,\Lambda,d,\Psi^\dagger,\alpha,\eta) > 0$. In the last estimate,  we have absorbed the additional term proportional to $
m^{-2\alpha-1}$ in the first (larger) term $m^{-2\alpha + \eta}$. Taking the expectation over $a \sim \mu$, and noting that $\Pi_\cX\bar{a} - \bar{a} = \E_{a\sim \mu}[\Pi_\cX a - a]$ implies,
\[
\E_{a\sim\mu} \left[\Vert \Pi_{\cX}(a-\bar{a}) - (a-\bar{a}) \Vert_{\cX}^2 \right]
\le
\E_{a\sim\mu} \left[\Vert \Pi_{\cX}a - a \Vert_{\cX}^2 \right]
=
\Err_{\cX}^2,
\]
we conclude that
\[
\Err_{\cX}^\ast{}^2
\le
C \left( m^{-2\alpha+\eta} + \Err_\cX^2 \right),
\]
where we are free to choose $\eta > 0$, and $C = C(\lambda,d,\Psi^\dagger,\alpha,\eta) > 0$ is independent of $m$, $\dx$ and $\dy$. Combining the above estimates, we obtain
\[
\E_{a\sim\mu}[\Vert \Psi^\dagger(a) - \Psi(a) \Vert_{\cY}^2] = \Err^2 \le 2\Err_{\cX}^\ast{}^2 + 2\Err_{\cY}^2
\le 
C \left( m^{-2\alpha+\eta} + \Err_\cX^2+ \Err_{\cY}^2 \right).
\]
Next, we recall that $\mu$ and $\Psi^\dagger_\#\mu$ are concentrated on bounded sets; indeed, this is the case by assumption for $\mu$, and is a consequence of the a priori bound \eqref{eq:apriori} in the case of $\Psi^\dagger_\#\mu$. Invoking Proposition \ref{prop:pcaproj}, we can therefore estimate
\begin{align*}
\Err_\cX^2 &\le \cR^\opt_{\dx}(\mu) + \sqrt{\frac{Q_\cX \dx\log(1/\delta)}{N}},
\\
\Err_\cY^2 &\le \cR^\opt_{\dy}(\mu) + \sqrt{\frac{Q_\cY \dy\log(1/\delta)}{N}},
\end{align*}
with probability at least $1-\delta$, and with constants $Q_\cX, Q_\cY$ depending only on $\mu$ and $\Psi^\dagger_\#\mu$, respectively.
The optimal PCA projection errors have been bounded in Proposition \ref{prop:Darcy-enc-err}, showing that $\cR^\opt_{\dx}(\mu) \le C \dx^{-2\alpha-1}$ and $\cR^\opt_{\dy}(\Psi^\dagger_\#\mu) \le C \dy^{-2\alpha + \eta}$. Therefore, with the assumed choice of $\dx = \dy =d$ and number of PCA samples $N \ge d^{1+4\alpha} \log(1/\delta) $, we can now estimate
\[
\Err_\cX^2 + \Err_\cY^2 \le C d^{-2\alpha + \eta},
\]
with probability at least $1-\delta$.

To summarize: Fix $\eta > 0$. We have shown that there exists a constant $C=C(\lambda,\Lambda,d,\Psi^\dagger,\alpha,\eta)>0$, such that for any $m\in \N$, there exists a neural network $\psi: \R^\dx \to \R^\dy$, satisfying the complexity estimate
\[
\size(\psi) \le C m (\log(m)^2 + d), 
\quad
\depth(\psi) \le C \log(m)^2,
\]
and such that the (squared) PCA-Net error $\Err^2$ for $\Psi = \cD_\cY \circ \psi \circ \cE_\cX$ can be bounded by
\begin{align*}
\E_{a\sim\mu}[\Vert \Psi^\dagger(a) - \Psi(a)\Vert_{\cY}^2] 
&\le C \left(\Err_{\cX}^2 +  \Err_{\cY}^2 + m^{-2\alpha + \eta} \right)
\\
&\le C(d^{-2\alpha+\eta} + m^{-2\alpha+\eta}),
\end{align*}
with probability at least $1-\delta$ in the $N \sim d^{1+4\alpha} \log(1/\delta)$ PCA samples.
Setting $m \sim d \sim \epsilon^{-1/(2\alpha-\eta)} = \epsilon^{-\frac{1}{2\alpha} - \eta'}$, for some $\eta'=\eta'(\eta)$, now yields the claimed error and complexity bounds.
\end{proof}

\section{Proofs for Navier-Stokes}
\label{app:ns}

This appendix contains the background and proof of the PCA-Net emulation result for the Navier-Stokes equations, Theorem \ref{thm:ns}. For simplicity, the Theorem in the main text was only stated for $n=2$. Here, we allow for arbitrary $n$ under a suitable smoothness assumption on the underlying solutions. As indicated in the main text, the relevant smoothness assumption is only known to hold for spatial dimension $n=2$.

\subsection{A convergent numerical scheme} \label{sec:scheme}

For $K\in \N$, let $\P_K: L^2(\T^n;\R^n) \to L^2(\T^n;\R^n)$ denote the $L^2$-orthogonal projection onto divergence-free real-valued vector fields of the form, 
\[
u_K(x) := \sum_{|k|_\infty\le K} \hat{u}_k e^{i\langle k, x\rangle}.
\]
We will also denote $L^2_K := \P_K (L^2(\T^n;\R^n))$ as the image of $L^2$ under $\P_K$.

We now consider the numerical discretization of \eqref{eq:ns}, obtained by the following recursion: Let $\Delta t > 0$ be a (small) time-step. We assume we are given (discrete) initial data $u^0_K \in L^2_K$, $u^0_K \approx \P_K \bar{u}$. For $m\ge 0$, determine  $u^{m+1}_K \in L^2_K$ by solving
\begin{align} \label{eq:scheme}
\frac{u^{m+1}_K - u^m_K}{\Delta t}
=
-\P_K \left(
u^m_K \cdot \nabla u^{m+1/2}_K
\right)
+ \nu \Delta u^{m+1/2}_K,
\end{align}
where $u^{m+1/2}_K := \frac12 \left(u^{m+1}_K + u^m_K\right)$. The recursion \eqref{eq:scheme} is equivalent to a recursion for the corresponding Fourier coefficients $\hat{u}_k^{m} \in \C^n$, $k\in \Z^n$, ${|k|_\infty\le K}$, where
\[
u^m_K(x) = \sum_{|k|_\infty\le K} \hat{u}^m_k e^{i\langle k,x\rangle},
\]
which are uniquely defined by (and equivalent to) the scheme \eqref{eq:scheme}: The Fourier transformed version of \eqref{eq:scheme} is given by
\begin{align} \label{eq:scheme-ft}
\frac{\hat{u}^{m+1}_k - \hat{u}^m_k}{\Delta t}
=
-\hat{\P}_K \left(
u^m_K \cdot \nabla u^{m+1/2}_K
\right)
- \nu |k|^2 \hat{u}^{m+1/2}_k,
\end{align}
where we introduce the shorthand notation
\begin{align}
\hat{\P}_K \left(
u^m_K \cdot \nabla u^{m+1/2}_K
\right)
:=
\left(
1 - \frac{k\otimes k}{|k|^2}
\right)
\sum_{|\ell|_\infty, |k-\ell|_\infty\le K}
i(\ell\cdot \hat{u}_{k-\ell}^{m+1/2}) \hat{u}_\ell^{m},
\end{align}
which is the Fourier transform of the non-linear term $\P_K\left(u^m_K \cdot \nabla u^{m+1/2}_K\right)$.

 In \citep[Lem. 53]{thfno}, it has been shown that this numerical scheme is well-defined for time-steps satisfying $\Vert \bar{u}\Vert_{L^2} K^{n/2+1} \Delta t \le \frac12$, and that, if $t\mapsto u(t) \in L^2(\T^n;\R^n)$ denotes a smooth solution of \eqref{eq:ns}, and $t^m:= m\Delta t$, then we have
\begin{align} \label{eq:scheme-error}
\begin{aligned}
\frac{u(t^{m+1}) - u(t^m)}{\Delta t}
&=
-\P_K \left(
u(t^m) \cdot \nabla \frac12\left[u(t^{m+1}) + u(t^m)\right]
\right)
\\
&\qquad
+ \nu \Delta \frac12\left[u(t^{m+1}) + u(t^m)\right] + R_K^m(u),
\end{aligned}
\end{align}
where the remainder $R_K^m(u)$ can be estimated by
\begin{align} \label{eq:remainder}
\Vert R_K^m(u) \Vert_{L^2}^2
\le
C \left(
\Delta t^2 + M^{-2r}
\right),
\end{align}
for any $r\ge n/2+1$, and where the constant $C>0$ depends only on
$\Vert u \Vert_{C_tH^r_x}, \Vert u \Vert_{C^1_tH^{r-1}_x}$. Here, we have introduced $C_tH^r_x := C([0,T];H^r_x)$, $C_t^1H^{r-1}_x:= C^1([0,T];H^{r-1}_x)$, given as the space of vector fields $u: [0,T] \to H^r_x(\T^n;\R^n)$ with continuous dependence on time, and the space of vector fields $u: [0,T] \to H^{r-1}_x(\T^n;\R^n)$ with continuously differentiable dependence on time, respectively.

\subsection{Neural network emulation of scheme \eqref{eq:scheme}} \label{sec:scheme-emulation}

Our first goal is to show that the numerical scheme \eqref{eq:scheme-ft} can be efficiently emulated by a neural network. To this end, we will first construct a neural network $\NL_{M,\epsilon}(\hat{u}, \hat{v})$, acting on the Fourier coefficients with $|k|_\infty\le K$, such that 
\[
\Vert 
\NL_{M,\epsilon}(\hat{u},\hat{v}) - \hat{\P}_K\left(
u_K \cdot \nabla v_K
\right)
\Vert_{L^2}
\le\epsilon,
\]
for all $\hat{u}, \hat{v}$ with $\Vert \hat{u} \Vert_{\ell^2}, \Vert \hat{v} \Vert_{\ell^2} \le M$. In the following, we will denote for a given $K\in \N$, by $\cK = \cK_K$ the following set of multi-indices:
\[
\cK_K = \set{k\in \Z^n}{|k|_\infty \le K}.
\]
We will denote by $\cJ_K$ ($\cJ_{2K}$) an equidistant grid with $2K+1$ ($4K + 1$) grid points $\{x_j\}_{j\in \cJ_K}$ in each direction. We will identify $\C^\cK \sim \R^{2|\cK|}$, i.e. the real and imaginary parts of any vector  $\hat{w}\in \C^\cK$ are considered as separate components of the corresponding element of $\R^{2|\cK|}$ under this identification, and all neural networks are understood to act componentwise. We can then state the following approximation result:

\begin{lemma} \label{lem:nonlin}
There exists a constant $C>0$, such that for any $\epsilon, M>0$ and $K\in \N$, there exists a ReLU neural network $\NL_{M,\epsilon}: \C^{\cK}\times \C^{\cK} \to \C^{\cK}$, $\cK = \cK_K$, satisfying the error bound
\begin{align} \label{eq:err}
\Vert 
\NL_{M,\epsilon}(\hat{u},\hat{v}) - \hat{\P}_K\left(
u_K \cdot \nabla v_K
\right)
\Vert_{\ell^2}
\le \epsilon,
\end{align}
and the Lipschitz bound
\begin{align} \label{eq:Lip}
\Vert \NL_{M,\epsilon}(\hat{u}, \hat{v}) - \NL_{M,\epsilon}(\hat{u}, \hat{w}) \Vert_{\ell^2}
\le
CK^{n/2+1} M \Vert \hat{v} - \hat{w} \Vert_{\ell^2},
\end{align}
for all $\hat{u}, \hat{v}, \hat{w}\in \C^{\cK}$, such that $\Vert \hat{u} \Vert_{\ell^2}, \Vert \hat{v} \Vert_{\ell^2}, \Vert \hat{w} \Vert_{\ell^2} \le M$. Furthermore, the size of the neural network $\NL_{M,\epsilon}$ can be estimated by
\[
\size(\NL_{M,\epsilon}) \le C |\cK_K| \log(MK/\epsilon),
\quad
\depth(\NL_{M,\epsilon}) \le C \log(MK/\epsilon).
\]
where $C> 0$ is independent of $K,M,\epsilon$.
\end{lemma}

The proof of Lemma \ref{lem:nonlin} is provided in Appendix \ref{app:nonlin}. Using the above lemma, we can then show that the approximate equation 
\begin{align} \label{eq:scheme-emulate}
\frac{\hat{u}^{m+1}_k - \hat{u}^m_k}{\Delta t}
=
-\NL_{M,\epsilon}(\hat{u}^m,\hat{u}^{m+1/2})
- \nu |k|^2 \hat{u}^{m+1/2}_k,
\end{align}
has a unique solution at each time-step, for suitably chosen $\Delta t,M,\epsilon>0$.

\begin{lemma} \label{lem:update}
If $\Vert \hat{u}^m\Vert_{L^2} \le M$ for some $M>0$, and if the non-linearity $\NL_{\bar{M},\epsilon}(\hat{u},\hat{v}) \mapsto \NL_{\bar{M},\epsilon}(\hat{u},\hat{v})$ satisfies the error estimate \eqref{eq:err} for a given $0<\epsilon \le 1$ with the Lipschitz estimate \eqref{eq:Lip} for $\bar{M} = 2(M+2)$, and if the time-step $\Delta t \in (0,1]$ satisfies the CFL condition 
\begin{align}
C\Delta t K^{n/2+1}\bar{M} \le 1,
\end{align}
then the update rule \eqref{eq:scheme-emulate} is well-posed, in the sense that a unique solution $\hat{u}^{m+1}$ exists with $\Vert \hat{u}^{m+1} \Vert_{L^2} \le \bar{M}$. Furthermore, $\hat{u}^{m+1}$ can be approximated via the following fixed-point iteration: Setting $\hat{w}^0_k = 0$, and
\[
\hat{w}^{\ell+1}_k := \frac{1}{1+\frac12 \Delta t \,\nu|k|^2}\left[
\hat{u}^m + \Delta t \, \NL_{\bar{M},\epsilon}\left(\hat{u}^m, \frac{\hat{u}^m + \hat{w}^\ell}{2}\right){}_k - \frac12 \Delta t\, \nu|k|^2 \hat{u}^m
\right],
\]
 for $\ell=0,1,\dots$.
Then we have the following a priori estimates, valid for any $\ell=0,1,\dots$:
\[
\Vert \hat{w}^{\ell} - \hat{u}^{m+1} \Vert_{\ell^2}
\le
\frac{\exp\left(2\Delta t\, \epsilon \right)}{2^\ell}  \Vert \hat{u}^{m} \Vert_{\ell^2}, 
\]
and
\[
\Vert \hat{w}^{\ell} \Vert_{\ell^2}
\le
\left(1 + 2^{-\ell}\right) \exp\left(2\Delta t\, \epsilon \right) \Vert \hat{u}^{m} \Vert_{\ell^2}.
\]
\end{lemma}

Given the well-posedness of the update rule in the emulation scheme \eqref{eq:scheme-emulate}, and the approximation by a simple fixed-point iteration provided in Lemma \ref{lem:update}, it is natural to define the following neural network-based emulation algorithm:

\SetKwComment{Comment}{$\triangleright$ }{}
\SetKwInOut{Input}{Input}  
\RestyleAlgo{ruled}
\begin{algorithm}[H]
\caption{Navier-Stokes NN-emulation}\label{alg:emulation}
\Input{%
parameters $M\ge 1$, $r>0$, $r>n/2+1$, final time $T > 0$, 
\\
$\hat{u}^0$ initial data, with $\Vert \hat{u}^0 \Vert_{\ell^2}\le M$.}
\KwResult{$\hat{u}^{m_T}$, where $n_T \Delta t = T$.}
\DontPrintSemicolon
\;
Let $C$ be the constant of estimate \eqref{eq:Lip} for $\NL_{\bar{M},\epsilon}$.\;
\Comment*[r]{Note:~$C$ is independent of $\bar{M},\epsilon$!}
Set $\bar{M} \leftarrow 2(eM+1)$.\Comment*[r]{$e = \exp(1)$ is Euler's number.}
Set $\bar{T} := \max(T,1)$\;
Set $\Delta t\in (0,1]$ maximal, s.t.:
\[
 CK^{n/2+1}\bar{M}\Delta t \le 1, \quad  \Delta t \le K^{-r}, \quad n_T = T/\Delta t \in \N.
\]
Set $\epsilon \leftarrow \Delta t / \left(3M\bar{T}\right)$.\;%\Comment*[r]{Notation: $a\vee b = \max(a,b)$}
Set $L \leftarrow \lceil \log_2(1/\Delta t \,\epsilon) \rceil = \lceil \log_2\left(3M\bar{T}/\Delta t^2\right) \rceil$.\;
\;
Construct a neural network $\NL_{\bar{M},\epsilon}$ as in Lemma \ref{lem:nonlin}.\;
\For{$m=0,1,\dots, n_T-1$}{
    Set $\hat{w}^{m,0} \leftarrow 0$.\;
    Define $F(\hat{w}) = F(\hat{u}^m;\hat{w})$ by \eqref{eq:Ffixedpt} with $\NL_{\bar{M},\epsilon}$.\;
    \For{$\ell=0,1,\dots, L-1$}{
        Set $\hat{w}^{m,\ell+1} \leftarrow F(\hat{u}^m;\hat{w}^{m,\ell})$.\;
    }
    Set $\hat{u}^{m+1} \leftarrow \hat{w}^{m,L}$.\;
}
\end{algorithm}
\vspace{1em}
\par
We first note that the recursive algorithm \ref{alg:emulation} can be represented by either a deep neural network, or a recursive neural network:

\begin{lemma} \label{lem:nn-representation}
There exists $C>0$, such that for any $K\in \N$ and $r>n/2+1$, $M\ge 1$, $T>0$, there exists a neural network $\psi: \C^\cK \to \C^\cK$ of size
\[
\left\{
\begin{aligned}
\size(\psi) &\le C M\bar{T} K^{n+r} \log(M\bar{T}K)^2,
\\
\depth(\psi) &\le CM\bar{T} K^{r} \log(M\bar{T}K), 
\end{aligned}
\right.
\]
where $\bar{T} := \max(T,1)$, which maps the input of algorithm \ref{alg:emulation} to its output. In fact, $\psi$ can be written in the form of a \emph{recurrent} neural network, or more precisely, as a composition
\[
\psi = 
\underbrace{
\psi^\ast \circ \dots \circ \psi^\ast
}_{
\text{$n_T$-fold}
}:
\C^{\cK}
\to 
\C^{\cK}
\]
of a neural network $\psi^\ast: \C^{\cK} \to \C^{\cK}$, 
where $n_T \le CMT K^{r}$ and 
\[
\size(\psi^\ast) \le C K^{n} \log(M\bar{T}K)^2,
\quad
\depth(\psi^\ast) \le C\log(M\bar{T}K). 
\]
\end{lemma}

For completeness, a detailed proof of Lemma \ref{lem:nn-representation} is provided in Appendix \ref{app:nn-representation}. Before providing an error estimate on the sequence $\hat{u}^m \approx \hat{u}(t^m)$ generated in Algorithm \ref{alg:emulation}, we derive an improved upper bound on the norms $\Vert \hat{u}^m\Vert_{L^2}$, which is required in order to ensure that the emulating neural network nonlinearity $\NL_{\bar{M},\epsilon}$ in \eqref{eq:scheme-emulate} provides an accurate approximation of the non-linear term $\P_K(u^m_K \cdot u^{m+1}_K )$, uniformly over all time-steps.
\begin{lemma} \label{lem:normbound}
If $\hat{u}^{m+1} = \hat{w}^{m,L}$, for $m=0,\dots, n_T-1$, is determined as in Algorithm \ref{alg:emulation}, then we have 
\begin{align} \label{eq:normbound}
\Vert \hat{u}^m \Vert_{\ell^2}
\le 
\exp(3m\Delta t \, \epsilon) M \le e M,
\end{align}
for all $m=0,\dots, n_T$.
\end{lemma}

The simple proof of this lemma is provided in Appendix \ref{app:normbound}. Based on the last lemma, we then obtain the following:

\begin{lemma} \label{lem:emulation}
If $\hat{u}^{m+1} = \hat{w}^{m,L}$, for $m=0,\dots, n_T-1$, is determined as in Algorithm \ref{alg:emulation}, then the corresponding function $u^{m+1}_K := \sum_{|k|_\infty \le K} \hat{u}^{m+1}_k e^{i\langle k, x\rangle}$ satisfies 
\begin{align} \label{eq:emulation}
\frac{u^{m+1}_K - u^m_K}{\Delta t}
&=
-\P_K\left(
u^m_K \cdot \nabla u^{m+1/2}_K
\right)
+
\nu \Delta u^{m+1/2}_K
+
Q_K^m,
\end{align}
where the remainder $Q_K^m$ satisfies a bound of the form 
\begin{align} \label{eq:remainder-q}
\sup_{m} 
\Vert 
Q_K^m
\Vert_{L^2}
\le C \Delta t,
\end{align}
where $C >0$ is independent of $K$, $M$ and $T$.
\end{lemma}

\subsection{Error estimate for neural network emulation} \label{sec:scheme-error}

With Lemma \ref{lem:emulation} in hand, we can now prove the following error estimate for the approximation of the underlying solution, $u^m_K \approx u(t^m)$:
\begin{lemma} \label{lem:emulation-estimate}
Let $u^m_K = \sum_{|k|_\infty \le K} \hat{u}^m_k e^{i\langle k,x\rangle}$ be obtained by the emulation algorithm \ref{alg:emulation} with parameters $M,r,T>0$, $r>n/2+1$, and initial data $\Vert \hat{u}^0 \Vert_{\ell^2}\le M$. Let $u(t)$ be the exact solution of \eqref{eq:ns} with the initial data $\bar{u}\in H^r(\T^n;\R^n)$. Assume that $u \in C([0,T];H^r_x) \cap C^1([0,T];H^{r-1}_x)$. Then $u^m_K$ satisfies
\[
\max_{m=0,\dots, n_T} \Vert u^m_K - u(t^m) \Vert_{L^2} 
\le
C\left( K^{-r} + \Vert \hat{u}^0 - \hat{\P}_K u(0) \Vert_{\ell^2} \right),
\]
where $C = C(\Vert u \Vert_{C_tH^r_x}, \Vert u \Vert_{C^1_tH^{r-1}_x},T,M,n,r) > 0$, is independent of $K$, and where $\hat{\P}_K  u(0) = \{\hat{u}_k(t=0)\}_{|k|_\infty \le K}$ denote the Fourier coefficients of the initial data $\bar{u} = u(0)$ with wavenumbers $|k|_\infty \le K$.
\end{lemma}

To simplify notation in the following, we introduce $C_tH^r_x := C([0,T];H^r_x)$, $C_t^1H^{r-1}_x:= C^1([0,T];H^{r-1}_x)$, and define
\begin{align} \label{eq:ns-space}
\cH^r_{M,T}
=
\set{
u \in C_tH^r_x\cap C^1_tH^{r-1}_x
}{
u \text{ solves \eqref{eq:ns} and }
\Vert u \Vert_{\cH^r} \le M
},
\end{align}
where 
\[
\Vert u \Vert_{\cH^r}
:=
\Vert u \Vert_{C_tH^r_x} + \Vert u \Vert_{C^1_tH^{r-1}_x}.
\]
Finally, combining Lemma \ref{lem:emulation-estimate} with Lemma \ref{lem:nn-representation}, we obtain the following estimate for the emulation of the spectral method \eqref{eq:scheme} by ReLU neural networks:

\begin{proposition} \label{prop:nsd-nn}
Let $M, \, r, \, T > 0$ be given. For any $\epsilon > 0$, there exists $K\in \N$, $K\sim \epsilon^{-1/r}$, and a ReLU neural network $\hat{\psi}: \C^{\cK} \to \C^{\cK}$, $\cK = \cK_K$, such that 
\[
\Vert \hat{\psi}(\hat{u}(0)) - \hat{u}(T) \Vert_{\ell^2}
\le
\epsilon, 
\qquad
\forall \, u\in \cH^r_{M,T},
\]
where $\hat{u}(t) = \{\hat{u}_k(t)\}_{|k|_\infty\le K}$ denote the Fourier coefficients of $u(t)$, and such that 
\[
\left\{
\begin{aligned}
\size(\hat{\psi}) &\le C \epsilon^{-n/r - 1} \log(\epsilon^{-1})^2,
\\
\depth(\hat{\psi}) &\le C \epsilon^{-1} \log(\epsilon^{-1})^2.
\end{aligned}
\right.
\]
The constant $C>0$, depends on $M, T, r,n$, but is independent of $\epsilon$. Furthermore, $\hat{\psi}$ can be written as a $n_T$-fold composition $\hat{\psi} = \hat{\psi}_\ast \circ \dots \circ \hat{\psi}_\ast$, where $\hat{\psi}_\ast: \C^\cK \to \C^{\cK}$ is a ReLU neural network with 
\[
\left\{
\begin{aligned}
\size(\hat{\psi}_\ast) &\le C \epsilon^{-n/r} \log(\epsilon^{-1})^2,
\\
\depth(\hat{\psi}_\ast) &\le C \log(\epsilon^{-1})^2,
\end{aligned}
\right.
\]
and $n_T \le C \epsilon^{-1}$.
\end{proposition}

Proposition \ref{prop:ns-nn} in the two-dimensional case, $n=2$, is an immediate consequence of the above result.

\subsection{An error estimate for Navier-Stokes} \label{sec:PCA-error}

In the last section, we have obtained quantitative estimates on the ReLU DNN emulation of the spectral scheme \eqref{eq:scheme}. This provides an approximate mapping from the Fourier coefficients of the input function, to the Fourier coefficients of the output function. In the present section, we combine this idea with the PCA projection, to construct a PCA-Net approximation of the Navier-Stokes solution operator.

To obtain an estimate on the approximation error $\Err(\Psi)$ for the Navier-Stokes equations, we fix parameters $M,r>0$, $r>n/2+1$, and a final time $T>0$, and consider the operator $\Psi^\dagger: L^2(\T^n;\R^n) \to L^2(\T^n;\R^n)$, given by the forward solution operator from time $t=0$ to $t=T$ of \eqref{eq:ns}. To ensure that $\Psi^\dagger(u)$ is well-defined for $\mu$-almost every $u$, we assume that 
\begin{align}
\mu\left( \set{u(0)}{u\in \cH^r_{M,T}} \right) = 1,
\end{align}
i.e. that $\mu$ is concentrated on initial data for which the corresponding solution $u$ belongs to $\cH^r_{M,T}$ (cp. \eqref{eq:ns-space} for the definition of $\cH^r_{M,T}$).

\begin{lemma}
\label{lem:NSd-approx-err}
Given the above setting, for any $K\in \N$ (corresponding to Fourier cut-off for wavenumbers $\vert k\vert_{\infty} \le K$), there exists a PCA-Net $\Psi = \cD_\cY \circ \psi \circ \cE_\cX$, such that 
\begin{align*}
\E_{u\sim \mu}\left[
\Vert \Psi^\dagger(u) - \Psi(u) \Vert_{L^2}^2
\right]^{1/2}
&\le
C K^{-r} + C\E_{u\sim \mu}\left[\Vert u - P_\cX u \Vert_{L^2}^2\right]^{1/2} 
\\
&\qquad + \E_{v \sim \cG_\#\mu}\left[\Vert v - P_\cY v \Vert_{L^2}^2\right]^{1/2},
\end{align*}
where $\psi$ is a neural network of size 
\[
\size(\psi) \le C K^n \left( K^{r} \log(K)^2 + (d_\cX + d_\cY) \right), 
\quad
\depth(\psi) \le C K^r \log(K),
\]
and $C = C(M,r,T,n) > 0$ is a constant independent of $K, d_\cX, d_\cY$.
\end{lemma}

The proof of Lemma \ref{lem:NSd-approx-err} is given in Appendix \ref{app:NSd-approx-err}. Setting $K \sim \epsilon^{-1/r}$ and specializing to $n=2$ immediately implies Lemma \ref{lem:NS-approx-err} in the main text. We finally provide an estimate on the PCA projection errors, which is the general $n$-dimensional version of Lemma \ref{lem:NS-approx-err} in the main text:
\begin{lemma}
\label{lem:NSd-enc-err}
Assume that 
\[
\mu\left(
\set{u(0)}{u\in \cH^r_{M,T}}
\right) =1,
\]
and the operator $\Psi^\dagger$ is given by
\[
\Psi^\dagger(u(0)) = u(T), \quad \forall \, u\in \cH^r_{M,T},
\]
where $u(t)$ solves the Navier-Stokes equations \eqref{eq:ns} with initial data $u(0)$. Then there exists a constant $C = C(M,r,T,n) > 0$, such that for PCA dimensions $\dx = \dy$, arbitrary $\delta > 0$ and $N \gtrsim \dx^{1+4r/n} \log(1/\delta)$ PCA samples, the empirical PCA projection errors based on $N$ samples $u_1,\dots, u_N \simiid \mu$ satisfy,
\begin{align*}
\E_{u \sim \mu}\left[\Vert u - P_\cX u \Vert_{L^2}^2\right]
+
\E_{v \sim \Psi^\dagger_\#\mu}\left[\Vert v - P_\cY v \Vert_{L^2}^2\right]
&\le
C \dx^{-2r/n},
\end{align*}
with probability at least $1-\delta$.
\end{lemma}

The proof of Lemma \ref{lem:NSd-enc-err} is immediate by combining the general estimate on empirical PCA, Proposition \ref{prop:pcaproj}, with the upper bound on the PCA projection error $R_p(\mu)$ based on the smoothness of the underlying functions, Proposition \ref{prop:PCAbound}. We only need to observe that $\mu$ and $\Psi^\dagger_\#\mu$ are concentrated on a bounded set $\{\Vert u \Vert_{H^r}\le M\}$, by assumption. 

\begin{proof}[Proof of Theorem \ref{thm:ns}]
Setting $K \sim \epsilon^{-1/r}$ in Lemma \ref{lem:NS-approx-err}, with $\Psi = \cD_\cY \circ \psi \circ \cE_\cX$, and combining it with the encoding error estimate of Lemma \ref{lem:NSd-enc-err} for $\dx=\dy$, we arrive at 
\begin{align*}
\E_{u\sim \mu}
\left[
\Vert \Psi^\dagger(u) - \Psi(u) \Vert_{L^2(\mu)}^2
\right]
\le
\epsilon^2 + C
d^{-2r/n},
\end{align*}
where $\psi$ is a neural network of size 
\[
\size(\psi) \le C \epsilon^{-n/r} \left( \epsilon^{-1} \log(\epsilon^{-1})^2 + \dx \right), 
\quad
\depth(\psi) \le C \epsilon^{-1} \log(\epsilon^{-1}),
\]
and $C = C(M,r,T,n) > 0$ is a constant independent of $\epsilon$, $\dx$. For spatial dimension $n = 2$, and upon replacing $\epsilon$ by $\epsilon^{1/2}$ throughout, and choosing $\dx \sim \epsilon^{-1/r}$, this yields the estimate of Theorem \ref{thm:ns}. 
\end{proof}

The remaining sections contain detailed proofs of all Lemmas above.

\subsection{Proof of Lemma \ref{lem:nonlin}}
\label{app:nonlin}

\begin{proof}
  We note that the mapping $\C^\cK \times \C^\cK \mapsto \C^\cK$, $(\hat{u},\hat{v}) \mapsto \hat{\P}_K\left( u_K \cdot \nabla v_K\right)$ can be written as a composition,
  \begin{align*}
    (\hat{u}_k,\hat{v}_k)_{k\in \cK}
    &\mapsto
    (\hat{u}_k, k\otimes \hat{v}_k)_{k\in \cK}
    \\
    &
    \mapsto
      (u_K(x_j), \nabla v_K(x_j))_{j\in \cJ_{2K}}
    \\
    &
    \mapsto
    ([u_K\cdot \nabla v_K] (x_j))_{j\in \cJ_{2K}}
    \\
    &
    \mapsto
    \hat{\P}_K\left( u_K \cdot \nabla v_K \right).
  \end{align*}
  The first mapping can trivially be represented by a matrix multiplication with $O(|\cK|)$ non-zero entries, the second step can be efficiently computed via the fast Fourier-transform, which can be represented exactly by a ReLU neural network with $O(|\cK|\log(|\cK|))$ non-zero entries (and depth $O(\log(|\cK|))$). Thus, there exists a (absolute) constant $C = C(n)>0$, and a ReLU neural network $\psi^\star: \C^\cK\times \C^\cK \to [\R^n]^{\cJ} \times [\R^{n\times n}]^{\cJ}$, such that
  \[
    \psi^\star(\hat{u}_k, \hat{v}_k) = (u_K(x_j), \nabla v_K(x_j))_{j\in \cJ},
  \]
  and $\size(\psi^\star) \le C |\cK|\log(|\cK|)$, $\depth(\psi^\star) \le C \log(|\cK|)$.
  Similarly, for the last mapping, there exists a ReLU neural network $\psi^{\star\star}: [\R^{n}]^{\cJ} \to \C^{\cK}$, such that
  \[
    \psi^{\star\star}(w(x_j)) = \hat{\P}_K(w), \quad \forall \; w \in L^2_{2K},
  \]
  and $\size(\psi^{\star\star}) \le C |\cK|\log(|\cK|)$, $\depth(\psi^{\star\star}) \le C \log(|\cK|)$.
  The main task is thus to show that the point-wise multiplication
  \[
   ( u_K(x_j), \nabla v_K(x_j) ) \mapsto u_K(x_j) \cdot \nabla v_K(x_j),
 \]
 can be suitably approximated by a neural network. To this end, we will rely on the following lemma:
 \begin{lemma}[{\citep[Lemma 2.6]{DRM}}] \label{lem:DRM}
   Let $M,L > 0$ and denote by $\times: \R^2 \to \R$, $(x,y)\mapsto xy$ the multiplication operator. For any $m\in \N$, there exists a realization of a ReLU neural network $\hat{\times}_m: [-M,M]\times [-L,L] \to \R$, such that $\hat{\times}_m$ satisfies for all $y\in [-L,L]$ the error bound
   \[
     \Vert \times(\slot, y) - \hat{\times}_m(\slot,y) \Vert_{\Lip([-M,M];\R)} \le \frac{M+L}{2^{m+1}}.
   \]
   Furthermore, it holds that $\size(\hat{\times}_m) \lesssim m$, $\depth(\hat{\times}_m) = m+1$.
 \end{lemma}

 We now note that for any $w \in L^2_{2K}$, $w(x)= \sum_{|k|_\infty \le 2K} \hat{w}_k e^{i\langle k, x\rangle}$, we have
 \[
   \Vert w \Vert_{L^2}^2
   =
   (2\pi)^n \sum_{|k|_\infty \le 2K} |\hat{w}_k|^2
   =
   \frac{(2\pi)^n}{|\cJ_{2K}|} \sum_{j\in \cJ_{2K}} |w(x_j)|^2.
 \]
 and we have
 \[
\Vert u_K \Vert_{L^\infty} \le CK^{n/2} \Vert u_K \Vert_{L^2} \le C_n K^{n/2} \Vert \hat{u}_k \Vert_{\ell^2} \le C_n K^{n/2} M,
\]
and
\[
  \Vert \nabla v_K \Vert_{L^\infty} \le CK^{n/2} \Vert \nabla v_K \Vert_{L^2} \le C_n K^{n/2+1} \Vert \hat{v}_k \Vert_{\ell^2} \le C_n K^{n/2+1} M.
\]
Applying Lemma \ref{lem:DRM}, it follows that there exists a neural network $\psi^\ast: \R^{n}\times \R^{n\times n} \to \R^n$, such that
\[
  \Vert
  \psi^\ast(u_K(x_j),\nabla v_K(x_j)) - u_K(x_j)\cdot \nabla v_K(x_j)
  \Vert_{\Lip}
  \le \epsilon,
\]
and $\size(\psi^\ast) \le C\log(MK/\epsilon)$, $\depth(\psi^\ast) \le C \log(MK/\epsilon)$, where $C = C(n)>0$ is independent of $M,K,\epsilon$. Clearly, the parallelized mapping $\psi^\epsilon : [\R^{n}]^\cJ \times [\R^{n\times n}]^{\cJ} \to [\R^n]^{\cJ}$, defined by the point-wise application $(u_K(x_j),\nabla v_K(x_j)) \mapsto \psi^\ast(u_K(x_j),\nabla v_K(x_j))$, $j\in \cJ_{2K}$, can be represented by a neural network with $\size(\psi^\epsilon) \le C |\cJ_{2K}| \log(MK/\epsilon)$. Furthermore, let $\hat{\psi}_\epsilon$ denote the coefficients $|k|_\infty \le K$ of the (fast) Fourier transform of $\psi^\epsilon$ based on the values on the grid $x_j$, $j\in \cJ_{2K}$, i.e. $\hat{\psi}_\epsilon = \psi^{\star\star}\circ \psi^\epsilon$. Then, we have
 \begin{align*}
   \Vert \hat{\psi}_\epsilon & (u_K(x_j),\nabla v_K(x_j)) - \hat{\P}_K(u_K(x_j)\cdot \nabla v_K(x_j))\Vert_{\ell^2(\cK)}
   \\
   &
     \le
     \frac{1}{|\cJ_{2K}|^{1/2}}
     \Vert \psi^\epsilon(u_K(x_j),\nabla v_K(x_j)) - u_K(x_j)\cdot \nabla v_K(x_j) \Vert_{\ell^2(\cJ_{2K})}
   \\
   &\le
     |\cJ_{2K}|^{1/2}
     \Vert \psi^\epsilon(u_K(x_j),\nabla v_K(x_j)) - u_K(x_j)\cdot \nabla v_K(x_j) \Vert_{\ell^\infty(\cJ_{2K})}
   \\
   &\le |\cJ_{2K}|^{1/2} \epsilon.
 \end{align*}
 In particular, replacing $\epsilon \to \epsilon/|\cJ_{2K}|^{1/2}$, we conclude that there exists a neural network $\psi: [\R^{n}]^\cJ \times [\R^{n\times n}]^{\cJ} \to [\R^n]^{\cJ}$, such that
 \[
   \size(\psi) \le C |\cK| \log(MK/\epsilon), \quad \depth(\psi) \le C \log(MK/\epsilon),
 \]
 where $C = C(n)>0$ is independent of $M,K,\epsilon$, 
 and
 \[
   \Vert \hat{\psi}(u_K(x_j),\nabla v_K(x_j)) - \hat{\P}_K(u_K(x_j)\cdot \nabla v_K(x_j)) \Vert_{\ell^2(\cK)}
   \le \epsilon.
 \]
 Furthermore, we note that for any $v_K, w_K$, we have
 \begin{align*}
   \Vert \hat{\psi}_\epsilon & (u_K(x_j),\nabla v_K(x_j)) - \hat{\psi}_\epsilon(u_K(x_j),\nabla w_K(x_j)) \Vert_{\ell^2(\cK)}^2
   \\
   &=
   \frac{1}{|\cJ_{2K}|}\sum_{j\in \cJ_{2K}} |\hat{\psi}_\ast(u_K(x_j),\nabla v_K(x_j)) - \hat{\psi}_\ast(u_K(x_j),\nabla w_K(x_j))|^2
   \\
   &\le
   \frac{1}{|\cJ_{2K}|}\sum_{j\in \cJ_{2K}} \Lip(\psi^\ast(u_K(x_j),\slot)) |\nabla v_K(x_j) - \nabla w_K(x_j)|^2.
 \end{align*}
 Since
 \begin{align*}
   \Lip(\psi^\ast(u_K(x_j),\slot))
   &\le
   \Lip(\psi^\ast(u_K(x_j),\slot) - u_K(x_j)\cdot (\slot)) + \Lip(u_K(x_j)\cdot (\slot))
   \\
   &\le
     \epsilon + \Vert u_K \Vert_{L^\infty}
   \\
   &\le \epsilon + CK^{n/2}M \le 2CK^{n/2}M,
 \end{align*}
 where we have wlog assumed that $\epsilon \le CK^{n/2}M$, it follows that there exists a constant $C = C(n)>0$, such that
 \begin{align*}
      \Vert \hat{\psi}_\epsilon & (u_K(x_j),\nabla v_K(x_j)) - \hat{\psi}_\epsilon(u_K(x_j),\nabla w_K(x_j)) \Vert_{\ell^2(\cK)}^2
   \\
                               &\le
                                 \frac{CK^{n}M^2}{|\cJ_{2K}|} \sum_{j\in \cJ_{2K}} |\nabla v_K(x_j) - \nabla w_K(x_j)|^2
   \\
                               &= CK^{n}M^2 \Vert \nabla v_K - \nabla w_K\Vert_{L^2}^2
   \\
                               &\le
                                 CK^{n+2}M^2 \Vert \hat{v} - \hat{w} \Vert_{\ell^2}^2.
 \end{align*}
 We conclude that 
 \begin{align*}
   \Vert \hat{\psi} & (u_K(x_j),\nabla v_K(x_j)) - \hat{\psi}(u_K(x_j),\nabla w_K(x_j)) \Vert_{\ell^2(\cK)}
   \\
   &\le
   CK^{n/2+1}M \Vert \hat{v} - \hat{w} \Vert_{L^2},                         
 \end{align*}
 for all $\Vert \hat{u}\Vert_{\ell^2}, \Vert \hat{v}\Vert_{\ell^2}, \Vert \hat{w}\Vert_{\ell^2} \le M$. Thus, the composition $\NL_{M,\epsilon} := \hat{\psi} \circ \psi^\star = \psi^{\star\star}\circ \psi \circ \psi^\star$ defines a neural network satisfying the estimates of the present lemma.
\end{proof}

\subsection{Proof of Lemma \ref{lem:update}}

\begin{proof}
Fix $\hat{u}^m$ as in the statement of this lemma. We will denote $\NL := \NL_{\bar{M},\epsilon}$, in the following. Consider the mapping $F: B_{\bar{M}}(0)\subset \C^{\cK} \to \C^{\cK}$, given by
\begin{align} \label{eq:Ffixedpt}
F(\hat{w})_k := 
\frac{1}{1+\frac12 \Delta t \,\nu|k|^2}
\left[
\hat{u}^m 
- \Delta t \, \NL\left(\hat{u}^m, \frac{\hat{u}^m + \hat{w}}{2}\right){}_k 
- \frac12 \Delta t\, \nu|k|^2 \hat{u}^m
\right].
\end{align}
It is straight-forward to check that $\hat{u}^{m+1}$ solves the update rule \eqref{eq:scheme-emulate} if, and only if, $\hat{u}^{m+1} = F(\hat{u}^{m+1})$. Furthermore, we note that 
\begin{align*}
\Vert F(\hat{w}) - &F(\hat{v}) \Vert_{\ell^2}
\\
&=
\frac{\Delta t}{1+\frac12 \Delta t\, \nu |k|^2}
\left\Vert 
\NL\left(\hat{u}^m, \frac{\hat{u}^m + \hat{w}}{2}\right) - \NL\left(\hat{u}^m, \frac{\hat{u}^m + \hat{v}}{2}\right)
\right\Vert_{\ell^2}
\\
&\le
\Delta t
\Lip\left(\NL\left(\hat{u}^m, \slot \right)\right) \frac{\Vert \hat{w}- \hat{v} \Vert_{\ell^2}}{2}
\\
&\le 
\frac12 CK^{n/2+1}\bar{M} \Delta t \, \Vert \hat{w} - \hat{v} \Vert_{\ell^2}.
\end{align*}
Thus, under the CFL condition $CK^{n/2+1}\bar{M} \Delta t \le 1$, it follows that $\Lip(F) \le 1/2$ is a contraction, with constant $\gamma = 1/2 < 1$. We also note that 
\begin{align*}
\Vert F(0) \Vert_{\ell^2}
&\le
\left\Vert \left(\frac{1-\frac12\Delta t \, \nu |k|^2}{1+\frac12\Delta t \, \nu |k|^2}\right) \hat{u}^m \right \Vert_{\ell^2}
+
\Delta t \left\Vert \NL(\hat{u}^m, \hat{u}^m/2) \right \Vert_{\ell^2}
\\
&\le 
\left\Vert \hat{u}^m \right \Vert_{\ell^2}
+
\Delta t \left( 
\epsilon + \frac12 \left\Vert \hat{P}_K( \hat{u}^m \cdot \nabla \hat{u}^m) \right \Vert_{\ell^2}
\right)
\\
&\le
M + \Delta t \, \epsilon + \frac12 C K^{n/2+1}M \Delta t
\\
&\le
M+2.
\end{align*}
In particular, this implies that 
\[
R := \frac{\Vert F(0) \Vert}{1-\gamma} \le 2(M+2) \le \bar{M}.
\]
We can thus apply Lemma \ref{lem:fixedpt} to conclude that there exists a unique fixed point $\hat{u}^{m+1}\in B_{\bar{M}}(0)$ ($= \hat{w}^\infty$) of $F$, i.e. such that $F(\hat{u}^{m+1}) = \hat{u}^{m+1}$, and satisfying the estimate
\begin{align} \label{eq:iterate-1}
\Vert \hat{w}^\ell - \hat{u}^{m+1} \Vert_{\ell^2}
\le \gamma^\ell \Vert \hat{u}^{m+1} \Vert_{\ell^2} = 2^{-\ell} \Vert \hat{u}^{m+1} \Vert_{\ell^2},
\end{align}
and hence
\begin{align} \label{eq:iterate-2}
\Vert \hat{w}^\ell \Vert_{\ell^2}
\le
\Vert \hat{u}^{m+1} \Vert_{\ell^2} + \Vert \hat{w}^\ell - \hat{u}^{m+1} \Vert_{\ell^2}
\le \left(1+2^{-\ell}\right) \Vert \hat{u}^{m+1}\Vert_{\ell^2}.
\end{align}
To provide a (improved) estimate on $\Vert \hat{u}^{m+1} \Vert_{\ell^2}$, we multiply \eqref{eq:scheme-emulate} by $\hat{u}^{m+1/2} = \frac12 (\hat{u}^m + \hat{u}^{m+1})$, to find
\begin{align} 
\frac{\Vert \hat{u}^{m+1} \Vert_{\ell^2}^2 - \Vert \hat{u}^m \Vert_{\ell^2}^2}{\Delta}
&=
\left\langle
\psi(\hat{u}^m, \hat{u}^{m+1/2}), \hat{u}^{m+1/2}
\right\rangle_{\ell^2}
-\nu \sum_{|k|_\infty \le K} |k|^2\left|\hat{u}^{m+1/2}_k\right|^2
\notag \\
&\le 
\left\langle
\psi(\hat{u}^m, \hat{u}^{m+1/2}), \hat{u}^{m+1/2}
\right\rangle_{\ell^2}.
\label{eq:unp1}
\end{align}
Next, we use the fact that 
\begin{align*}
\left\langle 
\hat{\P}_K (u^m_K\cdot \nabla u^{m+1/2}_K), \hat{u}^{m+1/2}
\right\rangle_{\ell^2}
&=
(2\pi)^{-n}
\left\langle 
{\P}_K (u^m_K\cdot \nabla u^{m+1/2}_K), {u}^{m+1/2}_K
\right\rangle_{L^2}
\\
&= 
(2\pi)^{-n}
\left\langle 
u^m_K\cdot \nabla u^{m+1/2}_K, {u}^{m+1/2}_K
\right\rangle_{L^2}
\\
&= 
0,
\end{align*}
and the fact that 
\[
\Vert \hat{u}^{m+1/2} \Vert_{\ell^2}\le \frac{1}{2} \left( \Vert \hat{u}^m \Vert_{\ell^2} + \Vert \hat{u}^{m+1} \Vert_{\ell^2} \right) \le \frac12 (M + R) \le \bar{M},
\]
to find
\begin{align*}
\Big\langle
\psi(\hat{u}^m, \hat{u}^{m+1/2}), & \hat{u}^{m+1/2}
\Big\rangle_{\ell^2}
\\
&= 
\left\langle
\psi(\hat{u}^m, \hat{u}^{m+1/2}) - \hat{P}_K(u^m_K \cdot \nabla u^{m+1/2}_K), \hat{u}^{m+1/2}
\right\rangle_{\ell^2}
\\
&\le
\left\Vert
\psi(\hat{u}^m, \hat{u}^{m+1/2}) - \hat{P}_K(u^m_K \cdot \nabla u^{m+1/2}_K)
\right\Vert_{\ell^2}
\left\Vert\hat{u}^{m+1/2}\right\Vert_{\ell^2}
\\
&\le
 \frac{\epsilon}{2} \left( \Vert \hat{u}^m \Vert_{\ell^2} + \Vert \hat{u}^{m+1} \Vert_{\ell^2} \right).
\end{align*}
Upon substitution of this estimate in \eqref{eq:unp1}, and simple manipulations, we thus conclude that 
\[
\Vert \hat{u}^{m+1} \Vert_{\ell^2} 
\le
\frac{1+\frac{\Delta t \, \epsilon}{2}}{1-\frac{\Delta t \, \epsilon}{2}}
\Vert \hat{u}^m \Vert_{\ell^2}.
\]
We finally note that 
\[
\frac{1+\frac{\Delta t \, \epsilon}{2}}{1-\frac{\Delta t \, \epsilon}{2}}
=
1 + \frac{\Delta t \, \epsilon}{1-\frac{\Delta t \, \epsilon}{2}}
\le
1 + 2\Delta t \, \epsilon 
\le \exp(2\Delta t\, \epsilon),
\]
where we used the assumption $\Delta t, \epsilon \in (0,1]$ in the first bound. The claimed error bounds on the recursive sequence thus follow from $\Vert \hat{u}^{m+1} \Vert_{\ell^2} \le \exp(2\Delta t\, \epsilon) \Vert \hat{u}^m \Vert_{\ell^2}$, and \eqref{eq:iterate-1}, \eqref{eq:iterate-2}.
\end{proof}

\subsection{Proof of Lemma \ref{lem:nn-representation}}
\label{app:nn-representation}

\begin{proof}
By Lemma \ref{lem:nonlin}, the neural network $\NL_{\bar{M},\epsilon}$ in Algorithm \ref{alg:emulation} can be chosen to be of size 
\[
\size(\NL_{\bar{M},\epsilon}) \le C |\cK_K|\log(\bar{M}K/\epsilon),
\quad
\depth(\NL_{\bar{M},\epsilon}) \le C \log(\bar{M}K/\epsilon).
\]
We now recall that $|\cK_K| = |\set{k\in \Z^n}{|k|_\infty \le K} \sim_n K^n$, $\bar{M} \sim M$, $\epsilon \gtrsim \Delta t / MT \sim \min(K^{-r},K^{-n/2-1}/M)/ MT$, for $K\to \infty$, and hence 
\[
\size(\NL_{\bar{M},\epsilon}) \lesssim_{r,n} K^n \log(MTK),
\quad
\depth(\NL_{\bar{M},\epsilon}) \lesssim_{r,n} \log(MTK).
\]
The fixed-point function $F(\hat{w}) = F(\hat{u}^m;\hat{w})$  (see \eqref{eq:Ffixedpt}) applied in the innermost loop of Algorithm \eqref{alg:emulation}, lines 11-13, can clearly be represented by a neural network of a size $\sim \size(\NL_{\bar{M},\epsilon})$, as it only differs from $\NL_{\bar{M},\epsilon}$ by a linear input and a linear output. Composing $L \lesssim_{r,n} \log(MTK)$ of these networks, we obtain a representation $\psi^{\ast}: \hat{u}^m \mapsto \hat{u}^{m+1}$, corresponding to lines 9-14 of Algorithm \ref{alg:emulation}, and such that 
\[
\size(\psi^{\ast}) \le C K^n \log(MTK)^2,
\quad
\depth(\psi^{\ast}) \le C \log(MTK)^2,
\]
where $C = C(n,r)>0$ is independent of $M,T,K$. Finally, composing 
\[
\psi = \underbrace{\psi^{\ast} \circ \dots \circ \psi^{\ast}}_{n_T \text{-fold}}
\]
allows us to represent the entire algorithm \ref{alg:emulation} by a ReLU neural network. The estimate on the size of $\psi$ follows from the fact that 
\[
n_T \sim T/\Delta t \sim T/\min(K^{-n/2-1}/M,K^{-r}) \lesssim_{r,d} MTK^r.
\]
\end{proof}

%%%%%%%%%%%%%%%%%%%%%%%%%%%%%%%%%%%%%%%%%%%%%%%%%%%%%%%%

\subsection{Proof of Lemma \ref{lem:normbound}} \label{app:normbound}

\begin{proof}
By definition of $M$, the claim is true for $m=0$, since $\Vert \hat{u}^0\Vert_{\ell^2}\le M \le 2M$. We prove the claim for $m>0$ by induction: To this end, we assume that \eqref{eq:normbound} holds true for some $m$. We aim to show that \eqref{eq:normbound} also holds for $m+1$. By the choice of $\bar{M} = 2(eM+2)$, and the choice of $\Delta t$, $\epsilon$ in Algorithm \ref{alg:emulation}, the assumptions of Lemma \ref{lem:update} are fulfilled (with $M$ replaced by $eM$), and hence there exists a unique fixed point $\hat{u}^{\ast,m+1} \in B_{eM}(0)$ of \eqref{eq:scheme-emulate}, i.e. such that 
\[
\frac{\hat{u}^{\ast,m+1}_k - \hat{u}^m_k}{\Delta t}
=
- \NL_{\bar{M},\epsilon}(\hat{u}^m, \hat{u}^{\ast,m+1/2})_k
- \nu |k|^2 \hat{u}^{\ast,m+1/2}_k.
\]
Furthermore, the iterates $\hat{w}^{\ell,m}$ defined by Algorithm \ref{alg:emulation} satisfy the estimates
\begin{align} \label{eq:iterates-est}
\Vert \hat{w}^{\ell,m} - \hat{u}^{\ast,m+1} \Vert_{\ell^2} \le \frac{\exp(2\Delta t\, \epsilon)}{2^\ell}\Vert \hat{u}^m \Vert_{\ell^2}, 
\end{align}
and
\begin{align} \label{eq:iterates-norm}
\Vert \hat{w}^{\ell,m} \Vert_{\ell^2}
\le
(1+2^{-\ell}) \exp(2\Delta t\, \epsilon) \Vert \hat{u}^m \Vert_{\ell^2}.
\end{align}
In particular, with the choice of $L \ge -\log(\Delta t\, \epsilon)$ of Algorithm \ref{alg:emulation}, we find that $\hat{u}^{m+1} := \hat{w}^{L,m}$ satisfies
\[
\Vert \hat{u}^{m+1} \Vert_{\ell^2} 
\le 
(1+\Delta t\, \epsilon) \exp(2\Delta t\, \epsilon) \Vert \hat{u}^m \Vert_{\ell^2}
\le
\exp(3\Delta t\, \epsilon) \Vert \hat{u}^m \Vert_{\ell^2}.
\]
By the induction hypothesis on $\Vert \hat{u}^m \Vert_{\ell^2} \le \exp(3m\Delta t\, \epsilon) M$, and the fact that $3(m+1)\Delta t \, \epsilon \le 3T \epsilon\le 1$, we can thus estimate
\[
\Vert \hat{u}^{m+1} \Vert_{\ell^2} 
\le
\exp(3(m+1)\Delta t\, \epsilon) M
\le 
e M.
\]
This completes the induction step.
\end{proof}

\subsection{Proof of Lemma \ref{lem:emulation}}

\begin{proof}
We note that upon Fourier transformation, the claim is equivalent to the following:
\begin{align} \tag{\ref{eq:emulation}'} \label{eq:emulation'}
\frac{\hat{u}^{m+1} - \hat{u}^m}{\Delta t}
=
-\hat{\P}_K \left(
u^m_K \cdot \nabla u^{m+1/2}_K
\right)
- \nu |k|^2 \hat{u}^{m+1/2}
+
\hat{Q}^m,
\end{align}
where the remainder $\hat{Q}^m \in \C^\cK$ satisfies an estimate of the form
\begin{align} \tag{\ref{eq:remainder-q}'} \label{eq:reminader-q'}
\sup_{m} 
\Vert \hat{Q}^m \Vert_{\ell^2} 
\le 
C \Delta t,
\end{align}
and $C>0$ is independent of $K$, $M$ and $T$.

By the choice of $\bar{M},\epsilon > 0$ and $\Delta t > 0$ of Algorithm \ref{alg:emulation}, the approximation of the non-linearity, $\NL_{\bar{M},\epsilon}$, satisfies \eqref{eq:err} and \eqref{eq:Lip} with 
\[
CK^{n/2+1}\bar{M} \le \frac{1}{\Delta t},
\]
i.e. we have
\begin{align} \label{eq:NLest}
\left\{
\begin{aligned}
\Vert 
\NL_{\bar{M},\epsilon}(\hat{u},\hat{v}) - \hat{\P}_K\left( u_K \cdot \nabla v_K \right) 
\Vert_{\ell^2} 
&\le 
\epsilon \le \frac{\Delta t}{3}, 
\\
\Vert 
\NL_{\bar{M},\epsilon}(\hat{u}, \hat{v}) - \NL_{\bar{M},\epsilon}(\hat{u},\hat{w}) 
\Vert_{\ell^2} 
&\le
\frac{1}{\Delta t} \Vert \hat{v} - \hat{w} \Vert_{\ell^2},
\end{aligned}
\right.
\end{align}
for all $\hat{u}, \hat{v}, \hat{w}$ such that $\Vert \hat{u} \Vert_{\ell^2}, \Vert \hat{v}\Vert_{\ell^2}, \Vert \hat{w} \Vert_{\ell^2} \le \bar{M} \equiv 2(eM+2)$. Furthermore, by Lemma \ref{lem:normbound}, we have $\Vert \hat{u}^m \Vert_{\ell^2} \le \bar{M}$ for all $m=0,\dots, n_T$.

As in the proof of Lemma \ref{lem:normbound}, for $m\in \{0,\dots, n_T\}$, we will in the following denote by $\hat{u}^{\ast,m}$ the (unique) solution of the emulation scheme with non-linearity $\NL_{\bar{M},\epsilon}$:
\begin{align} \label{eq:exact-fixed}
\frac{\hat{u}^{\ast,m+1}_k - \hat{u}^m_k}{\Delta t}
=
- \NL_{\bar{M},\epsilon}(\hat{u}^m, \hat{u}^{\ast,m+1/2})_k
- \nu |k|^2 \hat{u}^{\ast,m+1/2}_k,
\end{align}
where $\hat{u}^{\ast,m+1/2} := \frac12 (\hat{u}^m + \hat{u}^{\ast,m+1})$.
By the estimate \eqref{eq:iterates-est} on the fixed-point iteration, and the fact that $\exp(2\Delta t\, \epsilon) 2^{-L} \le e\Delta t\, \epsilon$ (by the choice of $L$ in Algorithm \ref{alg:emulation}), we then have
\begin{align} \label{eq:fptest}
\Vert \hat{u}^{m+1} - \hat{u}^{\ast,m+1} \Vert_{\ell^2}
\equiv
\Vert \hat{w}^{L,m} - \hat{u}^{\ast,m+1} \Vert_{\ell^2}
\le
2eM \Delta t \, \epsilon
\le
\frac{2e\Delta t^2}{3}.
\end{align}
Using the fact that $\hat{u}^{\ast,m+1}$ solves \eqref{eq:exact-fixed} exactly, we readily obtain
\begin{align*}
\frac{\hat{u}^{m+1} - \hat{u}^m}{\Delta t}
&=
\frac{\hat{u}^{m+1} - \hat{u}^{\ast,m+1}}{\Delta t}
+
\frac{\hat{u}^{\ast,m+1} - \hat{u}^{m}}{\Delta t}
\\
&=
\hat{Q}^m-\hat{\P}_K\left( u^m_K\cdot \nabla u^{m+1/2}_K \right) - \nu |k|^2 \hat{u}^{m+1/2}
,
\end{align*}
where $\hat{Q}^m = \hat{Q}_1^m + \hat{Q}_2^m + \hat{Q}^m_3$ is given by
\begin{align*}
\hat{Q}^m_1 
&:= 
\frac{\hat{u}^{m+1} - \hat{u}^{\ast,m+1}}{\Delta t}, \\
\hat{Q}^m_2
&:=
\NL_{\bar{M},\epsilon}(\hat{u}^m, \hat{u}^{m+1/2})
-
\NL_{\bar{M},\epsilon}(\hat{u}^m, \hat{u}^{\ast,m+1/2}), \\
\hat{Q}^m_3
&:=
\hat{\P}_K\left( u^m_K \cdot \nabla u^{m+1/2}_K \right)
-
\NL_{\bar{M},\epsilon}(\hat{u}^m, \hat{u}^{m+1/2})
.
\end{align*}
By \eqref{eq:fptest}, we can estimate $\Vert \hat{Q}_1^m \Vert_{\ell^2}
\le \frac{2e}{3} \Delta t$. By the second estimate of \eqref{eq:NLest}, we similarly find $\Vert \hat{Q}_2^m \Vert_{\ell^2}
\le \frac{e}{3} \Delta t$. By the first estimate of \eqref{eq:NLest}, we find $\Vert \hat{Q}^m_3 \Vert_{\ell^2} \le \frac{1}{3}\Delta t$. Combining these estimates, we conclude that 
\[
\Vert \hat{Q}^m \Vert_{\ell^2} \le C \Delta t,
\]
where $C = e + \frac13$ is an absolute constant independent of $M$, $T$, $m$.
\end{proof}

\subsection{Proof of Lemma \ref{lem:emulation-estimate}}

\begin{proof}
\label{pf:emulation-estimate}
We denote $U^m := u(t^m)$, and $U^{m+1/2} := \frac12 [u(t^m)+u(t^{m+1})]$. Define $w^m_K := u^m_K - U^m$, and $w^{m+1/2}_K := \frac12 [w^m_K + w^{m+1}_K]$. Then $w^m_K$ satisfies
\begin{align*}
\frac{w^{m+1}_K - w^m_K}{\Delta t}
&= - \P_K \left( w^m_K \cdot \nabla U^{m+1/2} \right)
- \P_K \left( U^m \cdot \nabla w^{m+1/2}_K \right)
\\
&\qquad 
+ \nu \Delta w^{m+1/2}_K 
+ R^m_K(u) - Q_K^m.
\end{align*}
Integrating against $w^{m+1/2}_K$, we find that
\begin{align} \label{eq:error-recursion}
\begin{aligned}
\frac12\Vert w^{m+1}_K \Vert_{L^2}^2 
-
\frac12\Vert w^m_K \Vert_{L^2}^2
&=
-\langle w_K^m\cdot \nabla U^{m+1/2}, w_K^{m+1/2} \rangle_{L^2}
-\nu \Vert \nabla w^{m+1/2}_K \Vert_{L^2}^2
\\
&\qquad
+ \langle R_K^m(u), w^{m+1/2}_K \rangle_{L^2}
- \langle Q_K^m, w^{m+1/2}_K \rangle_{L^2}.
\end{aligned}
\end{align}
We can now estimate (using the Cauchy-Schwarz inequality)
\begin{align*}
\left|\langle w_K^m\cdot \nabla U^{m+1/2}, w_K^{m+1/2} \rangle_{L^2}\right|
&\le
\frac12 \Vert w_K^{m+1} \Vert_{L^2}^2 + C \Vert w^m_K \Vert_{L^2}^2, 
\\
-\nu \Vert \nabla w^{m+1/2}_K \Vert_{L^2}^2 
&\le 0, 
\\
\left|\langle R_K^m(u), w^{m+1/2}_K \rangle_{L^2}\right|
&\le
\Vert R_K^m(u) \Vert_{L^2}^2 + \Vert w^{m+1}_K \Vert_{L^2}^2 + \Vert w^m_K \Vert_{L^2}^2, 
\\
\left|\langle Q_K^m, w^{m+1/2}_K \rangle_{L^2}\right|
&\le 
\Vert Q_K^m \Vert_{L^2}^2 + \Vert w^{m+1}_K \Vert_{L^2}^2 + \Vert w^m_K \Vert_{L^2}^2.
\end{align*}
Next, we recall that 
\[
\Vert R_K^m(u) \Vert_{L^2}^2
\le C(\Delta t^2 + K^{-2r} ), 
\]
by \eqref{eq:remainder}. Similarly, by \eqref{eq:remainder-q}, we have
\[
\Vert Q_K^m \Vert_{L^2}^2 \le C\Delta t^2, 
\]
where $C >0$ is independent of $K$, $\Delta t$ and $m$.
Combining these estimates with \eqref{eq:error-recursion}, we find that there exists a constant $C> 0$, such that
\begin{align*}
\Vert w^{m+1}_K \Vert_{L^2}^2
&\le
(1+C\Delta t) \Vert w^m_K \Vert_{L^2}^2
+
C\Delta t \left(
\Delta t^2 + K^{-2r}
\right).
\end{align*}
Applying the discrete Gronwall inequality, we conclude that 
\[
\max_{m=0,\dots, n_T} \Vert w^{m}_K \Vert_{L^2}^2
\le
e^{CT} \Vert w^0_K \Vert_{L^2}^2 + Ce^{CT} \left(\Delta t^2 + K^{-2r} \right).
\]
Noting that $w^0_K = u^0_K - u(0)$, we can use the assumed bound on $\Vert u \Vert_{C_tH^r_x}$ to conclude that 
\begin{align*}
\Vert w^0_K \Vert_{L^2} 
&\le \Vert u^0_K - \P_K u(0) \Vert_{L^2} + \Vert (1-\P_K) u(0) \Vert_{L^2} 
\\
&\le C\Vert \hat{u}^0 - \hat{\P}_K u(0)\Vert_{\ell^2} + C K^{-2r},
\end{align*}
where $C = C(n,r, \Vert u \Vert_{C_tH^r_x})>0$ is independent of $K$. Using also that $\Delta t\le K^{-r}$, we finally obtain 
\[
\max_{m=0,\dots, n_T} \Vert w^{m}_K \Vert_{L^2}^2
\le
Ce^{CT} \left(K^{-2r} + \Vert \hat{u}^0 - \hat{\P}_K u(0)\Vert_{\ell^2}\right),
\]
for a constant $C>0$, depending only on $\Vert u \Vert_{C_tH^r_x}$, $\Vert u \Vert_{C^1_tH^{r-1}_x}$, $n$, $r$. This implies the claimed estimate.
\end{proof}

\subsection{Proof of Lemma \ref{lem:NSd-approx-err}}
\label{app:NSd-approx-err}

\begin{proof}
\label{pf:NSd-approx-err}
We now note that 
\begin{align*}
\Vert \cD_{\cY} \circ \psi \circ \cE_{\cX} - \Psi^\dagger \Vert_{L^2(\mu)}
&\le
\Vert \cD_{\cY} \circ \psi \circ \cE_{\cX} - \cD_{\cY} \circ \cE_{\cY} \circ \Psi^\dagger \Vert_{L^2(\mu)}
\\
&\qquad +
\Vert \cD_{\cY} \circ \cE_{\cY} \circ \Psi^\dagger - \Psi^\dagger \Vert_{L^2(\mu)}
\\
&=
\Vert \cD_{\cY} \circ \psi \circ \cE_{\cX} - \cD_{\cY} \circ \cE_{\cY} \circ \Psi^\dagger \Vert_{L^2(\mu)}
\\
&\qquad +
\Vert P_{\cY} - \Id \Vert_{L^2(\Psi^\dagger_\#\mu)},
\end{align*}
where $P_{\cY} = \cD_{\cY} \circ \cE_{\cY}: \cY \to \cY$ is the (orthogonal) PCA projection.
The second term is the PCA projection error for the push-forward measure $\Psi^\dagger_\#\mu$. To estimate the first term, we first note that $\cD_{\cY}: \R^{d_\cY} \to \cY$ is a linear isometric injection, hence
\begin{align*}
\Vert \cD_{\cY} \circ \psi \circ \cE_{\cX} - \cD_{\cY} \circ \cE_{\cY} \circ \Psi^\dagger\Vert_{L^2(\mu)}
&=
\Vert \psi \circ \cE_{\cX} - \cE_{\cY} \circ \Psi^\dagger\Vert_{L^2(\mu)}.
\end{align*}
Given $K\in \N$, we now let $\psi: \R^{d_\cX} \to \R^{d_\cY}$ be a ReLU neural network of the form 
\[
\psi = (\cE_{\cY} \circ \cF_K^{-1}) \circ \hat{\psi} \circ (\cF_K \circ \cD_{\cX}),
\]
where $\cF_K^{-1}: [\C^n]^{\cK} \to L^2(\T^n;\R^n)$ and $\cF_K: L^2(\T^n;\R^n) \to [\C^n]^{\cK}$ denote the (truncated) inverse and forward Fourier transforms, respectively; we recall that
\begin{align*}
\cF_K^{-1}(\hat{u})
&=
\sum_{|k|_\infty\le K} \hat{u}_k e^{i\langle k, x\rangle},
\\
\cF_K(u)
&=
\left\{
\frac{1}{(2\pi)^n} \int_{\T^n} u(x) e^{-i\langle k, x \rangle} \, dx
\right\}_{|k|_\infty \le K}.
\end{align*}
Since $\cE_{\cY} \circ \cF_K^{-1}: [\C^n]^{\cK} \to \R^{d_\cY}$ and $\cF_K\circ \cD_{\cX}: \R^{d_\cX}\to [\C^n]^{\cK}$ are linear mappings, they can be represented by matrix multiplications, with matrices $Q \in \R^{2n|\cK| \times d_\cY}$, and $R\in \R^{d_\cX \times 2n|\cK|}$, respectively; i.e. we can write 
\[
\psi(w) = Q \cdot \hat{\psi}(R\cdot w),
\]
under the canonical identification $[\C^n]^{\cK} \simeq \R^{2n|\cK|}$. In particular, if $\hat{\psi}$ is a ReLU neural network of a given size, then $\psi$ satisfies
\[
\size(\psi) \le \size(\hat{\psi}) + 2n|\cK|(d_\cY+d_\cX), 
\quad
\depth(\psi) = \depth(\hat{\psi}).
\]
Let us assume that $\hat{\psi}$ is a ReLU neural network as in Proposition \ref{prop:ns-nn}. Then we can estimate:
\begin{align*}
\Vert \psi \circ \cE_{\cX} &- \cE_{\cY} \circ \Psi^\dagger \Vert_{L^2(\mu)}
\\
&=
\Vert (\cE_{\cY} \circ \cF_K^{-1}) \circ \hat{\psi} \circ (\cF_K \circ \cD_{\cX}) \circ \cE_{\cX} - \cE_{\cY} \circ \Psi^\dagger \Vert_{L^2(\mu)}
\\
&\le
\Vert \cF_K^{-1} \circ \hat{\psi} \circ \cF_K \circ (\cD_{\cX} \circ \cE_{\cX}) - \Psi^\dagger \Vert_{L^2(\mu)}.
\end{align*}
Denote the Fourier projection $P_K = \cF_K^{-1} \circ \cF_K$, and the PCA projection $P_\cX = \cD_{\cX} \circ \cE_{\cX}$, and estimate the last term by
\begin{align}
&\le 
\Vert \cF^{-1}_K \circ \hat{\psi} \circ \cF_K \circ P_\cX - \cF^{-1}_K \circ \cF_K \circ \Psi^\dagger \Vert_{L^2(\mu)}
+
\Vert (1 - P_K) \Psi^\dagger \Vert_{L^2(\mu)}
\notag
\\
&\le 
\Vert \hat{\psi} \circ \cF_K \circ P_\cX - \cF_K \circ \Psi^\dagger \Vert_{L^2(\mu)}
+
\Vert (1 - P_K) \Psi^\dagger \Vert_{L^2(\mu)}.
\label{eq:encap-est}
\end{align}
The integrand in the first term is of the form 
\[
\Vert \hat{\psi}( \hat{u}^0 ) - \hat{u}(T)\Vert_{\ell^2}, \quad u\in \cH^r_{M,T},
\]
where $\hat{u}^0 = \hat{\left[P_{\cX}u(0)\right]} \in [\C^n]^{\cK}$ collects the Fourier coefficients with $|k|_\infty \le K$ of the PCA projected $P_{\cX}u(0)$. By Proposition \ref{prop:ns-nn}, there exists a constant $C>0$, depending only on $M,r,T$, such that 
\[
\Vert \hat{\psi}( \hat{u}^0 ) - \hat{u}(T)\Vert_{\ell^2}
\le C(K^{-r} + \Vert \hat{u}^0 - \hat{u}(0) \Vert_{\ell^2}).
\]
Using also the definition of $\hat{u}^0 = \hat{[P_{\cX}u(0)]}$, we can further estimate the right-hand side in the form 
\[
\Vert \hat{\psi}( \hat{u}^0 ) - \hat{u}(T)\Vert_{\ell^2}
\le C(K^{-r} + \Vert (1-P_\cX)u(0) \Vert_{L^2_x}), \quad \forall \, u\in \cH^r_{M,T}.
\]
The other term in \eqref{eq:encap-est}, $\Vert (1-P_K)\Psi^\dagger \Vert_{L^2(\mu)}$, can be estimated using the uniform bound $\Vert \Psi^\dagger(u(0)) \Vert_{H^r_x} \le M$, for $u\in \cH^r_{M,T}$, to obtain 
\[
\Vert (1-P_K)\Psi^\dagger \Vert_{L^2(\mu)}
\le
C K^{-r},
\]
where $C=C(M,r,n)>0$ depends only on $r,M,n>0$. Upon substitution in \eqref{eq:encap-est}, we thus obtain an estimate of the form
\[
\Vert \psi \circ \cE_{\cX} - \cE_{\cY} \circ \Psi^\dagger \Vert_{L^2(\mu)}
\le
C K^{-r} + C\Vert \Id - P_\cX \Vert_{L^2(\mu)},
\]
where $C = C(M,r,n,T)>0$ is independent of $K$. We point out in passing that $\Vert \Id - P_\cX \Vert_{L^2(\mu)}$ is precisely the PCA projection error on $\cX$. Combining the above estimates, we obtain the claimed upper bound
\[
\Vert \Psi - \Psi^\dagger \Vert_{L^2(\mu)}
\le
C K^{-r} + C\Vert \Id - P_\cX \Vert_{L^2(\mu)} + \Vert \Id - P_\cY \Vert_{L^2(\Psi^\dagger_\#\mu)},
\]
where $\psi$ is a neural network of size 
\[
\size(\psi) \le C K^n \left( K^{r} \log(K)^2 + (d_\cX + d_\cY) \right), 
\quad
\depth(\psi) \le C K^r \log(K),
\]
and $C = C(M,r,T,n) > 0$ is a constant independent of $K, d_\cX, d_\cY$.
\end{proof}

\section{A lemma on fixed points}

The following lemma is a simple variant of the well-known Banach fixed point theorem. For completeness, we provide a detailed proof (although the result or a variant thereof may likely have appeared elsewhere).

\begin{lemma}\label{lem:fixedpt}
Let $X$ be a Banach space and fix $M>0$. Let $B_M(0)\subset X$ denote the closed ball of radius $M$. Let $F: B_M(0) \to X$ be a contractive mapping with $\gamma = \Lip(F) <1$, and assume that $\Vert F(0) \Vert \le (1-\gamma) M$. Denote $R := \Vert F(0) \Vert/(1-\gamma) \le M$. Then there exists a unique fixed point $w^\infty \in B_M(0)$ of $F$, and the norm of $w^\infty$ is bounded by $\Vert w^\infty \Vert \le R$. Furthermore, the recursively defined sequence $w^0 := 0$, $w^{k+1} := F(w^k)$, is well-defined for all $k=0,1,\dots$, and satisfies the a priori estimates:
\[
\Vert w^k \Vert \le R, \quad \Vert w^{k+1}-w^k \Vert \le \gamma^k R, \quad \Vert w^k - w^\infty \Vert \le \gamma^k R.
\]
\end{lemma}

\begin{proof}
Set $w^0 := 0$. Given $w^k$, we define $w^{k+1} := {F}(w^k)$, if $w^k\in B_M(0)$, otherwise we terminate the sequence. Let 
\[
k_0 = \sup\set{k\in \N}{w^1,\dots, w^k \text{ are well-defined}}.
\]
We claim that $k_0=\infty$, and that in fact $\Vert w^k \Vert \le R$ for all $k\in \N$. Suppose this was not the case. Then we must have $k_0<\infty$, and $w^{k_0} \notin B_M(0)$. We now note that for any $k < k_0$, we have
\begin{align*}
\Vert {w}^{k+1} - {w}^k \Vert
&=
\Vert F(w^{k}) - F(w^{k-1}) \Vert
\le
\gamma \Vert w^k - w^{k-1} \Vert
\\
&\le \dots
\le
\gamma^k \Vert {w}^1 - {w}^0 \Vert 
= \gamma^k \Vert F(0) \Vert,
\end{align*}
where we used the Lipschitz continuity of $F$, and the fact that $w^0 = 0$ and $w^1 = F(w^0) = F(0)$, in the last step.
Combined with the telescoping sum ${w}^{k_0} = \sum_{\ell=0}^{k_0-1} [{w}^{\ell+1}-{w}^{\ell}]$, this estimate implies that
\[
\Vert {w}^{k_0} \Vert \le \sum_{\ell=0}^{k_0-1} \Vert {w}^{\ell+1} - {w}^{\ell}\Vert \le \sum_{\ell=0}^{k_0-1} \gamma^{\ell} \Vert F(0) \Vert \le \frac{\Vert F(0) \Vert}{1-\gamma} = R,
\]
and hence $w^{k_0}\in B_R(0) \subset B_M(0)$. This clearly contradicts the definition of $k_0$. Hence, we must have $k_0 = \infty$. Thus, the recursive sequence $w^0=0$, $w^{k+1} = F(w^k)$ is well-defined for any $k=0,1,\dots$, and the above argument shows furthermore that $\Vert w^k \Vert \le R$, for all $k$. In fact, since the terms of the series $\sum_{\ell=0}^\infty [w^{\ell+1} - w^\ell]$ are summable in norm, we conclude that there exists a limit
\[
w^\infty = \sum_{\ell=0}^\infty [w^{\ell+1} - w^\ell] = \lim_{k\to \infty} \sum_{\ell=0}^{k-1} [w^{\ell+1} - w^\ell] = \lim_{k\to \infty} w^k.
\]
Passing to the limit, it is now straight-forward to show that $w^\infty$ is a fixed point of $F$, that $\Vert w^\infty \Vert\le R$, and that we have
\begin{align*}
\Vert w^k - w^\infty \Vert 
&=
\Vert F(w^{k-1}) - F(w^\infty) \Vert
\le
\gamma \Vert w^{k-1} - w^\infty \Vert 
\\
&\le
\dots 
\le 
\gamma^k \Vert w^0 - w^\infty \Vert
=
\gamma^k \Vert w^\infty \Vert \le \gamma^k R,
\end{align*}
for any $k\in \N$. To show that $w^\infty$ is unique, we assume that $v^\infty$ is another fixed point of $F$. Then 
\[
\Vert w^\infty - v^\infty \Vert 
=
\Vert F(w^\infty) - F(v^\infty) \Vert
\le
\gamma \Vert w^\infty - v^\infty \Vert,
\]
is only possible if $\Vert w^\infty - v^\infty \Vert = 0$.
\end{proof}

\end{document}